\documentclass[11pt]{article}

\usepackage[margin=1in]{geometry}
\usepackage{amsthm,amsmath,amsfonts,amssymb}
\usepackage[numbers,sort&compress]{natbib}
\usepackage[colorlinks,citecolor=blue,urlcolor=blue]{hyperref}
\usepackage{graphicx}
\usepackage{authblk}

\theoremstyle{plain}

\newtheorem{theorem}{Theorem}[section]
\newtheorem{lemma}[theorem]{Lemma}
\newtheorem{corollary}[theorem]{Corollary}
\newtheorem{proposition}[theorem]{Proposition}

\theoremstyle{definition}
\newtheorem{definition}[theorem]{Definition}
\newtheorem{example}{Example}

\newtheorem{remark}{Remark}
\newtheorem{assumption}{Assumption}

\newcommand{\SGD}{\text{SGD}}

\newcommand{\Ridge}{\text{Ridge}}

\input{mystyle}

\title{Statistical Inference for Misspecified Contextual Bandits}

\author[1]{Yongyi Guo\thanks{Corresponding author. Email: \texttt{guo98@wisc.edu}.}}
\author[2]{Ziping Xu\thanks{Email: \texttt{zipingxu@unc.edu}. The work was partially done when Ziping was a postdoctoral fellow at Harvard University with Professor Susan A. Murphy.}}

\affil[1]{Department of Statistics, University of Wisconsin--Madison}
\affil[2]{School of Data Science and Society, University of North Carolina at Chapel Hill}

\date{} 

\begin{document}
\maketitle





\begin{abstract}
Contextual bandit algorithms have transformed modern experimentation by enabling real-time adaptation for personalized treatment. Yet these advantages create challenges for statistical inference due to adaptivity. We study inference with contextual-bandit data without assuming a well-specified outcome model. In this setting, we show a previously overlooked issue: standard algorithms such as LinUCB may fail to stabilize under misspecified working models, leading to non-Gaussian estimator behavior and invalid inference. This issue is practically important, as misspecified working models---such as approximations of complex dynamical systems---are often employed by online agents in real-world adaptive experiments to balance reward, computational tractability, and robustness.

We develop an inverse-probability-weighted Z-estimation framework for a broad class of marginal moment targets, including projection parameters, structural parameters with noisy contexts, and off-policy values. We identify a stability condition tailored to this framework, scaled inverse-propensity convergence, under which the IPW-Z estimator is consistent and asymptotically normal with a consistent sandwich variance estimator. We further establish sufficient conditions for scaled inverse-propensity convergence for several policy classes, including multi-armed bandit algorithms and smooth contextual allocation policies. Simulations and a HeartSteps V1 real-data-calibrated application show reliable coverage and competitive performance across multiple targets. Overall, our results highlight the importance of stability-aware adaptive design for valid post-experiment inference.
\end{abstract}




\section{Introduction}



Adaptive experimental designs play an increasingly prominent role across science and industry. In contrast to traditional randomized experiments, they update arm allocation probabilities sequentially in response to observed outcomes, thereby allowing real-time improvement in treatment strategies and more efficient sample usage. For example, in mobile health, adaptive designs personalize digital interventions---such as motivational messages or prompts tailored to users’ states---to promote healthy behaviors \citep{nahum2016just, hardeman2019systematic}. In online advertising and recommender systems, content and ad delivery are dynamically adjusted to maximize engagement or click-through rates \citep{li2010contextual, tang2013automatic}. In education technology, adaptive experimentation has been used to optimize tutoring systems and learning platforms by tailoring exercises to student performance \citep{kizilcec2020scaling, fischer2020mining}. 

A common approach to implementing such adaptive experiments is through contextual bandit algorithms \citep{li2010contextual, offer2021adaptive, coughlin2024mobile, lauffenburger2024impact, nahum2024optimizing}. Contextual bandits provide a principled framework for sequential decision-making by balancing the tradeoff between exploration and exploitation. At each round $t$, the agent observes a context $\bX_t$ (e.g., user features), selects an action $A_t$ according to a behavior policy $\pi_t(\cdot \mid \bX_t)$, and then receives an outcome $Y_t$. The goal is to maximize cumulative reward while learning to improve the decision rule over time. 
This continual updating of the policy based on observed outcomes captures the essence of adaptivity in practice. Algorithms such as LinUCB \citep{li2010contextual} and Thompson Sampling \citep{russo2014learning} are influential examples of this paradigm and enjoy strong sample-efficiency guarantees under standard modeling conditions. 

In practice, however, many contextual bandit algorithms operate under a \emph{misspecified outcome model}. Misspecification is common in adaptive decision-making, especially in complex environments where data are limited and outcomes are noisy. One source is that the true reward-generating process is rarely known, leaving adaptive policies subject to model errors, particularly when there is rich heterogeneity or when the outcome mechanism is poorly understood \citep{dimakopoulou2017estimation, trella2025deployed}. Misspecification may also arise by design: practitioners often adopt simpler working models to stabilize learning and balance bias against variance \citep{athey2022contextual, tewari2017ads}. It can further occur when covariates are measured with error or when relevant features are unobserved. Because of its prevalence, misspecification has motivated a growing literature on designing online policies that remain reward-efficient under imperfect models \citep{ghosh2017misspecified, foster2020adapting, lattimore2020learning, krishnamurthy2021tractable, krishnamurthy2021adapting}. 

In this work, we tackle a different problem: \textbf{post-experiment statistical inference with data collected by contextual bandit algorithms, without assuming a well-specified outcome model}. Such inference is important because practitioners and scientists often seek more than an effective online policy: they need valid confidence intervals, hypothesis tests, and replicable findings for downstream scientific or policy questions. A substantial literature has studied inference with adaptively collected data under various modeling assumptions, including generalized linear and partial linear models \citep{deshpande2018accurate, boruvka2018assessing, qian2021estimating, lin2023semi}, structured nested mean models \citep{syrgkanis2023post}, and general nonlinear regression or M-estimation frameworks \citep{klimko1978conditional, lai1994asymptotic, zhang2021statistical}. However, the majority of these works focus on inferring parameters in a correctly specified outcome model. Much less is know about inference when no such model is assumed. Existing results in this direction either focus on specific targets, such as least-false linear parameters or off-policy values \citep{chen2021statistical, zhan2021off, bibaut2021post}, or study a different asymptotic regime with bounded horizons and a diverging number of independent trajectories \citep{zhang2022statistical, zhang2024replicable}. This paper studies adaptive inference in contextual bandits under a growing time horizon, without requiring a well-specified outcome model either for decision-making or for post-study analysis. This setting is both practically important and theoretically challenging.

One key difficulty we uncover is that statistical inference under misspecification is hindered \emph{from the very start of data collection}: popular bandit algorithms such as LinUCB may fail to stabilize as a policy, and this lack of convergence can fundamentally compromise subsequent inference. Different forms of stability have been recognized as important in adaptive experimentation---for example, in establishing asymptotic inference \citep{zhang2022statistical, khamaru2024inference, halder2025stable} and in achieving replicable conclusions across repeated deployments \citep{zhang2024replicable}. Our results reveal that while policy stability is typically not hard to establish under a well-specified model, it may fail entirely in the absence of such a model---even for standard bandit algorithms---leading to pathological estimator behavior, breakdowns of asymptotic normality, and ultimately invalid inference (see an example in Section \ref{sec::example-policy-nonconvergence}).

In this work, we develop a general statistical inference framework based on an inverse-probability-weighted Z-estimator (IPW-Z), targeting a broad class of parameters (Eq. (\ref{eq::theta-a-*})). This class includes many practically relevant examples such as projection parameters, structural parameters with noisy contexts, and off-policy policy values. Without a well-specified outcome model, the inverse probability weighting (IPW) component plays an essential first-order role: it removes the policy-induced reweighting of contexts in the estimating equation and restores estimation consistency. To achieve asymptotic normality, we further identify a policy stability requirement tailored to our framework: \emph{scaled inverse-propensity convergence} (Definition \ref{aspt:policy_convergence}). This condition controls the second-order behavior of the weighted estimating equation by requiring the inverse propensities to stabilize while allowing the policy to have action-specific sampling rate. It is flexible enough to accommodate adaptive policies that sample some actions much more frequently than others, as is common in bandit algorithms. Under this condition and mild regularity assumptions, we establish asymptotic normality and provide consistent variance estimators for data-driven inference.

To better understand when policy stability is preserved in our setting, we further establish general sufficient conditions under which a behavior policy satisfies scaled inverse-propensity convergence without a well-specified outcome model. Our analysis identifies two broad principles for designing stable adaptive policies in complex environments. First, policies that avoid tightly coupling exploration to a potentially misspecified reward model tend to be more stable; for example, context-free bandit algorithms satisfy inverse-propensity stability under suitable exploration schedules. Second, for contextual policies that rely on estimated reward scores, smoothness matters: policies with continuous and sufficiently smooth allocation rules are substantially more stable than policies with sharp or discontinuous decision boundaries. This explains why smooth allocation rules, such as Boltzmann sampling with sufficiently large temperature, are more robust than argmax-based policies such as LinUCB. We verify these principles for several commonly used policy classes, including multi-armed bandit algorithms and smooth contextual allocation policies based on ridge and stochastic-gradient-descent estimators. Together, these results provide practical guidance for designing adaptive policies that support valid post-experiment inference in complex environments.

To complement these theoretical guarantees, we evaluate the IPW-Z estimator in a range of complex environments, including synthetic simulations and a HeartSteps V1 real-data-calibrated application. The empirical results show that our method delivers confidence intervals with reliable coverage. In the special case of off-policy evaluation, it is competitive with existing approaches while avoiding the need to estimate a stable reward regression.

\subsection{Our Contributions}

Our main contributions are summarized as follows:
\begin{itemize}
    \item \textbf{A new inference framework.} We develop an inverse-probability-weighted Z-estimation framework for post-experiment inference with adaptively collected contextual-bandit data, without assuming a correctly specified outcome model. The framework covers the broad class of marginal moment targets in (\ref{eq::theta-a-*}), including projection parameters, structural parameters with noisy contexts, and off-policy policy values. We establish consistency, asymptotic normality, and consistent sandwich variance estimation.

    \item \textbf{The role of policy stability.} We show that standard contextual bandit algorithms, such as LinUCB, may fail to stabilize in the absence of a well-specified outcome model, which can lead to estimator pathologies and invalid inference. This issue has been largely overlooked in prior work on adaptive inference for contextual bandits, despite its importance for both inferential validity and replicability. We identify scaled inverse-propensity convergence (Def. \ref{aspt:policy_convergence}) as the stability condition naturally associated with IPW-Z inference. This condition accommodates action-specific sampling rates, allowing some actions to be sampled less frequently over time, which is common for many bandit algorithms. The usual fixed-overlap setting is recovered as a special case.
    
    \item\textbf{Principles for policy stability.} We establish general sufficient conditions under which adaptive policies satisfy scaled inverse-propensity convergence, yielding practical design principles for stable data collection. For example, simpler policies and smoother decision rules are more robust than policies that tightly couple exploration to complex or discontinuous reward models. We verify these principles for several commonly used policy classes, including multi-armed bandit algorithms and smooth contextual allocation policies based on ridge and stochastic-gradient-descent estimators.
    
    \item\textbf{Technical contributions.} We develop several technical tools for inference and policy-stability analysis without a well-specified outcome model. On the inference side, we establish asymptotic normality of the IPW-Z estimator under scaled inverse-propensity stability, which accommodates action-specific sampling rates and yields nonstandard, arm-specific rates. We also establish consistent sandwich variance estimation without requiring knowledge of the limiting inverse-propensity function. On the policy side, we establish stability guarantees with several proof strategies that are largely novel in this context. One example is the application of stochastic approximation theory, which is nontrivial since the policies depend on summary statistics whose evolution dynamics may fail to contract; in such cases, we instead identify hidden parameters that yield contractive dynamics. Finally, we present numerical examples where standard contextual bandit algorithms interact with misspecified environments in irregular ways---such as oscillating or converging to multiple limits---that are both interesting in their own right and informative for research on policy behavior.
\end{itemize}

\subsection{Related work}

Statistical inference with adaptively collected data has been a long-standing but challenging problem. Although such data are common in practice, classical inference procedures designed for i.i.d. samples can break down under adaptive collection \citep{nie2018adaptively, zhang2020inference}. A growing body of work has therefore focused on establishing conditions on the data generating process or developing new methodologies to ensure valid inference. Most existing work focuses on estimating parameters in correctly specified outcome models. For example, \citep{anderson1979strong, christopeit1980strong, lai1982least} analyze adaptive linear regression under various conditions; \citep{klimko1978conditional, lai1994asymptotic} study least squares estimation in adaptive nonlinear regression; and \citep{chen1999strong} establishes asymptotic properties of maximum quasi-likelihood estimators for generalized linear models with adaptive designs.

With the emergence of modern data collection schemes such as reinforcement learning, recent work has developed inference procedures accommodating more general mechanisms such as contextual bandits, which are not covered by earlier results. \citep{deshpande2018accurate, khamaru2021near} propose online debiasing estimators for adaptive linear regression with weaker restrictions on the data collection process. \citep{zhang2020inference} study batched linear regression under general contextual bandit algorithms. \citep{chen2021statistical} develop inference for linear contextual bandits with $\epsilon$-greedy policies, and \citep{chen2021statistical1} extend their results to nonlinear reward models with parameter updates via weighted stochastic gradient descent. \citep{zhang2021statistical} establish inference for $M$-estimators under contextual bandit sampling, and \citep{lin2023semi} consider adaptive inference in generalized partial linear outcome models. \citep{syrgkanis2023post} study inference for structural parameters in structural nested mean models with data collected via reinforcement learning algorithms. 

In the absence of a well-specified outcome model, several works have examined statistical inference with adaptively collected data for specific target estimands, primarily various forms of average outcomes. \citep{khamaru2024inference, han2024ucb} study inference on arm means under data collected by UCB-type algorithms, while \citep{halder2025stable} consider the same estimand with Thompson sampling variants. \citep{zhan2021off, hadad2021confidence, bibaut2021post, waudby2024anytime} analyze off-policy evaluation in bandits with general adaptive behavior policies, targeting the average reward of a specified evaluation policy. \citep{liao2021off, liao2022batch} develop inference for the long-run average reward in Markov decision processes with adaptive policies. Additional targets arising in longitudinal and causal panel data settings include mean responses to dynamic treatment regimes and average causal effects; see, for example, \citep{laan2003unified, chakraborty2013statistical}. 

Literature on inference for general target parameters beyond average outcomes without a well-specified reward model is comparatively sparse. A result closely related to ours is from \citep{chen2021statistical} who study inference in linear contextual bandits under both well-specified and misspecified reward models. In the misspecified case, they propose a weighted least squares estimator for the least false parameter, defined as the best linear projection of the true reward, and establish its asymptotic properties under an $\epsilon$-greedy behavior policy constructed from the same estimator. \citep{zhang2022statistical, zhang2024replicable} analyze inference after adaptive sampling in a general longitudinal data setting, but in a different regime where the time horizon is fixed and the number of trajectories diverges. 
Finally, a concurrent work \citep{LeinerDunnRamdas2025}, which generalizes \citep{bibaut2021post} on off-policy evaluation with adaptively collected data, studies M-estimation with respect to a fixed target policy under model misspecification. Their approach is complementary to ours: it accommodates unstable behavior policies by assuming consistent estimation of the conditional moments of the score function, requires either independent auxiliary data or a trajectory-splitting construction, and assumes action-selection probabilities are uniformly bounded away from zero. Our framework instead directly imposes a policy stability condition. This avoids consistent estimation of conditional score moments, does not require auxiliary data or control over the data-collection process, and allows action probabilities to decay at action-specific rates. Thus, neither framework dominates the other; they apply to different regimes of adaptive inference.

This work also connects adaptive inference to the convergence analysis of online learning algorithms. 
Stochastic approximation has been the main framework for convergence analyses of value-based RL, including TD and Q-learning in finite state-action spaces, and extensions that cover linear function approximation under additional stability and sampling conditions \citep{borkar2000ode,lee2019unified,carvalho2020new,liu2025ode}. These analyses typically assume correctly specified models. Although works on robustness and misspecification exist, they are less unified. For example, \citep{roy2017reinforcement} propose a robust version of Q-learning under model mismatch. We instead explored a broad family of policies that are stable without a well-specified model.

\paragraph*{Section layout} 
The remainder of the paper is organized as follows. Section~\ref{sec::problem-setup} introduces the notation and problem setup for adaptive inference in contextual bandits, together with three motivating examples: projection targets under misspecified linear working models, structural targets with noisy contexts, and off-policy evaluation. Section~\ref{sec::inference-guarantee} introduces the proposed inverse-probability-weighted Z-estimator (IPW-Z), establishes its consistency and asymptotic normality under scaled inverse-propensity convergence, and provides consistent variance estimators. It also presents a numerical example showing how policy instability can induce nonnormality and pathological estimator behavior. Section~\ref{sec::policy-convergence} develops sufficient conditions for inverse-propensity stability and verifies them for several policy classes, including multi-armed bandit algorithms and smooth contextual allocation policies based on ridge and stochastic-gradient-descent estimators. Section~\ref{sec::simulation} presents simulation studies and a HeartSteps V1 real-data-calibrated application. In the off-policy evaluation setting, we also compare with \citep{bibaut2021post} and \citep{zhan2021off}.

\section{Problem Setup}\label{sec::problem-setup}

We consider the problem of statistical inference with an adaptively collected dataset $\cD = \{\bX_t, \pi_t, A_t, Y_t\}_{t=1}^T$ from a contextual bandit environment. The data collection process proceeds at each time $t$ as follows:
\begin{itemize}
    \item \textbf{Context:} The environment reveals a context $\bX_t\in \cX\subseteq\RR^{d_X}$.
    \item \textbf{Action Selection:} Based on the current context $\bX_t$ and past history $\cH_{t-1}:=  \{\bX_\tau, \pi_\tau,\\ A_\tau, Y_\tau\}_{\tau<t}$, the agent selects an action $A_t \in \cA$ according to a stochastic behavior policy $\pi_t(\cdot \mid \bX_t, \cH_{t-1}) \in \Delta(\cA)$, where $\Delta(\cA)$ denotes the set of probability distributions over the action space $\cA$. The realized selection probability is   recorded as $\pi_t := \pi_t(A_t \mid \bX_t, \cH_{t-1})$. 
    \item \textbf{Outcome:} After choosing the action, the agent observes outcome $Y_t\in \RR$. 
\end{itemize}

We consider a finite action space $\cA$ and, without loss of generality, write $\cA = \{1, \ldots, K\}$. Adopting the potential outcomes framework \citep{imbens2015causal}, we let $\{Y_t(a): a \in \cA\}$ denote the potential outcomes for each action, with the observed outcome satisfying $Y_t = Y_t(A_t)$. We assume a stochastic contextual bandit environment in which $\{\bX_t, Y_t(a): a\in\cA\}\overset{\text{i.i.d.}}{\sim}\cP$, for $t = 1, \ldots, T$. In addition, in this adaptive experimental setting, we assume the following unconfoundedness condition.
\begin{assumption}\label{aspt:unconfoundedness}
    $A_t\perp \{Y_t(a)\}_{a\in\cA}|(\cH_{t-1}, \bX_t)$, for $t = 1, \ldots, T$.
\end{assumption}
Note that even though the potential outcomes are i.i.d., the observations in $\cD$ are not. This is because each action $A_t$ is selected based on the evolving history $\cH_{t-1}$, introducing temporal dependence into the observations. This dependence poses additional challenges for valid estimation and inference.

\subsection{Inference Targets}

Our goal is to infer the parameter $\btheta_a^* \in \RR^d$ associated with a treatment arm $a \in \cA$, or jointly the collection $\{\btheta_a^*\}_{a \in \cA}$. Each $\btheta_a^*$ is defined as the solution to the equation
\begin{equation}\label{eq::theta-a-*}
    \EE [\bg(\bX, Y(a); \btheta_a^*)] = \mathbf{0}
\end{equation}
for some known score function $\bg: \cX\times \RR\times \RR^d\rightarrow \RR^d$.%
\footnote{For simplicity, we assume a common score function $\bg$ across actions. Our analysis extends to action-specific score functions; see Appendix \ref{apdx::proof-thm::asymptotic-normality-joint-general} for details.} Here, $(\bX, Y(a))$ denotes a generic observation drawn from the same distribution as $(\bX_t, Y_t(a))$. 

Unlike many prior works on statistical inference with adaptively collected data \citep{lai1982least, zhang2021statistical, chen2021statistical1, lin2023semi, zhang2024replicable}, we do not assume a well-specified outcome model. Specifically, the validity of our inference procedure does not depend on specifying the law of \(Y_t\) given \((A_t,\bX_t)\) under \(\cP\), either for constructing the behavior policy or for post-experiment analysis. Throughout the paper, ``misspecified'' refers to the absence of such a reward law specification, rather than to the inferential target itself. This model-agnostic perspective is well suited to adaptive experiments in complex environments, where data may be limited, outcomes may be noisy, and simplified working models are commonly used.

The following examples show how (\ref{eq::theta-a-*}) accommodates several inferential targets, including a projection target under a misspecified working model, a structural target with noisy observed contexts, and a model-free policy-value target.

\begin{example}[Projection target under a misspecified linear model \citep{chen2021statistical}]\label{ex::misspecified-linear-bandits}
Consider a target parameter $\btheta_a^*$ which solves (\ref{eq::theta-a-*}) with 
\begin{equation}\label{eq::target-parameter-misspecified-linear-bandits}
\bg(\bx, y;\btheta) := \bx(y - \bx^\top \btheta).
\end{equation}
This score function corresponds to the best linear approximation of $Y_t(a)$ based on $\bX_t$ for arm $a \in \cA$. When the true dependence of $Y_t(a)$ on $\bX_t$ is linear, (\ref{eq::target-parameter-misspecified-linear-bandits}) yields the true linear parameter. Otherwise, (\ref{eq::target-parameter-misspecified-linear-bandits}) defines the best linear projection of $Y_t(a)$ onto the covariates $\bX_t$ in the least squares sense. Such parameters are standard model-robust targets in the literature \citep{white1980using, buja2019models}.

One concrete application is the characterization of treatment effect heterogeneity. In a two-arm setting, the contrast \(\btheta_1^*-\btheta_0^*\) is the coefficient of the best linear approximation to the conditional treatment effect $\tau(\bx) := \EE[Y_t(1)-Y_t(0)\mid \bX_t=\bx]$ over the linear span of \(\bX_t\). Thus, even when the true reward surfaces are nonlinear, \(\bX_t^\top(\btheta_1^*-\btheta_0^*)\) provides an interpretable low-dimensional summary of how treatment effects vary with covariates. Related best-linear-predictor summaries of heterogeneous
treatment effects are used in works such as \citep{chernozhukov2018generic}.
\end{example}
\begin{example}[Structural target under noisy contexts \citep{guo2024online}]\label{ex::bandits-noisy-contexts}
Suppose the potential outcome $Y_t(a)$ follows a linear model based on the unobserved true covariates $\bS_t$:
$$
Y_t(a) = \bS_t^\top\btheta_a^* + \eta_t,
$$
where $\eta_t$ is a mean-zero noise term. Instead of observing $\bS_t$, the observed context $\bX_t$ is a noisy proxy $\bX_t = \bS_t + \bepsilon_t$, where $\bepsilon_t$ is mean zero, uncorrelated with $\eta_t$, and has covariance matrix $\bSigma_e$. Although the outcome model is linear in $\bS_t$, we do not assume any parametric form for the distribution of the measurement error $\bepsilon_t$, making it difficult to characterize the conditional distribution of $Y_t(a) \mid \bX_t$.

With a non-adaptive data collection process, this model is the classical measurement error model well studied in statistics literature \citep{carroll1995measurement, fuller2009measurement}. Assuming $\bSigma_e$ is known, then $\btheta_a^*$ solves (\ref{eq::theta-a-*}) with
\begin{equation}\label{eq::target-parameter-bandits-noisy-contexts}
    \bg(\bx, y; \btheta) := \bx y - (\bx \bx^\top - \bSigma_e)\btheta.
\end{equation}
\end{example}

The measurement error model is motivated by practice in behavioral psychology. For example, in a mobile health study about reducing negative affect to improve medication adherence \citep{xu2025reinforcement}, the negative affect can only be measured through a short survey, a noisy proxy of the latent negative affect. However, the scientists are truly interested in the relationship between medication adherence and the latent negative affect rather than the noisy proxy.

\begin{example}[Off-policy evaluation]\label{ex::ope}
Off-policy evaluation is a canonical inferential task in contextual bandits \citep{li2011unbiased}. Here our goal is to estimate the average outcome under a target policy $\pi^e: \cX \to \Delta(\cA)$, defined as $V^* = \EE_{\bX}\EE_{A\sim \pi^e(\cdot|\bX)}Y(A)$. This target can be expressed as $V^* = \sum_{a \in \cA} \btheta_a^*$, where each $\btheta_a^*$ solves (\ref{eq::theta-a-*}) with an arm-specific score function $\bg$:
\begin{equation}\label{eq::target-parameter-ope}
\bg(\bx, y; \btheta) = \bg_a(\bx, y; \btheta) := \pi^e(a|\bx)y - \btheta.
\end{equation}
Thus, \(V^*\) is a model-free policy-value target.
\end{example}

\subsection{Challenge}\label{sec::challenge}
A central challenge in the absence of a well-specified outcome model is that standard Z-estimation approaches \citep{zhang2021statistical, lin2023semi, zhang2024replicable}, which analyze estimators $\hat\btheta_a'$ satisfying
\begin{equation}\label{eq::naive-z-estimator}
\frac1T\sum_{t=1}^T 1_{\{A_t = a\}}\bg\big(\bX_t, Y_t, \hat\btheta_a'\big) \approx \boldsymbol{0},
\end{equation}
fail to yield valid inference. This failure essentially stems from the interaction between the policy and the complex environment. Without a well-specified outcome model, the conditional moment $\EE[\bg(\bX, Y(a);\btheta_a^*)|\bX = \bx]$ need not vanish for every $\bx$, even though its average under the population distribution $\cP$ is zero.
However, the estimating equation (\ref{eq::naive-z-estimator}) uses only observations assigned to action $a$, which are  collected under a context-dependent behavior policy. 
Consequently, the solution to (\ref{eq::naive-z-estimator}) generally depends on the policy used to collect the data and need not approach $\btheta_a^*$.
This issue, as it arises in Examples \ref{ex::misspecified-linear-bandits} and \ref{ex::bandits-noisy-contexts}, is discussed in detail in \citep{chen2021statistical} and \citep{guo2024online}, respectively.

In this work, we address this issue by studying the asymptotic properties of the inverse probability weighted Z-estimators $\{\hat\btheta_a^{(T)}\}_{a\in\cA}$, where $\hat\btheta_a^{(T)}$ satisfies 
\begin{equation}\label{eq::ipwz-estimation-eq-approx}
    \bG_T(\btheta) := \frac{1}{T} \sum_{t=1}^T \frac{1_{\{A_t = a\}}}{\pi_t(A_t)}  \bg(\bX_t, Y_t; \btheta) \approx \mathbf 0.
\end{equation}
Here, $\pi_t(a)$ abbreviates the action selection probability $\pi_t(a \mid \cH_{t-1}, \bX_t)$ for $a\in\cA$. The precise approximation rate will be specified in Section \ref{sec::inference-guarantee}.
Our goal is to show that, in our setting, the inverse probability weights effectively decouple the policy from the underlying environment, enabling valid inference. Specifically, we will establish the joint asymptotic normality of $\{\hat\btheta_a^{(T)}\}_{a \in \cA}$ under mild conditions on the environment and for a broad class of behavior policies.

\section{Statistical Inference Guarantees}
\label{sec::inference-guarantee}
The main results of this section establish the joint asymptotic normality of the proposed estimators $\{\hat\btheta_a^{(T)}\}_{a \in \cA}$, along with consistent estimators of the asymptotic variance, which enable statistical inference for $\{\btheta_a^*\}_{a \in \cA}$. We begin by intuitively explaining how inverse probability weighting in the estimating equation (\ref{eq::ipwz-estimation-eq-approx}) lead to consistent estimation of the target parameter. We then introduce \emph{scaled inverse-propensity convergence}
(Definition~\ref{aspt:policy_convergence}), a general stability condition on the behavior policy
that helps stabilize the conditional variances of the weighted estimating
equation. Under this condition, Theorems~\ref{thm::asymptotic-normality-general}
and \ref{thm::asymptotic-normality-joint-general} establish asymptotic normality, and
Proposition~\ref{thm::consistent-var-estimator-general} provides a consistent variance estimator. Finally, to further illustrate the importance of policy stability, Section \ref{sec::example-policy-nonconvergence} presents a concrete example demonstrating how an unstable policy leads to the breakdown of asymptotic normality.

\paragraph{The role of inverse probability weighting (IPW).} IPW is not primarily needed for consistency in a well-specified conditional-moment model, but becomes essential in our setting where no such model is assumed. Specifically, when $\EE[\bg(\bX, Y(a);\btheta_a^*)|\bX] = \boldsymbol{0}$ a.s., following the arguments of \citep{zhang2021statistical}, one can show that the solution $\hat\btheta_a'$ to the unweighted estimating equation (\ref{eq::naive-z-estimator}) consistently estimates of $\btheta_a^*$ under suitable regularity conditions. By contrast, in the absence of a well-specified conditional-moment model, $\hat\btheta_a'$ is generally inconsistent. Below we will give a heuristic explanation of IPW's role in this setting. A numerical illustration is given in Section~\ref{sec::example-policy-nonconvergence}.


Fix an action $a\in\cA$, and define $\bZ_t(\btheta): = w_a(A_t)\bg(\bX_t, Y_t; \btheta)$, where $w_a(A_t):= \frac{1}{\pi_t(A_t)}1_{\{A_t = a\}}$. The key point is that the conditional mean of \(Z_t(\btheta)\) equals the population moment that defines the target parameter. Indeed, we have 
\begin{align}
\EE[\bZ_t(\btheta)|\cH_{t-1}]
&\stackrel{(a)}{=} \EE_{\bX_t}\!\left[\EE_{A_t\sim \pi_t(\cdot), Y_t(a)}[\bZ_t(\btheta)|\cH_{t-1},\! \bX_t]\Big|\cH_{t-1}\right]\nonumber\\
& \stackrel{(b)}{=} \EE_{\bX_t}\Big[\EE_{A_t\sim \pi_t(\cdot)}\big[w_a(A_t)\big|\cH_{t-1}, \!\bX_t\big]\!\cdot \!\EE_{Y_t(a)}[\bg(\bX_t, Y_t(a);\btheta)|\cH_{t-1},\! \bX_t]\Big|\cH_{t-1}\Big]\nonumber\\
& \stackrel{(c)}{=} \EE_{\bX_t}\!\left[1\cdot \EE_{Y_t(a)}[\bg(\bX_t, Y_t(a);\btheta)|\cH_{t-1},\! \bX_t]\Big|\cH_{t-1}\right]\nonumber\\
& \stackrel{(d)}{=} \EE\left[\bg(\bX_t, Y_t(a);\btheta)\right].\label{eq::main-mtg}
\end{align}
Here, step (a) follows from the law of iterated expectations. Step (b) uses Assumption~\ref{aspt:unconfoundedness}. Step (c) follows from a direct computation:
$$
\EE_{A_t\sim \pi_t(\cdot)}\big[w_a(A_t)\big|\cH_{t-1}, \!\bX_t\big] = \sum_{a'\in\cA}\pi_t(a')\cdot \frac{1_{\{a'=a\}}}{\pi_t(a')} = 1.
$$
Step (d) again applies the law of iterated expectations along with the assumption that $\{\bX_t, Y_t(a): a\in\cA\}$ are i.i.d. over time. Thus, when \(\btheta=\btheta_a^*\), $\EE[Z_t(\btheta_a^*)\mid \mathcal H_{t-1}]=\EE[\bg(\bX,Y(a);\btheta_a^*)]=
\boldsymbol{0}$.

This calculation shows that IPW restores the population estimating equation at each time point: the adaptive action-selection probability is canceled by IPW. Therefore, the empirical average \(\bG_T(\btheta)\) fluctuates around the population moment
$\EE[\bg(\bX,Y(a);\btheta)]$, whose unique zero is \(\btheta_a^*\). Under the usual regularity conditions for Z-estimation, approximate roots of \(\bG_T(\btheta)\) are consequently driven toward \(\btheta_a^*\).

In contrast, without inverse probability weighting, the conditional mean of the summand in (\ref{eq::naive-z-estimator}) is
$$
\EE[\bZ_t(\btheta)|\cH_{t-1}]=\EE_{\bX_t}\!\left[\pi_t(a|\cH_{t-1},\! \bX_t)\cdot \EE_{Y_t(a)}[\bg(\bX_t, Y_t(a);\btheta)|\cH_{t-1},\! \bX_t]\Big|\cH_{t-1}\right].
$$
This is a policy-weighted moment rather than the population moment in (\ref{eq::theta-a-*}). When the conditional moment is not centered within each context, the extra factor \(\pi_t(a\mid \mathcal \cH_{t-1},\bX_t)\) reweighs different contexts unequally, so the estimating equation (\ref{eq::naive-z-estimator}) generally results in an inconsistent solution.


\paragraph{Stability of inverse propensities.} 
The preceding calculation explains the first-moment role of IPW: by canceling the action-selection probability, IPW makes the conditional mean of the weighted score equal to the population moment defining $\btheta_a^*$.
For asymptotic normality, a natural next requirement is that the conditional variances of the weighted summands stabilize after an appropriate normalization. 
For the IPW-Z estimating equation, these variances depend on the inverse action-selection probabilities $1/\pi_t(a\mid \bX_t,\cH_{t-1})$. Thus, intuitively, it is desirable for these inverse propensities to converge in an appropriate sense. 

Importantly, under adaptive bandit policies, different actions may be sampled at different deterministic scales, so their inverse propensities may grow at action-specific rates. For example, a policy may increasingly favor empirically superior actions while sampling other actions less frequently. To keep the stability condition general, we allow each action to have its own sampling scale and use a context-dependent limit to describe the asymptotic behavior of the corresponding scaled inverse propensity. This motivates the following condition.

\begin{definition}[Scaled inverse-propensity convergence] 
\label{aspt:policy_convergence} 
The behavior policy $\pi = \{\pi_t(\cdot)\}_{t\geq 1}$ is said to satisfy scaled inverse-propensity convergence at action $a \in \cA$ with rate $\{r_{a, t}\}_t$ if there exists an integrable function $\rho: \cA\times \cX \mapsto \RR$ and a deterministic positive sequence $\{r_{a, t}\}_{t\geq 1}$ such that 
\begin{equation}\label{eq::def-policy-convergence}
\frac{r_{a, t}}{\pi_t(a|\bX_t, \cH_{t-1})}-\rho(a, \bX_t)\xrightarrow{L_1} 0\quad\text{ as }t\rightarrow\infty.
\end{equation}
We say that the behavior policy satisfies scaled inverse-propensity convergence with rates
\(\{r_{a,t}\}_{a\in\cA,t\ge 1}\) if the above condition holds for every action
\(a\in\cA\), possibly with action-specific rates \(\{r_{a,t}\}_{t\ge 1}\).
\end{definition}

In the above definition, $r_{a,t}$ serves as an action-specific deterministic sampling scale, making the condition flexible enough to normalize inverse propensities that may evolve at different rates across actions. 
For brevity, we sometimes refer to scaled inverse-propensity convergence as inverse-propensity stability. In Section~\ref{sec::policy-convergence}, we investigate sufficient conditions for inverse-propensity stability in general environments without a well-specified reward model, providing concrete guidance for designing adaptive policies that support valid inference.

\begin{remark}\label{rmk::policy-convergence-min-prob}
The condition in Definition~\ref{aspt:policy_convergence} is implied by a simpler convergence
condition when action probabilities are uniformly bounded away from zero. Specifically, suppose
there exists a policy \(\bar\pi:\cX\times \cA\to[0,1]\) such that
\[
    \pi_t(a\mid \bX_t,\cH_{t-1})-\bar\pi(a|\bX_t)
    \xrightarrow{p}0,
\]
and $\pi_t(a\mid \bX_t,\cH_{t-1})\geq \pi_{\min}$ a.s. for a constant \(\pi_{\min}>0\).
Then it is straightforward to verify that Definition~\ref{aspt:policy_convergence} holds with \(r_{a,t}\equiv 1\) and
\(\rho(a,\bx)=1/\bar\pi(a|\bx)\). Thus, the scaled formulation is more general as it also covers settings where the sampling probability of an action may decay over time, with \(r_{a,t}\) capturing the sampling scale for each arm. 
\end{remark}

\begin{remark}
Stability of the adaptive sampling process is a recurring requirement in inference with adaptively collected data, particularly in the absence of a well-specified model. For instance, in misspecified linear bandits (Example~\ref{ex::misspecified-linear-bandits}), \citep{chen2021statistical} study inference under an $\epsilon$-greedy algorithm with a weighted online least squares (LS) estimator, which is a special case of Definition~\ref{aspt:policy_convergence} with
$r_{a,t}\equiv 1$ (see Section~\ref{sec::policy-convergence}). Similar stability conditions are also assumed in other adaptive inference problems under model misspecification \citep{zhang2022statistical, zhan2021off, zhang2024replicable}. In settings without model misspecification, prior work has shown that various forms of stability conditions lead to valid statistical inference. These conditions may be weaker or different in form from Definition~\ref{aspt:policy_convergence},
but they play a similar role in stabilizing the adaptive sampling process for asymptotic inference. For example, \citep{lai1982least} studies stochastic linear regression and derives asymptotic normality of the OLS estimator under a stability condition that requires the design matrix to behave regularly over time. \citep{khamaru2024inference, han2024ucb} analyze the UCB algorithm in multi-armed bandits and show valid inference for the arm means under stability conditions where the ratio between the number of arm pulls and a diverging sequence converges to one.
\end{remark} 

\paragraph{Main results.} We first introduce the technical assumptions required to establish the asymptotic properties of the estimator $\hat\btheta_a^{(T)}$ for a single action $a \in \cA$.

\begin{assumption}[Well-separated solution]\label{aspt:identifiability-general}
    $\forall \epsilon>0$, $\inf_{\|\btheta - \btheta_a^*\|_2>\epsilon}\|\EE[\bg(\bX_t, Y_t(a);\btheta)]\|_2>0$.
\end{assumption}

\begin{assumption}[Bounded moments]\label{aspt:boundedness-general}
    There exist constants $R_\theta$, $M_2$ such that \\
    (i) $\|\EE[\bg(\bX_t, Y_t(a);\btheta_a^*)\bg(\bX_t, Y_t(a);\btheta_a^*)^\top|\bX_t]\|_2\leq M_2$, a.e. $\bX_t$; \\
    (ii) $\|\btheta_a^*\|_2<R_\theta$, $\sup_{\|\btheta\|_2\leq R_\theta}\EE[\|\bg(\bX_t, Y_t(a);\btheta)\|_2^2]<\infty$;
    (iii) $\EE [\|\bg(\bX_t, Y_t(a);\btheta_a^*)\|_2^4]<\infty$.
\end{assumption}

\begin{assumption}[Smoothness]\label{aspt:smoothness-general}
(i) The function $\bg(\bx, y; \btheta)$ is twice differentiable with respect to $\btheta$, with $\EE[\nabla\bg(\bX_t, Y_t(a);\btheta_a^*)]$ nonsingular; 
(ii) There exists a function $\phi: \RR^{d_X}\times \RR\mapsto \RR$ such that $\forall \bx, y$, $\sup_{\|\btheta\|_2\leq R_\theta}\|\nabla\bg(\bx, y; \btheta)\|_2\leq \phi(\bx, y)$, and $\EE[\phi(\bX_t, Y_t(a))^2|\bX_t]\leq M_2'$ a.e. $\bX_t$ for a constant $M_2'>0$; 
(iii) There exists a constant $\epsilon_0>0$ and a function $\Phi: \RR^{d_X}\times \RR\mapsto \RR$ such that $\sup_{\|\btheta - \btheta_a^*\|_2\leq \epsilon_0, i\in[d]}\|\nabla^2\bg^{(i)}(\bx, y; \btheta)\|_2\leq\Phi(\bx, y)$ and $\EE[\Phi(\bX_t, Y_t(a))]<\infty$. Here $\bg^{(i)}(\bx, y; \btheta)$ denotes the $i$-th entry of $\bg(\bx, y; \btheta)$.
\end{assumption}

\begin{assumption}[Minimum sampling probability] \label{aspt:min-sampling-prob}
There exists a deterministic sequence of positive constants $\{\pi_{\min,t}\}_{t\geq 1}$ such that $\pi_t(a)\geq \pi_{\min,t}$ almost surely.
\end{assumption}

\begin{remark}
We provide a few comments on these assumptions. First, Assumption \ref{aspt:identifiability-general} is a standard condition in Z-estimation that ensures the identifiability of the target parameter $\btheta_a^*$ \citep{van2000asymptotic}. Assumption \ref{aspt:boundedness-general} imposes mild regularity conditions on the boundedness of moments and conditional moments of the score function $\bg(\bX_t, Y_t(a); \btheta_a^*)$ evaluated at the true parameter $\btheta_a^*$. Notably, it does not require the potential outcomes $Y_t(a)$ themselves to be bounded or to satisfy sub-Gaussian tail conditions. Similar assumptions appear in related works, such as \citep{zhan2021off, chen2021statistical, zhang2021statistical, zhang2024replicable}. Assumption \ref{aspt:smoothness-general} imposes smoothness conditions on the score function $\bg$, a standard requirement in classical Z-estimation as well as in recent work on Z- and M-estimation with adaptively collected data \citep{zhang2021statistical, zhang2022statistical}. Finally,  Assumption~\ref{aspt:min-sampling-prob} imposes a time-varying lower bound on the action-selection probability, allowing exploration to decay while maintaining some randomization in the behavior policy. This type of randomized exploration is common in adaptive experiments and is useful for preserving information about all actions.

\end{remark}

We now state the first main result of this section, which establishes the asymptotic properties of the estimator $\hat\btheta_a^{(T)}$ for a single action $a \in \cA$. The proof is provided in Appendix~\ref{apdx::proof-thm::asymptotic-normality-general}.

\begin{theorem}\label{thm::asymptotic-normality-general}
Suppose Assumptions \ref{aspt:unconfoundedness} and  \ref{aspt:identifiability-general}--\ref{aspt:min-sampling-prob} hold for a given action $a\in\cA$. Suppose the behavior policy $\pi$ satisfies scaled inverse-propensity convergence at action $a$ with rate $\{r_{a, t}\}_{t\geq 1}$ and limit function $\rho$ in the sense of Definition~\ref{aspt:policy_convergence}. 
Define $b_{a, T} = T/S_{a, T}$, where $S_{a, T} = \sqrt{\sum_{t=1}^T 1/r_{a, t}}$. If in addition,  
\begin{itemize}
    \item[(i)] $\lim_{T\rightarrow \infty} S_{a, T} = \lim_{T\rightarrow \infty}b_{a, T} = \infty$,
    \item[(ii)] $\lim_{T\rightarrow \infty}\frac{\sum_{t=1}^T \pi_{\min,t}^{-1}}{T^2} = \lim_{T\rightarrow \infty}\frac{\sum_{t=1}^T \pi_{\min,t}^{-3}}{S_{a, T}^4} = 0$,
\end{itemize}
then there exists an estimator sequence $\{\hat\btheta_a^{(T)}\}_{T\geq 1}$ satisfying 
\begin{equation}\label{eq::estimating-equation-general}
    \bG_T(\btheta) = o_p\big(1/b_{a, T}\big)
\end{equation}
with $\|\hat\btheta_a^{(T)}\|_2 \leq R_\theta$ for all $T$. Moreover, for any such sequence, as $T\rightarrow \infty$, 
\begin{equation}\label{eq::asymptotic-normality-general}
    b_{a, T}(\hat\btheta_a^{(T)} - \btheta_a^*)\xrightarrow{d} \cN\left(\mathbf{0}, \bSigma_{a}^*\right).
\end{equation}
Here $\bSigma_{a}^*:= \bJ_a^{-1}\bar \bI_a \bJ_a^{-1, \top}$, with
$\bJ_a:= \EE[\nabla\bg(\bX_t, Y_t(a);\btheta_a^*)]$.
\end{theorem}

\begin{remark}\label{rmk::thm-asympt-norm-related-work}
Theorem~\ref{thm::asymptotic-normality-general} allows the minimum action-selection probability
$\pi_{\min,t}$ to decay over time. For example, if $\pi_{\min,t}\ge c t^{-\alpha}$ and
$r_{a,t}\asymp t^{-\alpha}$, then Conditions (i) and (ii) are satisfied for any
$\alpha\in[0,1)$. This is substantially weaker than the uniform-positivity conditions commonly
imposed in adaptive inference under model misspecification, which correspond to the case
$\alpha=0$ \citep{chen2021statistical, zhang2021statistical, zhang2022statistical, LeinerDunnRamdas2025}. In the special case of off-policy evaluation, some recent works allow decaying
overlap, but typically only for slower decay rates such as $\alpha\leq 1/2$
\citep{zhan2021off,bibaut2021post}. To our knowledge, Theorem~\ref{thm::asymptotic-normality-general} allows one of the broadest decay regimes currently available for inference with adaptively collected data under model misspecification.
\end{remark}

\begin{remark}
In the setting of misspecified linear bandits (Example \ref{ex::misspecified-linear-bandits}), Theorem 4.1 of \citep{chen2021statistical} establishes the asymptotic normality of $\hat\btheta_a^{(T)}$ under a specific behavior policy: an $\epsilon$-greedy policy paired with a weighted online LS estimator, with action-selection probabilities bounded away from zero. This is a special case of the policies that satisfies scaled inverse-propensity convergence shown in Definition \ref{aspt:policy_convergence}, with $r_{a,t}\equiv 1$. More generally, Section \ref{sec::policy-convergence} shows that a broad class of behavior policies satisfy this property. Therefore, Theorem \ref{thm::asymptotic-normality-general} generalizes Theorem 4.1 in \citep{chen2021statistical} in terms of both the allowable data collection mechanisms and the inferential targets. This broader formulation is useful in practice, as  statisticians seeking to conduct post-study inference often do not control the algorithm used to collect the data. 
\end{remark}

In practice, it is often desirable to jointly infer the collection $\{\btheta_a^*\}_{a\in\cA}$, where for each arm $a\in\cA= \{1, \ldots, K\}$, $\btheta_a^*$ denotes the solution to (\ref{eq::theta-a-*}). Such joint inference is particularly relevant for tasks like estimating individual treatment effects or evaluating the value of a general target policy. To achieve this goal, Theorem \ref{thm::asymptotic-normality-joint-general} below establishes the joint asymptotic normality of the estimators $\{\hat\btheta_a^{(T)}\}_{a\in\cA}$, with the proof provided in Appendix \ref{apdx::proof-thm::asymptotic-normality-joint-general}. For simplicity, we assume a common score function $\bg$ across actions, though our analysis extends to action-specific score functions; see Appendix \ref{apdx::proof-thm::asymptotic-normality-joint-general} for details. 

\begin{theorem}\label{thm::asymptotic-normality-joint-general}
Suppose Assumption \ref{aspt:unconfoundedness} holds, and Assumptions \ref{aspt:identifiability-general}--\ref{aspt:min-sampling-prob} hold for every action $a\in\cA$. Suppose the behavior policy $\pi$ satisfies scaled inverse-propensity convergence with action-specific rates
$\{r_{a,t}\}_{t\ge 1}$ and limit function $\rho$ in the sense of Definition~\ref{aspt:policy_convergence}. 
In addition, let Conditions (i) and (ii) in Theorem \ref{thm::asymptotic-normality-general} hold for each $a\in\cA$. Define $\bD_T = \mathrm{Diag}\big(b_{1, T} \bI_{d}, \ldots, b_{K, T} \bI_{d}\big)\in\RR^{(Kd)\times (Kd)}$. Then there exist estimators $\{\hat\btheta_a^{(T)}\}_{a\in\cA, T\geq 1}$ such that (\ref{eq::estimating-equation-general}) holds for each $a\in\cA$, and $\|\hat\btheta_a^{(T)}\|_2\leq R_\theta$ for all $a\in\cA, T\geq 1$. In addition,  any such estimators satisfy
\begin{equation}\label{eq::asymptotic-normality-joint-general}
    \bD_T(\hat\btheta^{(T)} - \btheta^*)\xrightarrow{d}\cN
    \left(
    \mathbf{0},\bSigma^*
    \right)
\end{equation}
as $T\rightarrow \infty$. Here $\hat\btheta^{(T)} = \big((\hat\btheta_1^{(T)})^\top, \ldots, (\hat\btheta_K^{(T)})^\top\big)^\top$, $\btheta^* = \big(\btheta_1^{*, \top}, \ldots, \btheta_K^{*, \top}\big)^\top$, and $\bSigma^* = \diag(\bSigma_1^*, \ldots, \bSigma_K^*)$, with each $\bSigma_a^*$ defined as in Theorem \ref{thm::asymptotic-normality-general}.
\end{theorem}

Theorem \ref{thm::asymptotic-normality-joint-general} indicates that the estimators $\{\hat\btheta_a^{(T)}\}_{a\in\cA}$ are asymptotically uncorrelated. Intuitively, this is because each $\hat\btheta_a^{(T)}$ is constructed using data exclusively from time points when action $a$ is selected, and these sets of time points are disjoint across different actions.

The asymptotic variances in Theorems \ref{thm::asymptotic-normality-general} and \ref{thm::asymptotic-normality-joint-general} can be consistently estimated from data, as shown in the proposition below, thereby enabling valid statistical inference. The proof is provided in Appendix \ref{apdx::proof-thm::consistent-var-estimator-general}.

\begin{proposition}\label{thm::consistent-var-estimator-general}
Under the same conditions of Theorem \ref{thm::asymptotic-normality-general}, the asymptotic variance $\bSigma_a^*$ can be consistently estimated by 
\begin{equation}\label{eq::estimator-asympt-var}
    \hat\bSigma_a = \left[\hat{\dot{\bG}}_{a, T}\right]^{-1}\hat \bI_{a, T}\left[\hat{\dot{\bG}}_{a, T}\right]^{-1, \top},
\end{equation}
where 
\begin{align*}
\hat{\dot{\bG}}_{a, T} &:= \frac1T\sum_{t=1}^T\frac1{\pi_t(A_t)}1_{\{A_t = a\}}\nabla\bg(\bX_t, Y_t; \hat \btheta_a^{(T)}),\\
\hat\bI_{a, T} & := \frac1{S_{a, T}^2}\sum_{t=1}^T\frac{1}{\pi_t(A_t)^2}1_{\{A_t = a\}}\bg(\bX_t, Y_t; \hat \btheta_a^{(T)})\bg(\bX_t, Y_t; \hat \btheta_a^{(T)})^\top.
\end{align*}
\end{proposition}

The above results show that valid inference only requires knowledge of the rate sequence
$\{r_{a,t}\}_{t\ge 1}$ up to a multiplicative constant. Indeed, for any $a\in\cA$, multiplying
$r_{a,t}$ by a constant changes both $b_{a,T}^2$ and the asymptotic variance
$\bSigma_a^*$ by the same factor, leaving the resulting confidence intervals unchanged.
Moreover, Proposition~\ref{thm::consistent-var-estimator-general} shows that consistent variance
estimation does not require knowledge of the limiting inverse-propensity function $\rho$.

The general results of this section are specialized in Appendix~\ref{seq::examples} to Examples~\ref{ex::misspecified-linear-bandits}-\ref{ex::ope}, where we derive the corresponding IPW-Z estimators, action-specific convergence rates, asymptotic variances, and consistent variance estimators.

\subsection{Policy Instability and Non-normality}\label{sec::example-policy-nonconvergence}

Previous work has shown that for off-policy evaluation with adaptively collected data, inverse probability weighted (IPW) estimators can exhibit non-normal asymptotic behavior when no constraints are imposed on the behavior policy \citep{hadad2021confidence, bibaut2021post}. In these works, the authors attribute such non-normality to the presence of unbounded inverse probability weights. In this section, we highlight another fundamental but less emphasized contributing factor to the non-normality of adaptive inference: \emph{policy instability}. Using a concrete numerical example, we show that under a misspecified environment, even a standard policy with bounded inverse probability weights may fail to stabilize at the natural scale $r_{a,t}\equiv 1$, \footnote{For a policy with bounded inverse propensities, one could
formally force Definition~\ref{aspt:policy_convergence} to hold by choosing a degenerate scale \(r_{a,t}\to 0\) and limit function \(\rho(a,x)\equiv 0\). Such a choice does not yield a nondegenerate
asymptotic distribution in Theorem~\ref{thm::asymptotic-normality-general}: the corresponding
asymptotic variance \(\bar{\bI}_a\) is zero and the resulting normalization does not provide meaningful
inference. We therefore refer here to stability at the natural, nondegenerate scale
\(r_{a,t}\equiv 1\), which is appropriate when the sampling probabilities are bounded away from
zero.} and such instability may in turn lead to non-normal asymptotic behavior of IPW-Z estimators. Similar phenomena have been observed in the context of estimating arm means with Thompson Sampling behavior policies \citep{halder2025stable}.



To illustrate, we consider a two-armed bandit environment with contexts generated as $\bX_t \sim \text{Uniform}(\{-4, 1\})$ and mean reward $y(\bx, a) = \mathbb{E}[Y_t(a) | \bX_t = \bx]$ given by 
$$
    y(-4, a_0) = y(1, a_0) = 1/2, \quad y(-4, a_1) = y(1, a_1) = 1/12.
$$
The rewards are subject to Gaussian noise. Under this setting, we independently run two algorithms---LinUCB with clipping and Random---across $2500$ replications, each for $10^4$ steps. LinUCB with clipping employs a working reward model $Y_t(a) = \btheta_{a}^{\top} \bX_t$ with unknown parameter $\btheta_{a} \in \mathbb{R}$, whose point estimate is the ridge regression estimator~(\ref{eq:ridge-estimator}) with $\lambda=1$. At each step LinUCB selects $\argmax_a\{\hat\bbeta_{t,a}^{\top}\bX_t + \alpha\,\|\bX_t\|_{\bPhi_{t,a}^{-1}}\}$ with exploration constant $\alpha=1$, and the resulting deterministic choice is then mixed with the uniform policy so that
its action probabilities are clipped at $0.01$ to ensure bounded inverse probability weights. The Random policy selects each action independently with equal probability $1/2$ at every time step, and by construction it naturally satisfies inverse-propensity stability as in Definition \ref{aspt:policy_convergence}.

In Figure \ref{fig:non-convergence}, we plot the distributions of key quantities across the $2500$ replications: the last-step ridge regression estimator for $\btheta_a$ (panel (a)), the sampling probabilities at context $\bx = -4$ (panel (b)), and the proposed IPW-Z estimator defined in (\ref{eq::estimating-equation-general}) for the target parameter in (\ref{eq::target-parameter-misspecified-linear-bandits}) (panel (c)). In panel (d), we present a QQ plot comparing the empirical distribution of the IPW-Z estimator with the standard Gaussian distribution. We observe that (i) the ridge regression estimator under clipped LinUCB exhibits two distinct convergence modes, evident from the bimodal histogram in panel (a), which does not occur under the Random algorithm; (ii) the clipped LinUCB policy itself fails to stabilize, as shown by the bimodal distribution of the sampling probability at $\bx=-4$ in panel (b), in contrast to the Random algorithm; and (iii) because the clipped LinUCB policy does not stabilize, the asymptotic distribution of the IPW-Z estimator substantially deviates from Gaussianity, whereas the estimator based on data collected by the convergent Random algorithm remains asymptotically normal (panels (c) and (d)). The bimodality in panel (a) also shows the inconsistency of the naive unweighted Z-estimator (\ref{eq::naive-z-estimator}) in this misspecified setting. The bimodality in panel (b) can be understood intuitively from the feedback between the misspecified working model and the LinUCB decision rule. Since the working model $Y=\btheta_a \bx$ has no intercept whereas each arm's true reward is constant in $\bx$, the least-squares slope for each arm depends on the empirical proportion of the two contexts among that arm's pulls. Under LinUCB, this proportion is itself affected by the current slope through the argmax rule. This feedback has two self-reinforcing fixed points, so the estimated ridge parameter and the induced sampling probability at $\bx=-4$ can settle near one of two values depending on early randomness. The IPW-Z estimator in panels (c)-(d) remains consistent because inverse-propensity weighting removes the first-order bias induced by adaptive sampling. However, because the behavior policy does not stabilize, the conditional variance of the weighted estimating equation also remains unstable, which produces the non-Gaussian mixture. This example illustrates the role of inverse-propensity stability in the IPW-Z estimator: it stabilizes the conditional variance of the weighted estimating equation and thereby enables asymptotic normality.


\begin{figure}[tb]
    \centering
    \includegraphics[width=\linewidth]{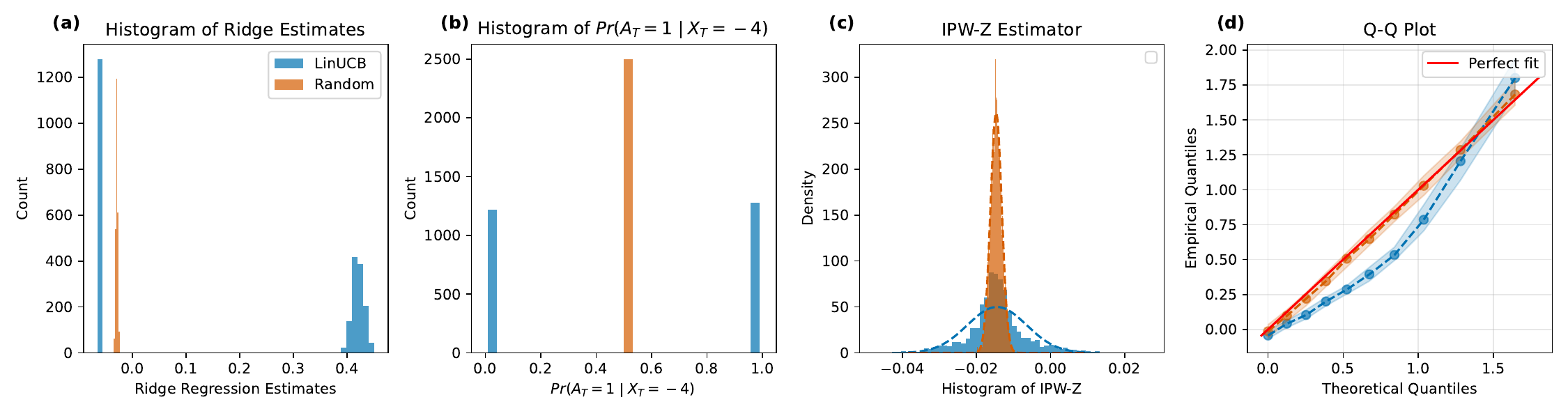}
    \caption{Example of policy non-convergence and non-normality of the IPW-Z estimator (\ref{eq::estimating-equation-general}). We independently run the LinUCB and Random algorithms for $10{,}000$ steps over $2{,}500$ replications. \textbf{(a)} Last-step ridge regression estimator of $\btheta_{1}$. \textbf{(b)} Last-step sampling probability at context $\bx = -4$. \textbf{(c)} The IPW-Z estimator (\ref{eq::estimating-equation-general}) for the inference target in (\ref{ex::misspecified-linear-bandits}). \textbf{(d)} QQ plot of the standardized empirical distribution of the IPW-Z estimator compared with the standard Gaussian distribution. Results shown in blue correspond to LinUCB, and results shown in orange correspond to Random.}
    \label{fig:non-convergence}
\end{figure}

\section{Sufficient Conditions for Inverse-Propensity Stability}\label{sec::policy-convergence}

In this section, we study sufficient conditions under which a policy satisfies the inverse-propensity stability condition in Definition~\ref{aspt:policy_convergence}. We mainly consider two complementary classes of policies. Specifically, Section~\ref{sec::MAB-policy-convergence} treats \emph{context-free multi-armed bandit} (MAB) algorithms, where the action-selection rule depends on the history only through arm-specific counts and sample means rather than on the context. We show that under mild conditions, three classical MAB algorithms---$\epsilon$-greedy, UCB, and Thompson sampling---all satisfy inverse-propensity stability under a minimum sampling probability that is either constant or decays to zero at a controlled rate.
Section~\ref{sec::smooth_allocation} then turns to contextual bandit algorithms and studies clipped smooth allocation policies, whose action-selection probabilities are smooth functions of an estimated score vector, clipped at a fixed minimum sampling probability. Unlike context-free MAB policies, these policies may rely on misspecified working score models; the fixed floor prevents over-commitment to such scores while preserving information across actions for post-study inference. We first prove a general result, Theorem~\ref{lem::statistics-converge-implies-policy-converge}, showing that convergence of the summary statistics used by the policy---the quantities through which the policy summarizes past data---implies inverse-propensity stability under a fixed minimum sampling probability. We then verify these conditions for smooth allocation policies based on two common summary statistics: ridge estimators and stochastic-gradient-descent estimators.

Throughout, our analysis applies to general environments under mild regularity conditions, without requiring a well-specified reward model. Together, these results offer a principled foundation for constructing stable bandit policies in practice, particularly in noisy and complex settings.

\paragraph{Policy parametrization by summary statistics.} We focus on behavior policies that belong to a parametric class of the form $\pi(\cdot|\bX_t, \bbeta)$, where $\bbeta\in\RR^{d_{\beta}}$. At each time $t$, the behavior policy $\pi_t(a|\bX_t, \cH_{t-1})$ takes the form 
\begin{equation}\label{eq::policy-statistics}
\pi_t(a|\bX_t, \cH_{t-1}) = \pi(a|\bX_t, \hat\bbeta_{t-1})
\end{equation}
for a fixed mapping $\pi: \cX\times \mathbb R^{d_{\beta}}\mapsto \Delta(\cA)$, where $\hat\bbeta_{t-1}\in\mathbb R^{d_{\beta}}$ is a $\cH_{t-1}$-measurable random vector. This vector can be interpreted as a \emph{summary statistic} that aggregates information from the past $t-1$ rounds and, together with the current context $\bX_t$, determines the action selection probability at time $t$. This policy class is broad and encompasses many commonly used bandit algorithms---such as $\epsilon$-greedy \citep{sutton1998reinforcement}, UCB \citep{auer2002finite, chu2011contextual}, and Thompson sampling \citep{agrawal2012analysis, russo2018tutorial}---and is widely adopted in adaptive experimental designs in practice \citep{liao2020personalized, athey2022contextual, xu2025reinforcement}.

Importantly, the summary statistic \(\hat\bbeta_{t-1}\) is part of the behavior policy and should not be confused with the inferential targets \(\{\btheta_a^*\}_{a\in\cA}\). The former is constructed by the agent during data collection to determine future action probabilities, whereas the latter are defined through the population estimating equations in (\ref{eq::theta-a-*}) for post-study statistical analysis.

\subsection{Multi-armed bandit ignoring context}\label{sec::MAB-policy-convergence} 
In real-world bandit deployments, simple algorithms like multi-armed bandits (MAB) are often preferred, especially in early trials when prior information is limited. By ignoring context, MAB avoids reward model misspecification, which helps to manage variance and prevent poor decisions in complex systems. In this section, we show that under mild conditions, common MAB algorithms additionally satisfy {scaled inverse-propensity convergence (Definition~\ref{aspt:policy_convergence})}, enabling valid post-study inference.

For any action $a\in\cA$ and time $t$, denote
\begin{equation}\label{eq::count-running-avg-reward-for-each-arm}
\hat\mu_{a, t}: =
\frac
{\sum_{\tau = 1}^t1_{\{A_\tau = a\}}Y_{\tau}}
{N_{a, t}},\quad N_{a, t} : = \sum_{\tau = 1}^t1_{\{A_\tau = a\}}.
\end{equation}
Consider the following MAB algorithms with a (possibly decaying) minimum sampling probability: 

\begin{itemize}
    \item \textbf{The $\epsilon$-greedy algorithm:}
        \begin{align}\label{eq::policy-eps-greedy-mab}
        \pi_t^{\epsilon\text{-greedy}}(a|\mathcal H_{t-1})=
        \begin{cases}
            1-(K-1)\pi_{\min,t},\quad &\text{if }a = \argmax_i \hat\mu_{i, t-1},\\
            \pi_{\min,t}, \quad &\text{otherwise.}
        \end{cases}
    \end{align}
    Here $\pi_{\min,t}\in(0, 1/K]$ is a (possibly time-varying) minimum sampling probability; equivalently, the standard $\epsilon_t$-greedy form with $\epsilon_t = K\pi_{\min,t}$.
    \item \textbf{The UCB algorithm:}
    \begin{align}\label{eq::policy-ucb-mab}
        \pi_t^{\text{UCB}}(a|\mathcal H_{t-1})= 
        \begin{cases}
            1-(K-1)\pi_{\min,t},\quad &\text{if }a \coloneq \argmax_i \left\{\hat\mu_{i, t-1} + \sqrt{\frac{C_t}{N_{i, t-1}}}\right\},\\
            \pi_{\min,t}, \quad &\text{otherwise.}
        \end{cases}
        \end{align}
        where $\{C_t\}_{t\geq 1}$ is a deterministic positive sequence. 
    \item \textbf{The Thompson Sampling algorithm:}   
\begin{align}\label{eq::policy-ts-mab-no-min-prob}
    \left(\pi_t^{\text{TS}}(a|\mathcal H_{t-1})\right)_{a\in\cA}=\mathrm{Clip}\left(\left(\bar\pi_t^{\text{TS}}(a|\mathcal H_{t-1})\right)_{a\in\cA};\pi_{\min,t}\right),
\end{align}
where    
\begin{itemize}
    \item
$
\bar\pi_t^{\text{TS}}(a|\mathcal H_{t-1}) \coloneq  \EE_{(\mu_i')_{i\in\cA}\sim \mathcal N(\bmu_{t-1}^{\mathrm{post}}, \bSigma_{t-1}^{\mathrm{post}})}1_{\{\forall i\neq a, \mu_i'<\mu_a'\}}
$
is the posterior probability of action $a$ being optimal under a Gaussian prior. Here
\begin{gather}
\bmu_{t-1}^{\mathrm{post}} = \big(\mu_{a, t-1}^{\mathrm{post}}\big)_{a\in\cA},\quad \bSigma_{t-1}^{\mathrm{post}} = \mathrm{diag}\big((\sigma_{a, t-1}^{\mathrm{post}})^2\big),\label{eq::ts-mab-posterior-mean-var-joint}\\
\mu_{a, t}^{\mathrm{post}}:=\bigg(\frac1{\sigma_0^2} +\frac{N_{a, t}}{\sigma^2}\bigg)^{-1}\bigg(\frac{\mu_0}{\sigma_0^2} +\frac{N_{a, t}\hat\mu_{a, t}}{\sigma^2}\bigg), \quad (\sigma_{a, t}^{\mathrm{post}})^2 := \bigg(\frac1{\sigma_0^2} +\frac{N_{a, t}}{\sigma^2}\bigg)^{-1},\label{eq::ts-mab-posterior-mean-var}
\end{gather}
    where $\mu_0\in\RR$ and $\sigma_0^2, \sigma^2>0$ are 
    fixed algorithm parameters representing the prior mean, prior variance, and observation noise variance, respectively.
    \item The mapping $\mathrm{Clip}: \RR^K\times \RR\rightarrow \RR^K$ adjusts a probability distribution over $K$ discrete actions to ensure that each coordinate is lower bounded by $\pi_{\min,t}$ (see Appendix \ref{apdx::clipping}).
    \end{itemize} 
\end{itemize}

All three policies above are special cases of the summary-statistic form~\eqref{eq::policy-statistics}, with degenerate dependence on the current context. For $\epsilon$-greedy, the statistic is the vector of sample means $(\hat\mu_{a,t-1})_{a\in\cA}$; for UCB, it can be taken as $(\hat\mu_{a,t-1}, C_t/N_{a,t-1})_{a\in\cA}$; and for Thompson sampling, it can be taken as the posterior mean and covariance $(\bmu^{\mathrm{post}}_{t-1},\bSigma^{\mathrm{post}}_{t-1})$.

The following proposition shows that the three MAB algorithms above satisfy the scaled inverse-propensity convergence condition in Definition~\ref{aspt:policy_convergence} when the minimum sampling probability decays to zero at a controlled rate. The proof is in Appendix~\ref{apdx::proof-prop::convergence-mab-no-pi-min}.

\begin{proposition}\label{prop::convergence-mab-no-pi-min}
Assume that the suboptimal gap $\Delta = \mu_{a^*}^* - \max_{a' \neq a^*}\mu_{a'}^* > 0$, and that $Y_t(a) - \mu_a$ is $\sigma_Y^2$-subgaussian $\forall t\geq 1, a\in\cA$. Suppose the exploration schedule satisfies:
\begin{itemize}
    \item[(i)] $\lim_{t\to \infty}\pi_{\min,t}= 0$, $\lim_{t\to \infty}\sum_{\tau\leq t}\pi_{\min,\tau} = \infty$,
    \item[(ii)] $\log (1/\pi_{\min,t}) = o(\sum_{\tau<t}\pi_{\min,\tau})$,
\end{itemize}
For the UCB policy, further assume that the sequence $\{C_t\}_{t\geq 1}$ satisfies $C_t = o(\sum_{\tau<t}\pi_{\min,\tau})$. 
Then the $\epsilon$-greedy, UCB, and Thompson sampling policies defined in \eqref{eq::policy-eps-greedy-mab}, \eqref{eq::policy-ucb-mab}, and \eqref{eq::policy-ts-mab-no-min-prob}, respectively, satisfy the scaled inverse-propensity convergence condition in Definition~\ref{aspt:policy_convergence}, with 
\begin{equation}\label{eq::mab-policy-limit-no-min-prob}
    r_{a, t} = 
\begin{cases}
1,\quad &\text{if } a = a^*,\\
\pi_{\min,t},\quad &\text{otherwise, } 
\end{cases}
\quad 
\rho(a, \bx) = 1.
\end{equation}

\end{proposition}

Conditions~(i) and~(ii) are satisfied, for example, by the polynomial schedule $\pi_{\min,t} = c\, t^{-\alpha}$ for any $c>0$ and $\alpha\in(0, 1)$. The fixed-floor regime $\pi_{\min,t}\equiv \pi_{\min}>0$ is treated separately in Appendix~\ref{apdx::proof-lem::convergence-mab}, where we show that the three algorithms satisfy scaled inverse-propensity convergence with $r_{a,t}\equiv 1$ for all arms.


\begin{remark}
Although regret is not the main focus of this paper, the decaying-floor regime also retains reasonable reward performance in the context-free setting. Under the polynomial schedule $\pi_{\min,t}=ct^{-\alpha}$ with any $\alpha\in(0,1)$, the cumulative exploration mass is $O(T^{1-\alpha})$; hence, the resulting MAB policy has sublinear regret relative to the best fixed arm while still satisfying the inverse-propensity stability under nonzero optimality gap.
\end{remark}

\subsection{Smooth Allocation Policies}\label{sec::smooth_allocation}\label{sec::boltzmann-policy-convergence}

So far we have considered multi-arm bandit algorithms, for which the dependence on the current context in~(\ref{eq::policy-statistics}) is degenerate. We now turn to contextual policies of the form  (\ref{eq::policy-statistics}) that use $\bX_t$ for decision-making. A common way to construct such policies is to first assign each action an estimated score $s_a$, such as a fitted reward index obtained from a working model based on $(\bX_t,\widehat\bbeta_{t-1})$ \citep{abbasi2011improved, agrawal2013thompson, cesa2017boltzmann}. The resulting score vector, 
$$
\bs_{\hat\bbeta_{t-1}}(\bX_t) := (s_a)_{a\in\cA},
$$
is then converted into action-selection probabilities. Intuitively, the vector $\bs_{\widehat\bbeta_{t-1}}(\bX_t)\in\RR^K$ summarizes how favorable the current context makes each action appear under the policy's current summary statistic $\widehat\bbeta_{t-1}$. We refer to the entries of $\bs_{\widehat\bbeta_{t-1}}(\bX_t)$ as \emph{action scores}, to distinguish them from the estimating score function $\bg$ used to define the inferential target.

A key design choice is how to transform the action scores to action-selection probabilities. Hard decision-rules, such as argmax-based policies, can be unstable under misspecified working models (see Section \ref{sec::example-policy-nonconvergence}). In this section, we study \emph{smooth allocation policies}, which convert the score vector $\bs_{\widehat\bbeta_{t-1}}(\bX_t)$ into the probability simplex $\Delta(\cA)$ via a smooth allocation rule. To maintain persistent exploration and support post-study inference, we additionally clip the resulting probabilities at a fixed minimum sampling probability. Specifically, we consider policies of the form (\ref{eq::policy-statistics}), where
\begin{equation}\label{eq::clipped-allocation}
\pi(a\mid\bX,\bbeta) = \pi^{\brho}(a\mid\bX,\bbeta) := \bigl[\mathrm{Clip}\bigl(\brho(\bs_{\bbeta}(\bX));\,\pi_{\min}\bigr)\bigr]_a,
\end{equation}
where $\brho: \RR^K\to \Delta(\cA)$ belongs to a smooth allocation family $\{\brho_{\gamma}\}_{\gamma>0}$, defined formally in Definition~\ref{def:smooth-allocation} below. The operator $\operatorname{Clip}(\cdot;\pi_{\min})$ is defined in Appendix~\ref{apdx::clipping} and ensures that every action has probability at least $\pi_{\min}$, where $\pi_{\min}\in(0,1/K)$ is fixed.

Our analysis proceeds in two steps. We first develop general sufficient conditions that reduce verification of inverse-propensity stability to convergence of the policy's summary statistic (Theorem \ref{lem::statistics-converge-implies-policy-converge}). We then verify this convergence for two widely used statistics---the ridge estimator and a stochastic-gradient-descent (SGD) estimator---under the common linear-score implementations in Sections \ref{sec::ridge-convergence} and \ref{sec::SGD}.


\paragraph{Smooth allocation families.}
We now formalize the class of smooth allocation maps used to convert action scores into action-selection probabilities. Many commonly used allocation rules, such as Boltzmann exploration, are indexed by a tuning parameter \(\gamma\). Intuitively, larger values of \(\gamma\) generally lead to smoother action-selection probabilities and less sensitivity to differences in the action scores.

\begin{definition}[Smooth allocation family]\label{def:smooth-allocation}
    Let
    \[
    \Delta_K:=\Bigl\{p\in[0,1]^K:\sum_{a=1}^K p_a=1\Bigr\},
    \qquad
    \mathrm{int}(\Delta_K):=\Bigl\{p\in(0,1)^K:\sum_{a=1}^K p_a=1\Bigr\}.
    \]
    A family of maps $\{\brho_\gamma:\mathbb R^K\to\operatorname{int}(\Delta_K)\}_{\gamma>0}$ is called a smooth allocation family if the following two conditions hold:
    \begin{itemize}
        \item[(i)] For each $\gamma>0$, \(\brho_\gamma\) is continuously differentiable. Moreover, for any fixed $R > 0$, there exists a constant $L_{\brho}(\gamma,R)$ depending only on $\gamma$ and $R$ such that
        $\sup_{\bz \in \mathbb{R}^K: \|\bz\|_2 \leq R}\|\nabla\brho_{\gamma}(\bz)\|_{\mathrm{op}} \leq L_{\brho}(\gamma, R)$, and $\lim_{\gamma \to \infty} L_{\brho}(\gamma, R) \to 0$.
        \item[(ii)] \(\brho_\gamma(\mathbf{0}) = (1/K, \ldots, 1/K)\).
    \end{itemize}
    \end{definition}
    
Here Condition~(i) controls how sharply action probabilities can change as the action scores vary; as $\gamma$ increases, this sensitivity vanishes on every bounded score region. Condition~(ii) is a natural normalization: when all actions have the same score, no action is preferred, so the allocation is uniform. 

\paragraph{Fixed minimum sampling probability.}

Unlike the MAB analysis in Section~\ref{sec::MAB-policy-convergence}, where the minimum sampling probability may decay to zero, we focus here on the fixed-floor regime with $\pi_{\min,t}\equiv \pi_{\min}>0$. In classical well-specified bandit settings, reducing exploration over time is often motivated by regret minimization: once the best action can be reliably identified, continued exploration may incur unnecessary reward loss. In our setting, however, the policy is constructed without assuming a well-specified outcome model, so the usual regret-based motivation for vanishing exploration is less decisive. Shrinking exploration does not, by itself, guarantee that the policy approaches the true optimal contextual decision rule in any way, and in addition, it may substantially reduce the information available for post-study inference. A fixed minimum sampling probability therefore provides a simple way to maintain stable exploration and preserve sample size. This design is also well aligned with practice: adaptive interventions often reserve a fixed share of users for exploration to preserve statistical power for flexible post-hoc analysis \citep{yao2021power, lauffenburger2024impact}, particularly when the analysis objective is not pre-specified at the time of data collection, and to enable policy updates and re-optimization for future users under potential non-stationarity across trials \citep{liao2020personalized, trella2025deployed, yang2024targeting}.


\paragraph{Example: Boltzmann exploration.}
A motivating example throughout this section is \emph{Boltzmann exploration}, also known as softmax or Gibbs exploration \citep{kaelbling1996reinforcement, sutton1998reinforcement, sutton1999policy, vermorel2005multi, cesa2017boltzmann}. Given a temperature parameter $\gamma>0$ and an estimated linear score $\hat\bbeta_{t-1, a}^\top \bX_t$ for each action, the Boltzmann policy selects action $a$ with probability
\begin{align}\label{eq::Boltzmann-policy-unclipped}
\pi^{\gamma}(a \mid \bX_t, \hat\bbeta_{t-1}) = \frac{\exp\left(\big\langle \hat\bbeta_{t-1, a}, \bX_t\big\rangle / \gamma\right)}{\sum_{a'\in\cA}\exp\left( \big\langle \hat\bbeta_{t-1, a'}, \bX_t\big\rangle / \gamma\right)},
\end{align}
where $\hat\bbeta_{t-1}\! \!=\!\! (\hat\bbeta_{t-1, a})_{a\in\cA}$ is a summary statistic that aggregates information from past history $\cH_{t-1}$. 

This policy corresponds to the linear action-score vector
$
\bs_{\widehat\bbeta_{t-1}}(\bX_t)
=
\bigl(\widehat\bbeta_{t-1,a}^{\top}\bX_t\bigr)_{a\in\cA}
$
and allocation map $\brho = \brho_\gamma$ satisfying
$
[\brho_\gamma(\bs)]_a
=
\frac{\exp(s_a/\gamma)}
{\sum_{a'\in\cA}\exp(s_{a'}/\gamma)}
$. The temperature $\gamma$ provides a continuous way to tune the exploration–exploitation tradeoff: $\gamma\rightarrow\infty$ recovers uniform exploration, while $\gamma\rightarrow 0$ approaches a pure greedy rule. The softmax family $\{\brho_\gamma\}_{\gamma>0}$ satisfies Definition~\ref{def:smooth-allocation}, since it is continuously differentiable, assigns uniform probabilities at equal scores, and has derivative of order $1/\gamma$ on any bounded score region.

In our setting, we additionally apply the clipping operator to the softmax probabilities to enforce a fixed minimum sampling probability. Note that when the score vector remains in a bounded region on which $\min_{a\in\cA}[\brho_\gamma(s)]_a\ge \pi_{\min}$, the clipping step does not bind, and the resulting policy coincides with the standard Boltzmann policy.

\paragraph{From statistic convergence to policy stability.}
With the minimum-sampling-probability floor in hand (from the clip~\eqref{eq::clipped-allocation}), verifying inverse-propensity stability reduces to two simpler tasks: proving convergence of the policy's summary statistic $\widehat\bbeta_t$, and checking continuity of the policy mapping at the limit. The following theorem makes this reduction precise and is the main tool for the rest of the section; its proof is in Appendix~\ref{apdx::proof-lem::statistics-converge-implies-policy-converge}.

\begin{theorem}\label{lem::statistics-converge-implies-policy-converge}
Suppose the behavior policy $\pi_t(a|\bX_t, \cH_{t-1})$ takes the form (\ref{eq::policy-statistics}). For any $a\in\cA$, as long as
\begin{itemize}
    \item[(i)] $\hat\bbeta_t\xrightarrow{p} \bbeta^*$ for some $\bbeta^*\in \RR^{d_{\beta}}$,
    \item[(ii)] $\pi(a|\bX_t, \cdot)$ is continuous at $\bbeta^*$ a.e.\ $\bX_t$,
    \item[(iii)] there exists a constant $\pi_{\min}>0$ such that $\pi_t(a\mid \bX_t,\cH_{t-1})\ge \pi_{\min}$ a.s. for all $t\ge1$.
\end{itemize}
Then the behavior policy satisfies scaled inverse-propensity convergence at action $a$ in the sense of Definition~\ref{aspt:policy_convergence}, with $r_{a, t}\equiv 1$, $\rho(a, \bx) = 1/\bar\pi(a|\bx)$, where the limit policy $\bar\pi(a|\bx) = \pi(a|\bx, \bbeta^*)$.
\end{theorem}

\begin{remark}
   Condition~(i) only requires $\hat\bbeta_t$ to converge to a deterministic limit $\bbeta^*$, which need not correspond to any parameter of a correctly specified model nor be optimal in any sense---a feature we will exploit when handling ridge and SGD estimators below. Condition~(ii) imposes only local continuity of the policy mapping at $\bbeta^*$, rather than global continuity; even if $\hat\bbeta_t\xrightarrow{p}\bbeta^*$, a discontinuity at $\bbeta^*$ can cause irregular long-run behavior and prevent convergence. Condition~(iii) imposes a  minimum sampling probability on the policy. For the clipped smooth allocation policies of this section the floor holds by construction~\eqref{eq::clipped-allocation}. 
   
\end{remark}

\begin{remark}
As a byproduct of Theorem~\ref{lem::statistics-converge-implies-policy-converge}, the $\epsilon$-greedy policy with the weighted online least-squares estimator studied by \citet{chen2021statistical} satisfies scaled inverse-propensity convergence with $r_{a,t}\equiv 1$ under their misspecified linear-bandit setting. Thus, their behavior policy falls within our stability framework; see Appendix~\ref{apdx::convergence-epsilon-greedy-weighted-LS} for details.
\end{remark}

By Theorem~\ref{lem::statistics-converge-implies-policy-converge}, it remains only to show that $\hat\bbeta_t\xrightarrow{p} \bar\bbeta$ for some deterministic $\bar\bbeta$. The next two subsections establish this for the smooth allocation policies under the two most common summary statistics: the ridge regression estimator and a stochastic-gradient-descent (SGD) estimator. In both cases, the summary statistic is arm-specific, taking the form $\hat\bbeta_t=(\hat\bbeta_{t,a})_{a\in\cA}$ with $\hat\bbeta_{t,a}\in\mathbb R^{d}$ {(so that the stacked summary statistic lies in $\RR^{d_{\beta}}$ with $d_{\beta} := Kd$)}. In these analyses, we use the common linear action-score specification
\begin{equation}\label{eq::linear-score}
\bs_{\bbeta}(\bx) = (\bbeta_a^\top \bx)_{a\in\cA},\quad \bbeta = (\bbeta_{a})_{a\in\cA}. 
\end{equation}
This linear specification is not essential; the same arguments in the propositions below apply to any action-score map that is bounded and Lipschitz in the summary statistic on bounded regions. The smooth allocation family $\{\brho_\gamma\}_{\gamma>0}$ remains general. 

For both ridge and SGD, our analysis views the joint sequence $\{(\hat\bbeta_t, \pi_t)\}_{t\geq 1}$ as a stochastic-approximation process and exploits the fact that, for a sufficiently smooth allocation function $\brho_\gamma$ (e.g., large enough temperature $\gamma$ in Boltzmann), the induced dynamics are a contraction. Crucially, neither argument requires the linear working model $\bX_t^\top\hat\bbeta_{t, a}\approx \EE[Y_t(a)|\bX_t]$ to be correctly specified: the limit of $\hat\bbeta_t$ is a deterministic vector that need not coincide with any optimal or population parameter, in keeping with the misspecified setting of this paper.

\subsubsection{The ridge estimator}\label{sec::ridge-convergence}

The ridge regression estimator at each time $t$ is defined by $\hat \bbeta_{t}^{\Ridge} = (\hat \bbeta_{t, a}^{\Ridge})_{a\in\cA}$, where
\begin{equation}
    \hat \bbeta_{t, a}^{\Ridge} = \left(\lambda \bI + \sum_{\tau=1}^{t-1} 1_{\{A_\tau = a\}} \bX_\tau \bX_\tau^\top\right)^{-1} \left(\sum_{\tau=1}^{t-1} 1_{\{A_\tau = a\}} \bX_\tau Y_\tau\right). \label{eq:ridge-estimator}
\end{equation}



To prove the convergence of the policy under the ridge estimator, we impose the regularity conditions below.
\begin{assumption}
\label{aspt:Ridge_Conv}
There exist positive constants $M$, $\sigma_{\min}$, $R_y$, $\sigma_{\eta}$ such that (i) $\|\bX_t\|_2 \leq M$ a.s.; (ii) $\EE[\bX_t\bX_t^\top] \succeq \sigma_{\min}^2 \bI$; (iii) For every $a\in\mathcal A$, $Y_t(a)=y(\bX_t,a)+\eta_t(a)$, where $\eta_t(a)$ are independent noise with $\EE[\eta_t(a)] = 0$ and $\EE[\eta_t(a)^2] \leq \sigma_{\eta}^2$. Also, $|y(x,a)|\le R_y<\infty$ for all $\|\bx\|_2 \leq M$ and $a \in \cA$.


\end{assumption}

Under Assumption~\ref{aspt:Ridge_Conv}, the next proposition shows that a policy with smooth allocation mapping (Definition~\ref{def:smooth-allocation}) paired with the ridge estimator satisfies inverse-propensity stability, provided $\gamma$ is sufficiently large.

\begin{proposition}
\label{theorem:convergence-of-Ridge}
Suppose Assumption~\ref{aspt:Ridge_Conv} holds. Consider the clipped smooth-allocation policy~\eqref{eq::clipped-allocation} with $\pi_{\min}\in(0,1/K)$ and summary statistic $\hat\bbeta_{t-1}=\hat\bbeta_t^{\Ridge}$, where $\hat\bbeta_t^{\Ridge}$ is the ridge estimator defined in~(\ref{eq:ridge-estimator}) and the action scores are linear as in~(\ref{eq::linear-score}). Then there exists $\gamma_0>0$ such that, for every $\gamma\ge\gamma_0$, there exists a deterministic vector $\bar\bbeta_\gamma\in\mathbb R^{d_{\beta}}$ satisfying
$$
\hat \bbeta_{t}^{\Ridge} \xrightarrow{p} \bar{\bbeta}_\gamma.
$$
Consequently, the behavior policy satisfies scaled inverse-propensity convergence in Definition~\ref{aspt:policy_convergence} with $r_{a,t}\equiv1$ for every $a\in\cA$.

More specifically, it suffices that $\gamma$ be large enough such that
\begin{align}
&L_{\brho}(\gamma, R)^{-1} \geq {4\sqrt{K}M (M R_y + M\sigma_\eta + M^2)\cdot \max\left\{\dfrac{2}{\pi_{\min}\sigma_{\min}^2},\; \dfrac{4 C_\varphi}{\pi_{\min}^2\sigma_{\min}^4} \right\}},\label{eq:Ridge-Lip-condition}
\end{align}
where $R := \sqrt{K}M (\frac{2 C_\varphi}{\pi_{\min}\sigma_{\min}^2} + 1)$, $C_\varphi := M R_{y} + 4M\sigma_{\eta}\sqrt{dK}$. Such $\gamma$ exists because the right-hand side is a fixed constant while $L_{\brho}(\gamma, R) \to 0$ as $\gamma \to \infty$ by Definition~\ref{def:smooth-allocation}(i).
\end{proposition}
The proof of Proposition \ref{theorem:convergence-of-Ridge} is in Appendix \ref{apdx::proof-thm::convergence-of-Ridge}. The key technical challenge is to identify a quantity tied to $\hat \bbeta_{t}^{\Ridge}$ whose evolution becomes a contraction when $L_{\brho}$ is small. We then apply tools from stochastic approximation theory to show the convergence of $\hat \bbeta_{t}^{\Ridge}$ to a fixed point of the underlying dynamics. 

\paragraph{Convergence considerations for ridge under alternative policies.} Many online contextual bandit algorithms, such as LinUCB \citep{chu2011contextual} and Thompson Sampling \citep{russo2018tutorial}, rely on the ridge estimator as summary statistics for their behavior policy. Proposition \ref{theorem:convergence-of-Ridge} suggests that the ridge estimator converges with data collected under smooth allocation functions with sufficiently small Lipschitz constant (e.g., Boltzmann exploration with large temperature). This conclusion need not hold for other policies. For example, if the action-selection mapping is too steep or noncontinuous, the induced dynamics may not be a contraction, and the sequence $\{\hat \bbeta_{t}^{\Ridge}\}_{t\geq 1}$ may not converge (see the LinUCB example in Section \ref{sec::example-policy-nonconvergence}). Appendix \ref{apdx::ridge-nonconvergence-with-no-fixed-point} provides a rigorous argument for one such mechanism: nonconvergence occurs if the policy mapping induces no fixed point in the dynamics of $\{\hat \bbeta_{t}^{\Ridge}\}_{t\geq 1}$, i.e., there exists an action $\bar a\in\cA$ so that no $\bbeta$ satisfies
\begin{align}
    \bm{\beta}_{\bar a}=\bSigma_{\bar a}^{-1}(\bm{\beta}) \bm{\varphi}_{\bar a}(\bm{\beta}).
\end{align}
Here $\bSigma_{a}(\bm{\beta}) = \EE_{\bX_t, A_t \sim \pi(\cdot \mid \bX_t, \bm{\beta})}[ 1_{\{A_t = a\}} \bX_t \bX_t^\top]$,  $\bm{\varphi}_a(\bm{\beta}) = \EE_{\bX_t, A_t\sim \pi(\cdot \mid \bX_t, \bm{\beta}), Y_t}[ 1_{\{A_t = a\}} \bX_t Y_t ]$. Guided by these considerations, in Section \ref{sec::simulation}, we present a number of challenging environments where common behavior policies that use $\{\hat \bbeta_{t}^{\Ridge}\}_{t\geq 1}$ as summary statistics fail to converge.

\subsubsection{Stochastic gradient descent}\label{sec::SGD}

A second natural choice of summary statistic for a smooth allocation policy is a stochastic gradient descent (SGD) estimator, which is appealing whenever fully refitting a ridge regression at every step is too expensive or when the working model goes beyond linear regression (e.g., logistic regression or a neural-network loss). Let $\hat \bbeta_{t}^{\SGD} = (\hat \bbeta_{t, a}^{\SGD})_{a\in\cA}$ denote the SGD estimator, where for each $a\in\cA$, $\hat \bbeta_{t, a}^{\SGD}$ admits the update rule:
\begin{align}\label{eq::SGD-update}
    \hat \bbeta_{t, a}^{\SGD} = \hat \bbeta_{t-1, a}^{\SGD} + \eta_t 1_{\{A_t = a\}}\bh\left(\bX_t, Y_t; \hat \bbeta_{t-1, a}^{\SGD}\right),
\end{align}
where $\eta_t$ is the learning rate at time $t$ and $\bh: \RR^d \times \RR \to \RR^d$ is a function parameterized by $\hat \bbeta_{t, a}^{\SGD}$. The role of $\bh$ is to specify the stochastic update direction, and in practice it can take the form of the gradient of a loss function (e.g., squared loss in linear models, logistic loss in classification, or a general neural-network loss function). This SGD estimator is widely used in practice especially when complex models such as neural networks are involved in decision-making \citep{zhou2020neural}. 



In this section, we study the stochastic convergence of the SGD estimator via stochastic approximation theory in Proposition \ref{theorem:convergence-of-SGD}. The proof of Proposition \ref{theorem:convergence-of-SGD} is given in Appendix \ref{apdx::proof-thm::convergence-of-SGD}. Proposition \ref{prop::SGD-boundedness} further demonstrates that the requirements of Proposition \ref{theorem:convergence-of-SGD} are satisfied in a broad class of settings.

\begin{proposition}
    \label{theorem:convergence-of-SGD}
    Consider the clipped smooth-allocation policy~\eqref{eq::clipped-allocation} with $\pi_{\min}\in(0,1/K)$ and summary statistic $\hat\bbeta_t=\hat\bbeta_t^{\SGD}$, where $\hat\bbeta_t^{\SGD}$ is the SGD estimator defined recursively from~(\ref{eq::SGD-update}) with $\sum_{t} \eta_t = \infty$, $\sum_{t} \eta_t^2 < \infty$, and the action scores are linear as in~(\ref{eq::linear-score}). Write $\tilde\bh_a(\bbeta):=\EE\big[\pi(a\mid\bX_t,\bbeta)\,\bh(\bX_t, Y_t(a);\bbeta_a)\big]$ and $\tilde\bh:=(\tilde\bh_a)_{a\in\cA}$, and assume:
    \begin{itemize}
        \item[(i)] There is a deterministic constant $R_{\bbeta}<\infty$ such that $\limsup_{t\to\infty}\|\hat \bbeta_{t, a}^{\SGD}\|_2 < R_{\bbeta}$ almost surely, for all $a \in \cA$.
        \item[(ii)] There is a constant $\rho_{\bh}$ such that $\EE[\|\bh(\bX_t, Y_t(a); \bu)\|_2^2] \leq \rho_{\bh} (1+\|\bu\|_2^2) < \infty$ for all $\|\bu\|_2 \leq R_{\bbeta}$ and $a \in \cA$.
        \item[(iii)] $\bu\mapsto\bh(\bx, y;\bu)$ is $L_{\bh}$-Lipschitz, uniformly over $(\bx, y)\in\cX\times\RR$.
        \item[(iv)] \emph{(Dissipativity)} There is $\mu>0$ such that $\langle\bbeta'-\bbeta'',\,\tilde\bh(\bbeta')-\tilde\bh(\bbeta'')\rangle\le -\mu\|\bbeta'-\bbeta''\|_2^2$ for all $\bbeta', \bbeta''\in\cB_{\bbeta}: = \{\bbeta = (\bbeta_a)_{a\in\cA}: \max_a\|\bbeta_a\|_2\leq R_{\bbeta}\}$; the o.d.e.\ $\dot\bbeta=\tilde\bh(\bbeta)$ has an equilibrium $\bar\bbeta_\gamma\in \cB_{\bbeta}$.
    \end{itemize}
    Then there exists a deterministic vector $\bar{\bbeta}_\gamma\in\RR^{d_\beta}$ satisfying
    \[
        \hat \bbeta_{t}^{\SGD} \xrightarrow{p} \bar{\bbeta}_\gamma,
    \]
    and the behavior policy satisfies scaled inverse-propensity convergence in Definition \ref{aspt:policy_convergence}. 
\end{proposition}




Proposition~\ref{theorem:convergence-of-SGD} gives a general convergence criterion for SGD-based smooth-allocation policies. Its conditions are standard in stochastic approximation: the stepsize conditions control stochastic noise, the eventual boundedness condition keeps the iterates in a deterministic region, the growth and Lipschitz assumptions ensure regularity of the updates, and the dissipativity condition gives a stable limiting mean field with a unique equilibrium in that region \citep{borkar2008stochastic, borkar2000ode, tsitsiklis1996analysis, bhandari2018finite}. These conditions hold in many practical settings. A common case arises when the policy statistics are intended to approximate target parameters $\{\btheta_a^*\}_{a\in\cA}$ satisfying (\ref{eq::theta-a-*}) for some score function $\bg$. In this scenario, it is natural to apply the update rule (\ref{eq::SGD-update}) with $\bh(\bx, y;\bbeta_a)=\bg(\bx, y;\bbeta_a)$. Proposition \ref{prop::SGD-boundedness} below verifies that all these conditions are met for the score functions $\bg$ in Examples \ref{ex::misspecified-linear-bandits}, \ref{ex::bandits-noisy-contexts}, and \ref{ex::ope}. The proof, which combines an inward-drift argument for a.s.\ eventual confinement with the affine structure of the example scores, is provided in Appendix \ref{appx::prop::SGD-boundedness}.

\begin{proposition}\label{prop::SGD-boundedness}
    Suppose the stepsizes satisfy $\sum_{t} \eta_t = \infty$ and $\sum_{t} \eta_t^2 < \infty$, Assumption \ref{aspt:Ridge_Conv} holds, and the data are collected by the clipped smooth-allocation policy~\eqref{eq::clipped-allocation} with global floor $\pi_{\min}\in(0,1/K)$. Let $\bh$ be the score function $\bg$ of Example \ref{ex::misspecified-linear-bandits}, \ref{ex::bandits-noisy-contexts}, or \ref{ex::ope}. For example \ref{ex::bandits-noisy-contexts}, additionally assume that 
    $$
    \inf_{a\in\mathcal A}\inf_{\beta}\lambda_{\min}\left\{ \EE\left[\pi^\rho_\gamma(a\mid \bX,\bbeta)(\bX\bX^\top-\bSigma_e)\right]\right\}\ge c_e
    $$
    for a positive constant $c_e$ for sufficiently large $\gamma$ \footnote{A clean sufficient condition is $\pi_{\min}\bSigma_S\succ(1-\pi_{\min})\bSigma_e$, i.e.\ the measurement-error covariance $\bSigma_e$ is small relative to $\bSigma_S$.}. Then for each example, there exist deterministic constants $R_{\bbeta}$ and $\gamma_0$ such that, with $\cB_{\bbeta}=\{\bbeta=(\bbeta_a)_{a\in\mathcal A}: \max_a\|\bbeta_a\|_2\le R_{\bbeta}\}$, for every $\gamma\ge \gamma_0$, there exists an equilibrium $\bar\bbeta_\gamma\in \cB_{\bbeta}$ and all conditions of Proposition \ref{theorem:convergence-of-SGD} hold. Consequently $\hat\bbeta_{t,a}^{\SGD}\xrightarrow{p}\bar\bbeta_{\gamma, a}$ for all $a\in\cA$, and the behavior policy satisfies scaled inverse-propensity convergence (Definition \ref{aspt:policy_convergence}).
\end{proposition}



\section{Simulation Studies}
\label{sec::simulation}

In this section, we conduct simulation studies to evaluate the proposed inference methods under three previously introduced inference targets: Target \ref{ex::misspecified-linear-bandits} (misspecified linear bandits with score function (\ref{eq::target-parameter-misspecified-linear-bandits})), Target \ref{ex::bandits-noisy-contexts} (linear bandits with noisy contexts with score function (\ref{eq::target-parameter-bandits-noisy-contexts})), and Target \ref{ex::ope} (off-policy evaluation with score function (\ref{eq::target-parameter-ope})).

For Target \ref{ex::bandits-noisy-contexts}, we design three simulation environments: \texttt{NC-Hard1}, \texttt{NC-Hard2}, \texttt{NC-Gaussian}. In these environments, the reward function w.r.t. the underlying true context is linear, and the model misspecification arises from noisy context. Specifically, \texttt{NC-Hard1}, \texttt{NC-Hard2} are deliberately designed challenging environments so that the ridge estimator either oscillates or tends to multiple limits under common behavior policies such as LinUCB. In \texttt{NC-Gaussian}, all variables and parameters are jointly Gaussian. For Targets \ref{ex::misspecified-linear-bandits} and \ref{ex::ope}, in addition to the three environments above, we also include two environments \texttt{MS-Polynomial} and \texttt{MS-Neural}, where the potential outcomes depend directly on the contexts through complex non-linear relationships: a polynomial reward function in \texttt{MS-Polynomial} and a neural network reward function in \texttt{MS-Neural}. 

We evaluate the inference methods using datasets generated by four distinct data-collection algorithms: pure random selection (\texttt{Random}), softmax (Boltzmann) exploration with the ridge estimator (\texttt{ridge + Softmax}), an $\varepsilon$-greedy algorithm with empirical mean reward of each arm and decaying minimum sampling probability $\pi_{\min,t} = 1/\sqrt{t+1}$ (\texttt{MAB-EG}), and softmax (Boltzmann) exploration with the stochastic gradient descent estimator (\texttt{SGD + Softmax}). Throughout, we use linear working models in the behavior policies, despite the true mean reward function $y(\bx,a) \coloneqq \EE[Y_t(a) \mid \bX_t = \bx]$ being non-linear, allowing us to assess the robustness of the inference methods under misspecifications. These algorithms are shown to have policy convergence in Section \ref{sec::policy-convergence}. The two smooth allocation policies, \texttt{ridge + Softmax} and \texttt{SGD + Softmax}, correspond to the Boltzmann exploration policies analyzed in Section~\ref{sec::policy-convergence}, with ridge and SGD used as the policy summary statistics, respectively. Both policies use a fixed inverse temperature $\gamma = 1$ and are clipped at a fixed floor $\pi_{\min} = 0.05$, matching the clipped smooth-allocation construction in Section~\ref{sec::policy-convergence}. In contrast, \texttt{MAB-EG} uses the decaying floor $\pi_{\min,t} = 1/\sqrt{t+1}$. The clipping prevents the propensities from becoming too small and is particularly important for \texttt{SGD + Softmax}, whose unregularized stochastic-gradient updates could otherwise saturate the softmax and drive some action probabilities toward zero.

\subsection{Results for Targets \ref{ex::misspecified-linear-bandits} and \ref{ex::bandits-noisy-contexts}}

We first examine the empirical coverage of nominal confidence intervals (95\%, 90\%, 80\%, 70\%, 60\%, and 50\%) for inference target \ref{ex::misspecified-linear-bandits} across the five environments, based on 2{,}500 Monte Carlo simulations. 
The full environment and algorithm settings are deferred to Appendix~\ref{apdx::simulation-details}. Figure \ref{fig::results-misspecified-linear-bandits} shows that the proposed inference method consistently achieves the desired coverage across all algorithms and environments.
The clipped smooth allocation policies (\texttt{ridge + Softmax}, \texttt{SGD + Softmax}) and \texttt{Random} are near nominal in every environment (their $95\%$ coverage lies between $0.93$ and $0.95$). The main deviations occur for \texttt{MAB-EG}, which uses a decaying exploration floor and therefore has heavier inverse-propensity weights. In \texttt{NC-Gaussian} and \texttt{MS-Polynomial}, \texttt{MAB-EG} mildly undercovers at the $95\%$ level, reflecting slower convergence of the normal approximation; in \texttt{NC-Hard1}, it is conservative, due to wider intervals from the plug-in variance estimator.
Similar trends appear for Target \ref{ex::bandits-noisy-contexts} in the three noisy context environments (Figure \ref{fig::results-bandits-noisy-contexts}, panels (a) and (b)).

We also track the variance of the IPW-Z estimator across algorithms and environments; the variances stabilize at comparable levels, with the smallest variances attained by the algorithms that keep action selection probabilities bounded away from zero, consistent with Section \ref{sec::boltzmann-policy-convergence}. We defer the details to Appendix \ref{apdx::simulation-details::additional-results}.

\begin{figure}[tb]
    \centering
    \includegraphics[width=1\textwidth]{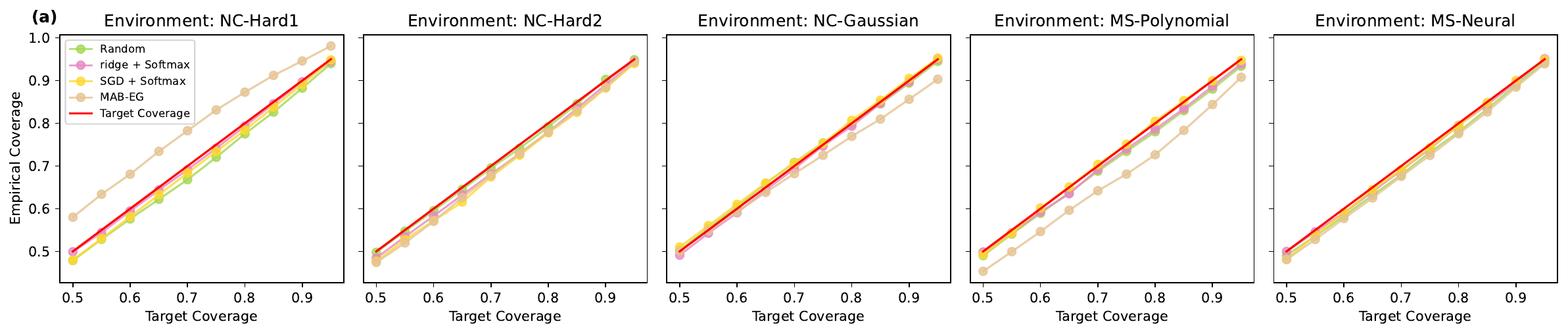}
    \includegraphics[width=1\textwidth]{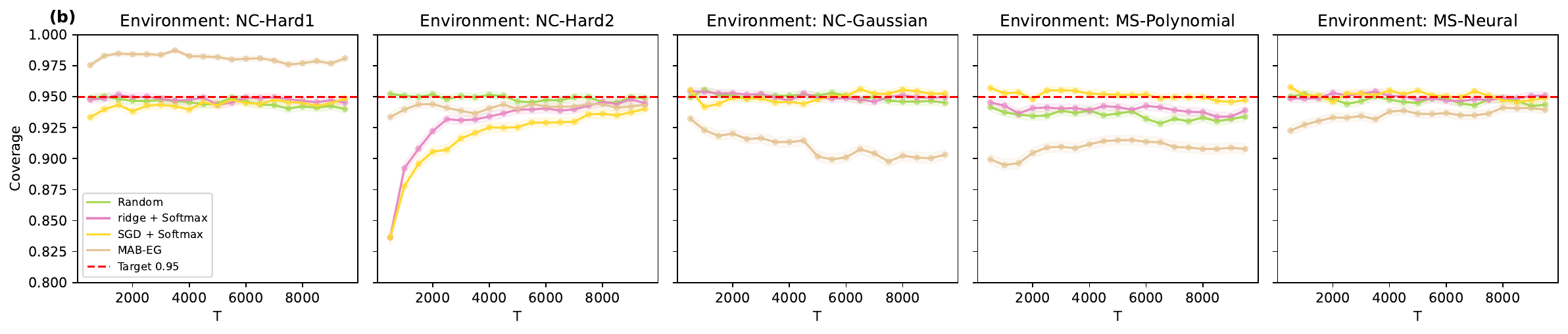}
    \caption{\textbf{(a)} Empirical coverages of 95\%, 90\%, 80\%, 70\%, 60\%, and 50\% confidence intervals vs. the target coverage for Target \ref{ex::misspecified-linear-bandits}. \textbf{(b)} Empirical coverages of 95\% confidence interval over 10,000 steps. Results averaged across 2,500 Monte Carlo simulations. Error bars/shaded bands denote $\pm 2$ Monte Carlo standard errors.}
    \label{fig::results-misspecified-linear-bandits}
\end{figure}

\begin{figure}[tb]
    \centering
    \includegraphics[width=0.8\textwidth]{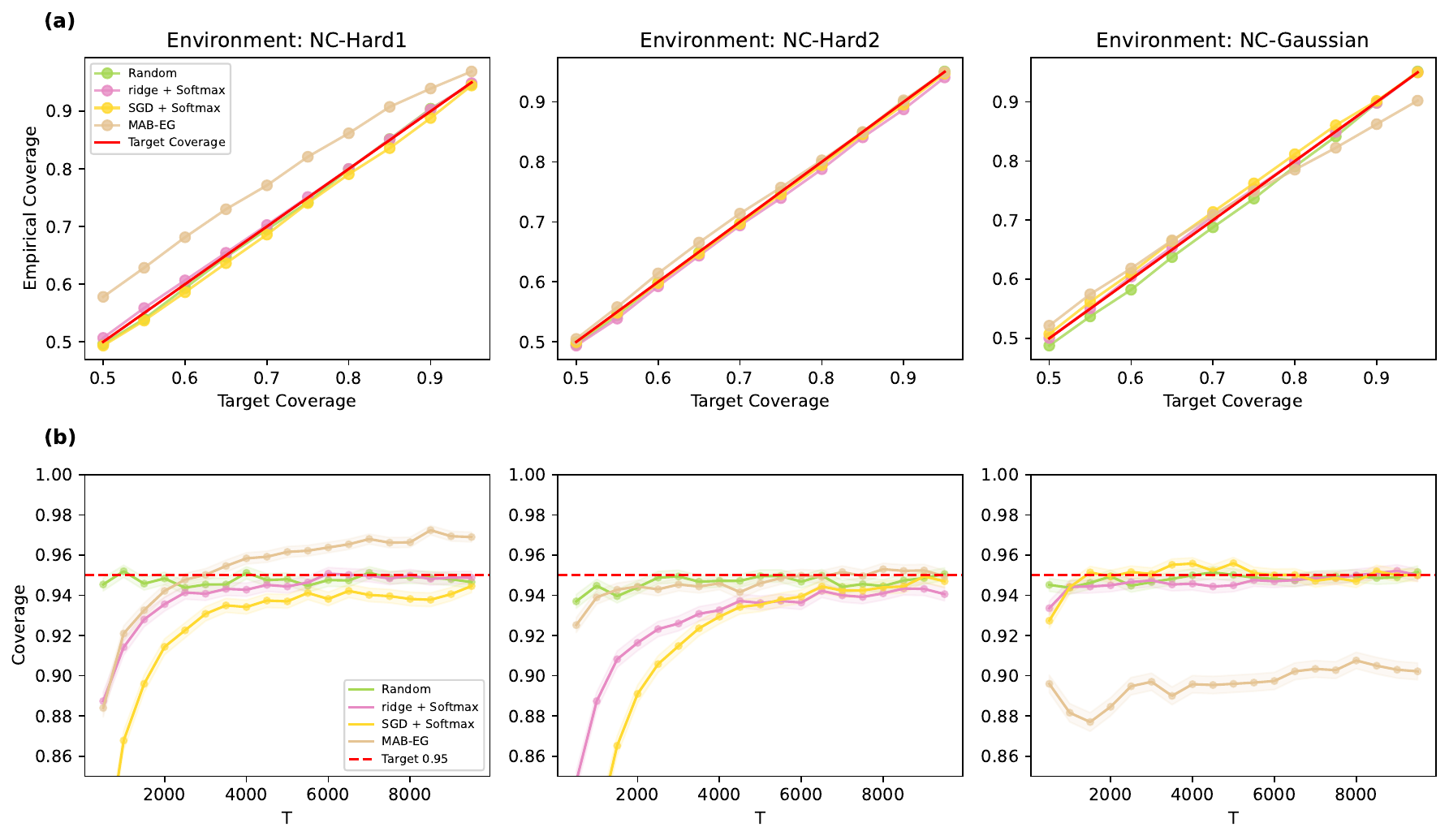}
    \caption{\textbf{(a)} Empirical coverages of 95\%, 90\%, 80\%, 70\%, 60\%, and 50\% confidence intervals vs. the target coverage for Target \ref{ex::bandits-noisy-contexts}. \textbf{(b)} Empirical coverages of 95\% confidence interval over 10,000 steps under three noisy context environments. {Error bars/shaded bands denote $\pm 2$ Monte Carlo standard errors.}}
    \label{fig::results-bandits-noisy-contexts}
\end{figure}

\subsection{Results for Target \ref{ex::ope}}\label{sec::simulations-ope}

In the OPE setting, we compare our proposed inference method with the CADR (Contextual Adaptive Doubly Robust) method \citep{bibaut2021post} and the Contextual Adaptive Weight (AW) method implemented via StableVar \citep{zhan2021off}. Both are stabilized doubly robust estimators, where we select the prediction model from a linear model, a tree-based model, or a dummy model that always outputs 0. We run all the inference methods on the same datasets collected by Boltzmann exploration w.r.t. Ridge regression estimator in five environments introduced above. In Figure \ref{fig::results-ope} (a), we show that CADR, AW and our proposed inference method all achieve empirical coverage close to the target coverage. In Figure \ref{fig::results-ope} (b), our proposed estimator has lower variance, beyond Monte Carlo error, under \texttt{NC-Hard2} and \texttt{MS-Polynomial}, and is statistically indistinguishable from CADR and AW in the others. 
We do not claim a uniform efficiency advantage. One possible explanation for the gains observed in these environments is the additional variability introduced by the estimated variance-stabilization terms and adaptively fitted regression models used in CADR and AW.


\begin{figure}[tb]
    \centering
    \includegraphics[width=1\textwidth]{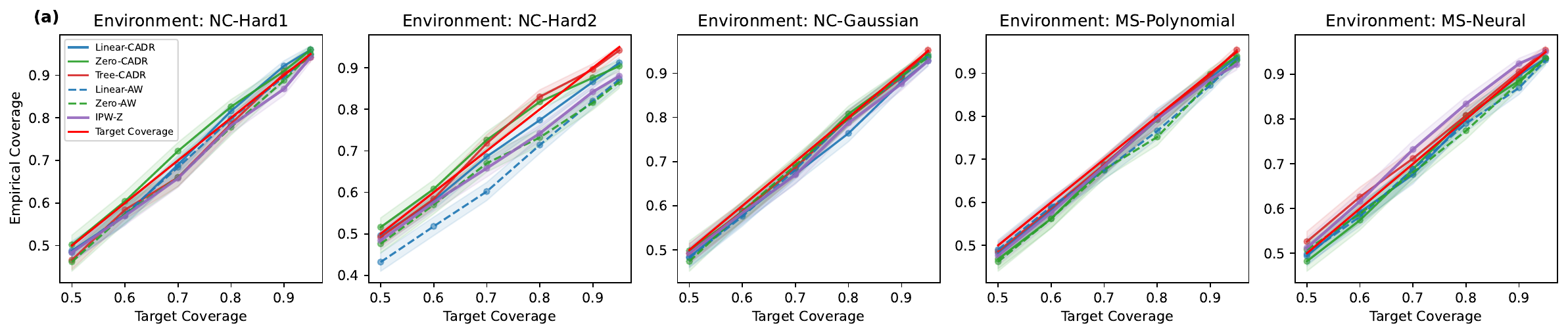}
    \includegraphics[width=1\textwidth]{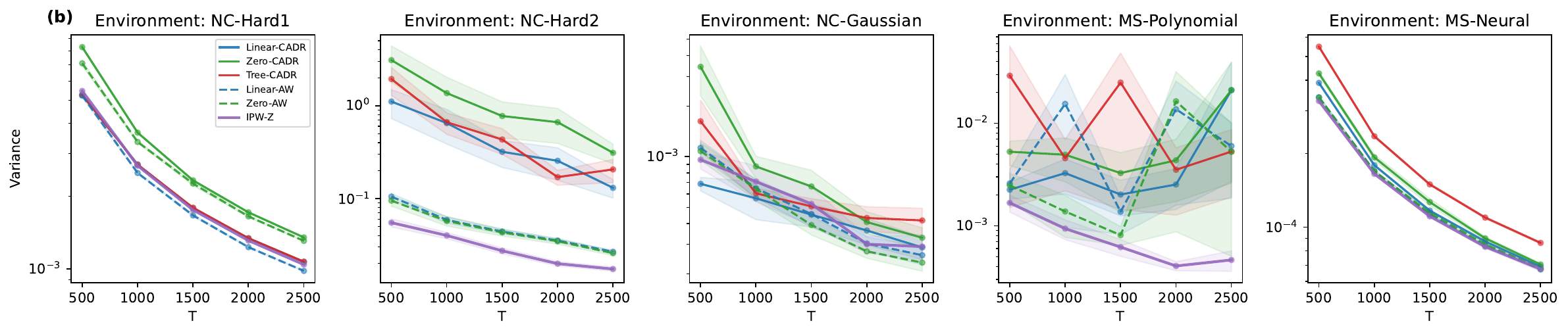}
    \caption{\textbf{(a)} Empirical coverage for confidence intervals based on CADR and AW through StableVar under different prediction models and our proposed inference method across five environments. \textbf{(b)} Monte Carlo estimates (based on 2,500 samples) of the variance of OPE target estimator based on CADR and AW through StableVar under different prediction models and our proposed inference method across five environments over 2,500 steps. Error bars denote $\pm 2$ Monte Carlo standard errors.}
    \label{fig::results-ope}
\end{figure}

\subsection{Real-Data Application: the HeartSteps V1 Mobile Health Study}\label{sec::simulations-heartsteps}

We close with a semi-synthetic study calibrated to real data, demonstrating that the proposed inference is valid in a setting drawn from an actual digital-intervention trial.
We build on the HeartSteps V1 simulation environment developed in \citep{guo2024online}.
HeartSteps V1 \citep{dempsey2015randomised,klasnja2015microrandomized,liao2016sample} is a 42-day mobile health trial, where participants are provided a mobile tracker and a mobile phone application.
One of the intervention components is a contextually-tailored physical activity suggestion that may be delivered at any of the five user-specified times during each day.
The delivery times are roughly separated by 2.5 hours.
The true context at time $t$ is $\bS_t\in\RR^{7}$, consisting of an intercept together with six logged covariates recorded in the trial: the step count in the $30$ minutes preceding the decision point, the cumulative step count, the ambient temperature at the decision point, the recent variability of the step count, an indicator for being at home, and the accumulated treatment burden, the \emph{dosage}, which summarizes the participant's recent treatment history.
The dosage, indexed below by $j$, is the covariate of scientific interest; the exact column definitions are given in Appendix~\ref{sec::simulations-heartsteps-details}.
Following \citep{guo2024online}, we sample the context $\bS_t$---including the dosage---i.i.d.\ from the pooled empirical distribution of the trial participants, collecting every decision point of all $37$ users into a single bank and drawing one record with replacement at each time $t$.
The potential outcome is linear in the true context, $Y_t(a) = \bS_t^\top \btheta_a^* + \eta_t$, with mean-zero reward noise $\eta_t$ and per-action parameters $\btheta_a^* \in \RR^{7}$ fixed at the generalized-estimating-equation fit of \citep{liao2016sample} on the real outcomes.
The agent does not observe the true context $\bS_t$; it observes only the noisy proxy $\bX_t = \bS_t + \bepsilon_t$, where the measurement error $\bepsilon_t$ perturbs only the dosage coordinate and is drawn i.i.d.\ with mean zero and variance $\sigma_\epsilon^2$, aligning the application with the noisy-context misspecification of Target~\ref{ex::bandits-noisy-contexts}.
This reflects practice in mobile health, where the burden is itself a noisy prediction rather than a directly observed quantity.

Our inferential target is the \emph{dosage-by-treatment interaction} $\tau^* = \theta_{1,j}^* - \theta_{0,j}^*$, where $j$ indexes the dosage coordinate: it quantifies how the accumulated burden moderates the effect of sending a suggestion, the scientifically central question of habituation in just-in-time interventions.
We collect data with the same four behavior policies as above (\texttt{Random}, \texttt{ridge + Softmax}, \texttt{MAB-EG}, and \texttt{SGD + Softmax}) and form confidence intervals for the interaction using the IPW-Z estimator, estimating the true-covariate second moment that enters the asymptotic variance from the observed contexts rather than treating it as known.
Figure~\ref{fig::results-heartsteps}(a) shows that the empirical coverage tracks the nominal level across all four behavior policies and all target levels, and panel (b) shows the $90\%$ coverage holding near the nominal value over $T$.
At the final horizon ($T = 10{,}000$, $2{,}000$ replications), the empirical $95\%$ coverage is within Monte Carlo error of the nominal level for every policy ($0.946$ for \texttt{Random}, $0.946$ for \texttt{ridge + Softmax}, $0.947$ for \texttt{MAB-EG}, and $0.951$ for \texttt{SGD + Softmax}).
The method is also robust to the amount of measurement error: as the dosage-noise variance increases from $0.1$ to $0.25$ to $0.5$, the $95\%$ coverage remains nominal ($0.947$, $0.947$, $0.956$) while the confidence intervals widen accordingly.
Thus the proposed inference remains valid on a real-data-calibrated adaptive experiment with measurement error, complementing the synthetic studies above.
The full details of the environment, the estimand and its score, the IPW-Z estimator and its variance estimator, and the behavior-policy configuration are deferred to Appendix~\ref{sec::simulations-heartsteps-details}.

\begin{figure}[tb]
    \centering
    \includegraphics[width=0.9\textwidth]{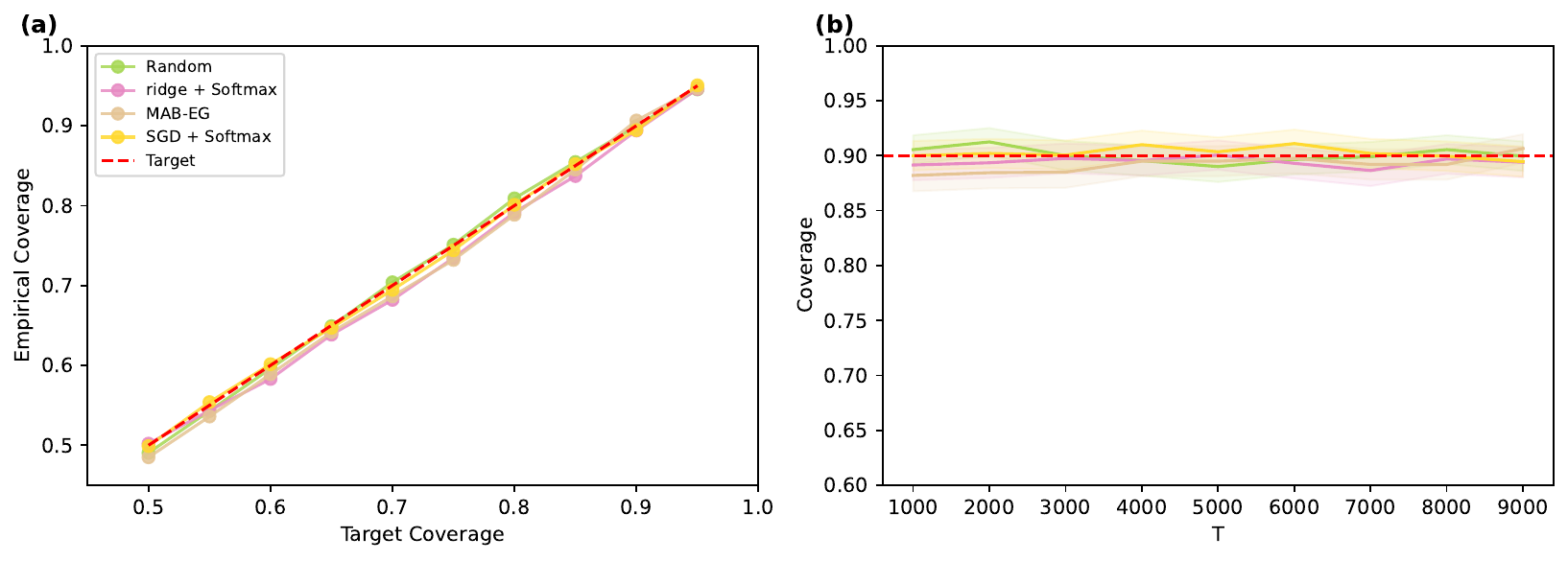}
    \caption{\textbf{(a)} Empirical coverage versus target level for the dosage-by-treatment interaction in the HeartSteps V1 environment, at the final horizon, across four behavior policies. \textbf{(b)} Empirical $90\%$ coverage over $T$. Results averaged across 2{,}000 Monte Carlo simulations; shaded bands denote $\pm 2$ Monte Carlo standard errors.}
    \label{fig::results-heartsteps}
\end{figure}

\newpage

\subsection{Acknowledgment} 
Yongyi Guo acknowledges support from the National Science Foundation under Grant DMS-2515285. Ziping Xu acknowledges support from the following NIH grants during his postdoctoral fellowship at Harvard University: NIH/NIDA P50DA054039, NIH/NIBIB and OD P41EB028242, NIH/NIDCR UH3DE028723, and NIH/NIA 5P30AG073107-03 GY3 Pilots.

\newpage

\newpage
\appendix

\section{Additional Technical Details}\label{appAA}

\subsection{The Clipping Operator}\label{apdx::clipping}

In this section, we present in details a useful transformation, 
$$\mathrm{Clip}: \RR^K\times\RR\rightarrow \RR^K,$$ which maps a probability distribution over $K$ discrete actions to a new distribution to ensure that each coordinate is lower bounded by $\pi_{\min}$. Specifically, for $\bm{\pi}\in [0, 1]^{K}$, define 
\begin{align}\label{eq::clip}
    \operatorname{Clip}(\bm{\pi}; \pi_{\min}) = \max\{\bm{\pi} - \nu^*(\bm{\pi}), \pi_{\min}\},
\end{align}
where $\nu^*(\bm{\pi})$ is defined as the unique value such that 
$$
q(\nu; \bm{\pi}) \coloneqq \sum_{a\in\cA} \max\{\bm{\pi}_a - \nu, \pi_{\min}\} = 1.
$$
This transformation adjusts $\bpi$ to remain a valid probability distribution while ensuring that each component is at least $\pi_{\min}$. The following lemma shows that this operation is in fact the $L_2$ projection onto the constrained simplex, providing a principled justification for its use. The proof is in Appendix \ref{apdx::proof-lem::clip-l2-projection}.
\begin{lemma}\label{lem::clip-l2-projection}
For a fixed $\pi_{\min}$, the mapping $\bm{\pi}\mapsto\operatorname{Clip}(\bm{\pi};\pi_{\min})$ is the $L_2$ projection of $\bm{\pi}$ onto the set 
$$\left\{ \bm{\pi} \in [0, 1]^{|\mathcal{A}|} \bigg| \sum_{a\in\cA} \bm{\pi}_a = 1, \bm{\pi}_a \geq \pi_{\min}\right\}.$$
\end{lemma}


{
\subsection{Stability of the Policy in \citep{chen2021statistical}}\label{apdx::convergence-epsilon-greedy-weighted-LS}
Consider behavior policies of the form (\ref{eq::policy-statistics}), where the statistics $\hat\bbeta_{t-1}$ consist of the IPW-Z estimators defined by (\ref{eq::estimating-equation-general}), together with additional algorithm parameters that converge over time:
\begin{equation}\label{eq::IPWZ-policy-plus-parameter}
    \pi(\cdot \mid \bX_t, \{\hat\btheta_a^{(t-1)}\}_{a \in \cA}, \bgamma_{t-1}).
\end{equation}
Here, $\{\bgamma_t\}_{t \geq 1}$ is a sequence of parameters governing the policy. For example, in the case of an $\epsilon$-greedy policy, $\bgamma_t$ can represent the exploration probability at time $t$. The following proposition shows that, under mild conditions, such behavior policies converge to a limiting policy parameterized by $\btheta^*=(\btheta_a^*)_{a\in\cA}$.

\begin{proposition}\label{thm::convergence-of-IPWZ-policy}
Under the assumptions of Theorem \ref{thm::asymptotic-normality-joint-general}, suppose the behavior policy at each time $t$ takes the form (\ref{eq::IPWZ-policy-plus-parameter}), where $\{\hat\btheta_a^{(t)}\}_{a\in\cA, t\geq 1}$ is the IPW-Z estimator defined by (\ref{eq::estimating-equation-general}), and $\{\bgamma_{t}\}_{t\geq 1}$ is a deterministic parameter sequence such that $\bgamma_{t}\in \RR^{d_{\gamma}}$,  $\lim_{t\rightarrow \infty}\bgamma_{t} = \bgamma^*$ for some $\bgamma^*\in\RR^{d_{\gamma}}$. Assume 
$
\pi:\cX\times \RR^{Kd}\times\RR^{d_{\gamma}}\rightarrow \Delta(\cA)
$
is continuous in its second and third argument at $(\btheta^*, \bgamma^*)$ for almost every $\bx\in\cX$ under the distribution of $\bX_t$. Then the behavior policy satisfies inverse propensity stability in Definition \ref{aspt:policy_convergence}, with $r_{a, t}\equiv 1$, $\rho(a, \bx) = 1/\bar\pi(a|\bx)$, and the limit policy $\bar\pi(a|\bx) = \pi(a|\bx, \btheta^*, \bgamma^*)$.
\end{proposition}

\begin{proof}
Let $\hat\bbeta_t = (\hat\btheta^{(t)}, \bgamma_t)$, where $\hat\btheta^{(t)} = (\hat\btheta_a^{(t)})_{a\in\cA}$. Using the same arguments in the proof of Lemma \ref{lem::consistency-general}, we deduce that $\hat\bbeta_{t}\xrightarrow{p} \bbeta^*:= (\btheta^*, \bgamma^*)$ (note that the proof of Lemma \ref{lem::consistency-general} does not require the behavior policy to converge), which implies that condition (i) of Theorem \ref{lem::statistics-converge-implies-policy-converge} holds in this setting. Under an additional condition of policy continuity, Theorem \ref{lem::statistics-converge-implies-policy-converge} can then be invoked to establish policy convergence.
\end{proof}

Now we analyze a behavior policy studied in \citep{chen2021statistical}, which combines an $\epsilon$-greedy algorithm with a weighted online LS estimator, in the setting of misspecified linear bandits. The policy can be written as
\begin{equation}\label{eq::policy-epsilon-greedy-weighted-LS}
    \pi_t(a|\bX_t, \cH_{t-1}) = (1-\epsilon_t)1_{\{(\hat\bbeta_a^{(t-1)}-\hat\bbeta_{1-a}^{(t-1)})^\top\bX_t>0\}} + \frac{\epsilon_t}{2}
\end{equation}
for $a\in\cA = \{0, 1\}$. Here, $\widehat{\bbeta}_a^{(t-1)}$ denotes the weighted online LS estimate of the reward parameter for arm $a$ using data up to time $t-1$, which is a special case of the IPW-Z estimator $\hat{\btheta}_a^{(t-1)}$ defined by (\ref{eq::estimating-equation-general}) with the score function $\bg$ given by (\ref{eq::target-parameter-misspecified-linear-bandits}). $\epsilon_t\in (0, 1)$ is a time-varying, nonincreasing exploration parameter such that $\lim_{t\rightarrow \infty}\epsilon_t =\epsilon_{\infty}>0$. 

It is straightforward to verify the above policy is of the form (\ref{eq::IPWZ-policy-plus-parameter}) with $\hat{\btheta}_a^{(t-1)} = \hat{\bbeta}_a^{(t-1)}$,  $\bgamma_{t-1} = \epsilon_t$, and 
$$
\pi(a|\bx, \{\btheta_a\}_{a\in\cA}, \bgamma)= (1-\bgamma)1_{\{(\btheta_a-\btheta_{1-a})^\top\bx>0\}} + \frac{\bgamma}{2}.
$$
The convergence of $\{\bgamma_{t}\}_{t\geq 1}$ is guaranteed by the convergence of $\{\epsilon_{t}\}_{t\geq 1}$. In addition, the function $\pi$ is continuous in $(\{\btheta_a\}_{a\in\cA}, \bgamma)$ at the point $(\{\btheta_a^*\}_{a\in\cA}, \epsilon_{\infty})$ for any $\bx\in\cX_0:= \{\bx: (\btheta_1^* - \btheta_0^*)^\top\bx \neq 0\}$. Here for $a\in\cA$, $\btheta_a^*$ is defined in (\ref{eq::theta-a-*}) with the score function $\bg$ given by (\ref{eq::target-parameter-misspecified-linear-bandits}). Note that $\cX_0^c$ is a Lebesgue null set, and from Assumption 1 of \citep{chen2021statistical}, we deduce that
$$
\PP(\cX_0^c) = 0.
$$
Combining the above arguments, we have verified the conditions of Proposition \ref{thm::convergence-of-IPWZ-policy}, which implies that the policy (\ref{eq::policy-epsilon-greedy-weighted-LS}) satisfies inverse-propensity stability in Definition \ref{aspt:policy_convergence}. 
}

\subsection{Nonconvergence of the Ridge Estimator When the Evolution Lacks a Fixed Point}\label{apdx::ridge-nonconvergence-with-no-fixed-point}


The following theorem states that, if the behavior policy is parameterized using the ridge estimator, and the policy mapping induces no fixed point in the evolution dynamics of the ridge statistics, then the ridge estimator will not converge. 

\begin{theorem}\label{theorem:necessary-condition-for-Ridge-convergence}
Assume the conditions of Proposition \ref{theorem:convergence-of-Ridge}. Suppose that the behavior policy takes the form $\pi_t(a|\bX_t, \cH_{t-1}) = \pi(a|\bX_t, \hat \bbeta_{t}^{\Ridge})$, where $\hat \bbeta_{t}^{\Ridge}$ is the ridge estimator defined in (\ref{eq:ridge-estimator}), and the policy mapping $\pi: \cX\times \RR^{Kd}\rightarrow\Delta(\cA)$ satisfies:
\begin{itemize}
\item $\pi(a \mid \bx, \bm{\beta})$ is a continuous function of $\bm{\beta}$ for any $\bx \in \cX$,

\item $\pi(a \mid \bx, \bm{\beta})>0$ for any $\bx \in \cX, \bm{\beta}\in\RR^{Kd}$,

\item There exists $\bar a\in\cA$, such that for any $\bm{\beta} = (\bbeta_1^\top, \ldots, \bbeta_K^\top)^\top\in \mathbb{R}^{Kd}$,  
    \begin{align}
    \bm{\beta}_{\bar a} \neq \bSigma_{\bar a}^{-1}(\bm{\beta}) \bm{\varphi}_{\bar a}(\bm{\beta}). \label{eq:necessary-condition-for-Ridge-convergence}
    \end{align}
Here $\bSigma_{a}(\bm{\beta}) = \EE_{\bX_t, A_t \sim \pi(\cdot \mid \bX_t, \bm{\beta})}[ 1_{\{A_t = a\}} \bX_t \bX_t^\top]$, and \\  $\bm{\varphi}_a(\bm{\beta}) = \EE_{\bX_t, A_t\sim \pi(\cdot \mid \bX_t, \bm{\beta}), Y_t}[ 1_{\{A_t = a\}} \bX_t Y_t ]$.
\end{itemize}
Then the ridge estimator $\hat \bbeta_{t}^{\Ridge}$ does not converge in probability as $t\rightarrow\infty$.
\end{theorem}

The proof is given in Appendix \ref{apdx::proof-thm::necessary-condition-for-Ridge-convergence}.

\subsection{Asymptotic Stability}\label{apdx::asymptotic-stability}

\begin{definition}[Asymptotic Stability]
    \label{def::asymptotic-stability}
    The equilibrium point $x=0$ of $\dot{x} = f(x)$ is 
    \begin{itemize}
        \item stable if, for each $\varepsilon>0$, there is $\delta=\delta(\varepsilon)>0$ such that
        $$
            \|x(0)\|<\delta \Rightarrow\|x(t)\|<\varepsilon, \quad \forall t \geq 0
        $$
        \item unstable if it is not stable.
        \item asymptotically stable if it is stable and $\delta$ can be chosen such that
        $$
            \|x(0)\|<\delta \Rightarrow \lim _{t \rightarrow \infty} x(t)=0
        $$
    \end{itemize}
\end{definition}

\section{Proofs of Main Theorems}\label{appA}
We organize the proofs of the main theorems by sections in the main text.

    
    {
    \subsection{Proof of Theorem \ref{thm::asymptotic-normality-general}}\label{apdx::proof-thm::asymptotic-normality-general}
    
    We first prove the following lemma. Its proof is in Appendix \ref{apdx::proof-lem::consistency-general}.
    
    \begin{lemma}\label{lem::consistency-general}
        Under the assumptions of Theorem \ref{thm::asymptotic-normality-general}, there exists a sequence of estimators $\{\hat\btheta_a^{(T)}\}_{T\geq 1}$ such that (\ref{eq::estimating-equation-general}) holds, and $\|\hat\btheta_a^{(T)}\|_2\leq R_{\btheta}$, $\forall T$. In addition, for any such sequence, as $T\rightarrow \infty$, $\hat\btheta_a^{(T)}\xrightarrow{p}\btheta_a^*$.
    \end{lemma}
    
    For the remainder of this proof, for notational convenience, we omit the dependence of $\hat\btheta_a^{(T)}$ on $T$ and write it as $\hat\btheta_a$. Let $\bG_T^{(i)}(\btheta)$ denote the $i$-th entry of $\bG_T(\btheta)$. By Taylor expansion, we have that for any $i\in\{1, \ldots, d\}$, there exists some $\tilde\btheta_{a, i}$ on the line segment between $\btheta_a^*$ and $\hat\btheta_a$ such that  
    \begin{align*}
    -\bG_T^{(i)}(\btheta_a^*)&= \bG_T^{(i)}(\hat\btheta_a) - \bG_T^{(i)}(\btheta_a^*)+o_p(1/b_{a, T})\\ 
    &= \langle\nabla\bG_T^{(i)}(\btheta_a^*), \hat\btheta_a-\btheta_a^*\rangle + \frac12(\hat\btheta_a-\btheta_a^*)^\top\nabla^2\bG_T^{(i)}(\tilde\btheta_{a, i})(\hat\btheta_a-\btheta_a^*)+o_p(1/b_{a, T}).
    \end{align*}
    Stacking the above expansions over the entries $i=1, \ldots, d$, we have 
    $$
    -\bG_T(\btheta_a^*)= \nabla\bG_T(\btheta_a^*)(\hat\btheta_a-\btheta_a^*) + \frac12\tilde{\bm{\delta}}_a(\hat\btheta_a-\btheta_a^*)+o_p(1/b_{a, T}),
    $$
    where 
    $$
    \tilde{\bm{\delta}}_a = 
    \begin{pmatrix}
    (\hat\btheta_a-\btheta_a^*)^\top\nabla^2\bG_T^{(1)}(\tilde\btheta_{a, 1})\\
    \vdots\\
    (\hat\btheta_a-\btheta_a^*)^\top\nabla^2\bG_T^{(d)}(\tilde\btheta_{a, d})
    \end{pmatrix}.
    $$
    By rearranging, we obtain 
    \begin{equation}\label{eq::taylor-expansion-general}
        \big[\nabla\bG_T(\btheta_a^*) + \frac12\tilde{\bm{\delta}}_a \big]\cdot b_{a, T}(\hat\btheta_a - \btheta_a^*) = -b_{a, T}\bG_T(\btheta_a^*) + o_p(1).
    \end{equation}
    Below we state the following lemmas, the proof of these lemmas are in Appendix \ref{apdx::proof-lem::convergence-true-derivative-general}, \ref{apdx::proof-lem::2nd-derivative-boundness-general}, and \ref{apdx::proof-lem::asymptotic-normality-true-G-general}, respectively.
    
    \begin{lemma}\label{lem::convergence-true-derivative-general}
        Under the assumptions of Theorem \ref{thm::asymptotic-normality-general}, as $T\rightarrow \infty$, $\nabla\bG_T(\btheta_a^*)\xrightarrow{p}\EE\nabla\bg(\bX_t, Y_t(a);\btheta_a^*)$.
    \end{lemma}
    
    \begin{lemma}\label{lem::2nd-derivative-boundness-general}
        Under the assumptions of Theorem \ref{thm::asymptotic-normality-general}, $\sup_{\|\btheta - \btheta_a^*\|_2\leq \epsilon_0}\|\nabla^2\bG_T(\btheta)\|_1 = \cO_p(1)$. Here, for a tensor $\bB\in\RR^{d_1\times d_2\times d_3}$, we define $\|\bB\|_1 = \sum_{i\in[d_1], j\in[d_2], k\in[d_3]}|\bB_{i, j, k}|$.
    \end{lemma}
    
    \begin{lemma}\label{lem::asymptotic-normality-true-G-general}
        Under the assumptions of Theorem \ref{thm::asymptotic-normality-general}, as $T\rightarrow \infty$, $b_{a, T}\bG_T(\btheta_a^*)\xrightarrow{d} \cN(\mathbf{0}, \bar\bI_a)$, where $\bar\bI_a = \EE\big[\rho(a, \bX_t)\bg(\bX_t, Y_t(a);\btheta_a^*)\bg(\bX_t, Y_t(a);\btheta_a^*)^\top\big]$.
    \end{lemma}
    
    We now derive the asymptotic distribution of $\hat\btheta_a$ from (\ref{eq::taylor-expansion-general}). First, since $\tilde\btheta_{a, i}$ is on the line segment between $\btheta_a^*$ and $\hat\btheta_a$, from Lemma \ref{lem::consistency-general} and Lemma \ref{lem::2nd-derivative-boundness-general}, we have $\forall i$,
    \begin{align*}
        \|\nabla^2\bG_T^{(i)}(\tilde\btheta_{a, i})\|_{1, 1} &\leq  \|\nabla^2\bG_T(\tilde\btheta_{a, i})\|_1\\
        &= \|\nabla^2\bG_T(\tilde\btheta_{a, i})\|_1\cdot1_{\{\|\tilde\btheta_{a, i}-\btheta_a^*\|_2\leq \epsilon_0\}} + \|\nabla^2\bG_T(\tilde\btheta_{a, i})\|_1\cdot1_{\{\|\tilde\btheta_{a, i}-\btheta_a^*\|_2> \epsilon_0\}}\\
        &\leq \sup_{\|\btheta - \btheta_a^*\|_2\leq \epsilon_0}\|\nabla^2\bG_T(\btheta)\|_1 + \|\nabla^2\bG_T(\tilde\btheta_{a, i})\|_1\cdot 1_{\{\|\hat\btheta_a-\btheta_a^*\|_2> \epsilon_0\}}= \cO_p(1).
    \end{align*}
    Here for a matrix $\bB\in\RR^{d_1\times d_2}$, we define $\|\bB\|_{1, 1} = \sum_{i\in[d_1], j\in[d_2]}|\bB_{i, j}|$. 
    
    Combine the above with Lemma \ref{lem::consistency-general} which implies $\hat\btheta_a - \btheta_a^* = o_p(1)$ and Lemma \ref{lem::convergence-true-derivative-general} which ensures convergence of $\nabla\bG_T(\btheta_a^*)$, we deduce that 
    \begin{equation}\label{eq::gradient-plus-convergence-general}
        \nabla\bG_T(\btheta_a^*) + \frac12\tilde{\bm{\delta}}_a\xrightarrow{p}\EE\nabla\bg(\bX_t, Y_t(a);\btheta_a^*).
    \end{equation}
    We further combine the above expression with Lemma \ref{lem::asymptotic-normality-true-G-general} and use Slutsky's theorem to obtain (\ref{eq::asymptotic-normality-general}).
    }

{
\subsection{Proof of Theorem \ref{thm::asymptotic-normality-joint-general}}\label{apdx::proof-thm::asymptotic-normality-joint-general}

We prove a more general version of Theorem \ref{thm::asymptotic-normality-joint-general} where the score function $\bg$ in (\ref{eq::theta-a-*}) for each arm can be different. Suppose that for any $a\in\cA$, $\btheta_a^*$ satisfies 
\begin{equation}\label{eq::theta-a-*-a}
    \EE\bg_a(\bX, Y(a);\btheta_a^*) = \mathbf{0}
\end{equation}
for a score function $\bg_a$. 

The following conditions extend Assumptions \ref{aspt:identifiability-general}, \ref{aspt:boundedness-general}, \ref{aspt:smoothness-general}, and \ref{aspt:min-sampling-prob} to the more general setting where the conditions apply not only to a single arm but simultaneously across all arms $a\in\cA$.

\begin{assumption}[Well-separated solution for all arms]\label{aspt:identifiability-general-joint}
    $\forall a\in\cA, \forall \epsilon>0$, \\ $\inf_{\|\btheta - \btheta_a^*\|_2>\epsilon}\|\EE\bg_a(\bX_t, Y_t(a);\btheta)\|_2>0$.
\end{assumption}

\begin{assumption}[Boundedness for all arms]\label{aspt:boundedness-general-joint}
    There exist constants $R_{\btheta}$, $M_2$ such that \\
    (i) $\forall a\in\cA$, $\|\EE[\bg_a(\bX_t, Y_t(a);\btheta_a^*)\bg_a(\bX_t, Y_t(a);\btheta_a^*)^\top|\bX_t]\|_2\leq M_2$, a.e. $\bX_t$; \\
    (ii) $\forall a\in\cA$, $\|\btheta_a^*\|_2<R_{\btheta}$, $\sup_{\|\btheta\|_2\leq R_{\btheta}}\EE\|\bg_a(\bX_t, Y_t(a);\btheta)\|_2^2<\infty$; \\
    (iii) $\forall a\in\cA$, $\EE \|\bg_a(\bX_t, Y_t(a);\btheta_a^*)\|_2^4<\infty$.
\end{assumption}

\begin{assumption}[Smoothness for all arms]\label{aspt:smoothness-general-joint}
(i)$\forall a\in\cA$, the function $\bg_a(\bx, y; \btheta)$ is twice differentiable with respect to $\btheta$, with $\EE\nabla\bg_a(\bX_t, Y_t(a);\btheta_a^*)$ nonsingular; \\
(ii) There exists a function $\phi$ such that $\forall a\in\cA$, $\forall \bx, y$, $\sup_{\|\btheta\|_2\leq R_{\btheta}}\|\nabla\bg_a(\bx, y; \btheta)\|_2\leq \phi(\bx, y)$, and $\EE[\phi(\bX_t, Y_t(a))^2|\bX_t]\leq M_2'$ a.e. $\bX_t$ for a constant $M_2'$;\\
(iii)There exists a constant $\epsilon_0>0$ and a function $\Phi$ such that $\forall a\in\cA$, \\
$\sup_{\|\btheta - \btheta_a^*\|_2\leq \epsilon_0, i\in[d]}\|\nabla^2\bg_a^{(i)}(\bx, y; \btheta)\|_2\leq\Phi(\bx, y)$ and $\EE\Phi(\bX_t, Y_t(a))<\infty$. Here $\bg_a^{(i)}(\bx, y; \btheta)$ denotes the $i$-th entry of $\bg_a(\bx, y; \btheta)$.
\end{assumption}

\begin{assumption}[Minimum sampling probability for all arms] \label{aspt:min-sampling-prob-joint}
$\forall a\in\cA$, $\pi_t(a)\geq \pi_{\min,t}$ almost surely for a deterministic positive sequence $\{\pi_{\min,t}\}_{t\geq 1}$.
\end{assumption}


Define $\bG_{a, T} := \frac1T\sum_{t=1}^T\frac1{\pi_t(A_t)}1_{\{A_t = a\}}\bg_a(\bX_t, Y_t;\btheta)$, and similar to (\ref{eq::estimating-equation-general}), we look for $\hat\btheta_a^{(T)}$ such that $\forall a\in\cA$,
\begin{equation}\label{eq::estimating-equation-general-A}
    \bG_{a, T}(\hat\btheta_a) = o_p(1/b_{a, T})
\end{equation}
as $T\rightarrow \infty$. We prove the following theorem, which is a generalization of Theorem \ref{thm::asymptotic-normality-joint-general}.

\begin{theorem}\label{thm::asymptotic-normality-general-joint-A}
    Under Assumptions \ref{aspt:unconfoundedness}, \ref{aspt:identifiability-general-joint}, \ref{aspt:boundedness-general-joint}, \ref{aspt:smoothness-general-joint}, and \ref{aspt:min-sampling-prob-joint}, suppose the behavior policy $\pi$ satisfies {scaled inverse-propensity convergence} with rate $\{r_{a, t}\}_{a\in\cA, t\geq 1}$ {and limit function $\rho$} in the sense of Definition \ref{aspt:policy_convergence}, so that $\bar\bI_a := \EE\big[\rho(a, \bX_t)\bg(\bX_t, Y_t(a);\btheta_a^*)\bg(\bX_t, Y_t(a);\btheta_a^*)^\top\big]$ is finite for all $a\in\cA$. In addition, for any arm $a$, Conditions (i) and (ii) in Theorem \ref{thm::asymptotic-normality-general} hold. Define $\bD_T = \mathrm{Diag}\big(b_{1, T} \bI_{d}, \ldots, b_{K, T} \bI_{d}\big)\in\RR^{(Kd)\times (Kd)}$. Then there exist estimators $\{\hat\btheta_a^{(T)}\}_{a\in\cA, T\geq 1}$ such that (\ref{eq::estimating-equation-general}) holds for each $a\in\cA$, and $\|\hat\btheta_a^{(T)}\|_2\leq R_\theta$ for all $a\in\cA, T\geq 1$. In addition,  any such estimators satisfy
\begin{equation}\label{eq::asymptotic-normality-joint-general-A}
    \bD_T(\hat\btheta^{(T)} - \btheta^*)\xrightarrow{d}\cN
    \left(
    \mathbf{0},\bSigma^*
    \right)
\end{equation}
as $T\rightarrow \infty$. Here $\bSigma^* = \diag(\bSigma_1^*, \ldots, \bSigma_K^*)$, where $\bSigma_{a}^*:= \bJ_a^{-1}\bar \bI_a \bJ_a^{-1, \top}$, and $$\bar\bI_a := \EE\big[\rho(a, \bX_t)\bg_a(\bX_t, Y_t(a);\btheta_a^*)\bg_a(\bX_t, Y_t(a);\btheta_a^*)^\top\big], \bJ_a:= \EE\nabla\bg_a(\bX_t, Y_t(a);\btheta_a^*).$$
\end{theorem}

\begin{proof}[Proof of Theorem \ref{thm::asymptotic-normality-general-joint-A}]
First, the existence of $\{\hat\btheta_a^{(T)}\}_{a\in\cA, T\geq 1}$ is guaranteed by applying Theorem \ref{thm::asymptotic-normality-general} to each individual arm $a\in\cA$. Now suppose $\{\hat\btheta_a^{(T)}\}_{a\in\cA, T\geq 1}$ satisfies (\ref{eq::estimating-equation-general-A}) $\forall a\in\cA$, and $\|\hat\btheta_a^{(T)}\|_2\leq R_{\btheta}$ $\forall a\in\cA, T\geq 1$. For notational convenience, for the rest of this proof, we suppress the dependence of $\hat\btheta_a^{(T)}$ and $\hat\btheta^{(T)}$ on $T$ and denote them as $\hat\btheta_a$ and $\hat\btheta$ respectively. $\forall i\in[d]$, using Taylor expansion, we have 
\begin{align*}
    -\bG_{a, T}^{(i)}(\btheta_a^*)&= \bG_{a, T}^{(i)}(\hat\btheta_a) - \bG_{a, T}^{(i)}(\btheta_a^*)+o_p(1/b_{a, T})\\ 
    &= \langle\nabla\bG_{a, T}^{(i)}(\btheta_a^*), \hat\btheta_a-\btheta_a^*\rangle + R_{a, T}^{(i)}+o_p(1/b_{a, T})
\end{align*}
For a remainder term $R_{a, T}^{(i)}$. From the proof of Theorem \ref{thm::asymptotic-normality-general} we deduce that $R_{a, T}^{(i)} = o_P(1/b_{a, T})$. Thus, we stack over index $i$ and obtain that
$$
-\bG_{a, T}(\btheta_a^*) = \nabla\bG_{a, T}(\btheta_a^*)(\hat\btheta_a-\btheta_a^*) + o_P(1/b_{a, T})
$$
Combine the above with Lemma \ref{lem::2nd-derivative-boundness-general}, we have 
$$
b_{a, T}(\hat\btheta_a-\btheta_a^*) = -\bJ_a^{-1}\cdot b_{a, T}\bG_{a, T}(\btheta_a^*) +  o_P(1).
$$
We stack over action $a\in\cA$ and obtain 
\begin{equation}
\bD_t (\hat\btheta - \btheta^*) = -\mathrm{Diag}\big(\bJ_1^{-1}, \ldots, \bJ_K^{-1}\big) \cdot
\begin{pmatrix}
b_{1, T} \bG_{1, T}(\btheta_1^*) \\
\vdots\\
b_{K, T} \bG_{K, T}(\btheta_K^*)
\end{pmatrix}
 + o_P(1).
\end{equation}
In order to establish the desired asymptotic results, we only have to prove the asymptotic normality on RHS of the above. From Cramer-Wold Theorem, it suffices to show for any set of deterministic vectors $\{\bbeta_a\}_{a\in\cA}$, $\bbeta_a\in\RR^d$, 
\begin{equation}\label{eq::asymptotic-normality-joint-general-A-CW}
\sum_{a\in\cA}\bbeta_a^\top\bJ_a^{-1}\cdot b_{a, T}\bG_{a, T}(\btheta_a^*)\xrightarrow{d}\cN\big(0, \sum_a\bbeta_a^\top \bJ_a^{-1}\bar \bI_a\bJ_a^{-1, \top} \bbeta_a\big).
\end{equation}
Denote $\bZ_{t, a}= \frac{1_{\{A_t = a\}}}{\pi_t(A_t)}\bg_a(\bX_t, Y_t(a);\btheta_a^*)$. Then LHS of (\ref{eq::asymptotic-normality-joint-general-A-CW}) is equal to 
$$
\sum_a\bbeta_a^\top \bJ_a^{-1}b_{a, T}\cdot \frac1T\sum_{t=1}^T \bZ_{t, a} = \frac1T\sum_t\bar\bZ_t,
$$
where $\bar \bZ_t := \sum_ab_{a, T}\bbeta_a^\top \bJ_a^{-1}\bZ_{t, a}$. From \cite{dvoretzky1972asymptotic}, Theorem 2.2, (\ref{eq::asymptotic-normality-joint-general-A-CW}) is guaranteed by the following conditions:
\begin{gather}
   \EE [\bar Z_t|\cH_{t-1}] = 0\quad  \forall t\in[T],\label{eq::cond-exp-general-linear-combination}\\
   \frac1{T^2}\sum_{t\in[T]}\mathrm{Var}(\bar Z_t|\cH_{t-1}) \xrightarrow{p} \sum_{a\in\cA}\bbeta_a^\top \bJ_a^{-1}\bar\bI_a\bJ_a^{-1, \top}\bbeta_a,\label{eq::cond-var-general-linear-combination}\\
   \frac1{T^2}\sum_{t\in[T]}\EE\left[\bar Z_t^21_{\{|\bar Z_t|>T\delta\}}\Big|\cH_{t-1}\right]\xrightarrow{p} 0\quad \forall \delta>0.\label{eq::cond-lindeberg-general-linear-combination}
\end{gather}
Below we check these facts one by one.

\textbf{Check (\ref{eq::cond-exp-general-linear-combination})}: we have
$$
\EE [\bar Z_t|\cH_{t-1}] = \sum_a b_{a, T} \bbeta_a^\top \bJ_a^{-1}\EE[\bZ_{t, a}|\cH_{t-1}] = 0.
$$
Here we used (\ref{eq::main-mtg}).

\textbf{Check (\ref{eq::cond-var-general-linear-combination})}: we have 
\begin{align*}
\frac1{T^2}\sum_{t\in[T]}\mathrm{Var}(\bar Z_t|\cH_{t-1}) &= \frac1{T^2}\sum_t\EE[\bar Z_t^2|\cH_{t-1}]\\
& = \frac1{T^2}\sum_t\EE\Big[\big(\sum\nolimits_a b_{a, T}\bbeta_a^\top\bJ_a^{-1}\bZ_{t, a}\big)^2|\cH_{t-1}\Big] = \bar I_1 + \bar I_2,
\end{align*}
where 
\begin{align*}
\bar I_1 & = \frac1{T^2}\sum_t\sum_a\EE\Big[\big(b_{a, T}\bbeta_a^\top\bJ_a^{-1}\bZ_{t, a}\big)^2|\cH_{t-1}\Big],\\
\bar I_2 & = 
\frac2{T^2}\sum_t\sum_{a<a'}\EE\Big[
\big(b_{a, T}\bbeta_a^\top\bJ_a^{-1}\bZ_{t, a}\big)\big(b_{a', T}\bbeta_{a'}^\top\bJ_{a'}^{-1}\bZ_{t, a'}\big)
|\cH_{t-1}\Big].
\end{align*}
From (\ref{eq::cond-var-general}), we have $\bar I_1\xrightarrow{p} \sum_a\bbeta_a^\top \bJ_a^{-1}\bar \bI_a \bJ_a^{-1, \top}\bbeta_a$. For $\bar I_2$, note that for any $a\neq a'$, $\bZ_{t, a}\bZ_{t, a'}^\top = \mathbf{0}$ due to the indicators. Thus $\bar I_2=0$. Combining the arguments leads to the desired result.

\textbf{Check (\ref{eq::cond-lindeberg-general-linear-combination})}: fix any positive constant $\epsilon>0$. We have 
\begin{align*}
&\frac1{T^2}\sum_{t\in[T]}\EE\left[\bar Z_t^21_{\{|\bar Z_t|>T\epsilon\}}\Big|\cH_{t-1}\right]
 \leq \frac{1}{T^4\epsilon^2}\sum_t\EE[\bar Z_t^4|\cH_{t-1}]\\
=& \frac{1}{T^4\epsilon^2}\sum_t\EE\left[\big(\sum_a b_{a, T}\bbeta_a^\top \bJ_a^{-1}\bZ_{t, a}\big)^4\Big|\cH_{t-1}\right]\\
=& \frac{1}{T^4\epsilon^2}\sum_t\EE\left[\sum_a\big( b_{a, T}\bbeta_a^\top \bJ_a^{-1}\bZ_{t, a}\big)^4\Big|\cH_{t-1}\right]\\
=& \sum_a\frac{b_{a, T}^4}{T^4\epsilon^2}\sum_t\EE
\left[
\frac{1_{\{A_t = a\}}}{\pi_t(A_t)^4}\big(\bbeta_a^\top\bJ_a^{-1}\bg_a(\bX_t, Y_t(a);\btheta_a^*)\big)^4
\Big|\cH_{t-1}\right]\\
=& \sum_a\frac{1}{S_{a, T}^4\epsilon^2}\sum_t\EE
\left[
\EE_{A_t}\bigg[\frac{1_{\{A_t = a\}}}{\pi_t(A_t)^4}\Big|\cH_{t-1}, \bX_t\bigg]
\cdot
\EE_{Y_t(a)}\big[\big(\bbeta_a^\top\bJ_a^{-1}\bg_a(\bX_t, Y_t(a);\btheta_a^*)\big)^4
|\cH_{t-1}, \bX_t\big]
\Big|\cH_{t-1}\right]\\
=& \sum_a\frac{1}{S_{a, T}^4\epsilon^2}\sum_t\EE
\left[
\frac{1}{\pi_t(a)^3}
\cdot
\EE_{Y_t(a)}\big[\big(\bbeta_a^\top\bJ_a^{-1}\bg_a(\bX_t, Y_t(a);\btheta_a^*)\big)^4
|\cH_{t-1}, \bX_t\big]
\Big|\cH_{t-1}\right]\\
\leq & \sum_a\frac{1}{S_{a, T}^4\epsilon^2}\sum_t\pi_{\min,t}^{-3}\EE\big[\big(\bbeta_a^\top\bJ_a^{-1}\bg_a(\bX_t, Y_t(a);\btheta_a^*)\big)^4\big]\\
\leq & \sum_a\frac{\sum_{t=1}^T\pi_{\min,t}^{-3}}{\epsilon^2S_{a, T}^4}\cdot \max_a\|\bJ_{a}^{-1, \top}\bbeta_a\|_2^4 \cdot \max_{a}\EE\|\bg_a(\bX_1, Y_1(a);\btheta_a^*)\|_2^4\to 0.
\end{align*}
Here for the first inequality we use Chebyshev's Inequality. The second equality is due to the fact that all the cross terms in the expansion is exactly zero because each $\bZ_{t, a}$ contains the indicator term $1_{\{A_t= a\}}$. The last convergence is from the conditions in Theorem \ref{thm::asymptotic-normality-general-joint-A}, as well as Assumptions \ref{aspt:boundedness-general-joint} and \ref{aspt:smoothness-general-joint}.
\end{proof}
}
{
\subsection{Proof of Proposition \ref{thm::consistent-var-estimator-general}}\label{apdx::proof-thm::consistent-var-estimator-general}

We first prove that as $T\rightarrow \infty$,
\begin{equation}\label{eq::var-estimator-component-1-general}
\hat{\dot{\bG}}_{a, T}\xrightarrow{p}\EE[\nabla\bg(\bX_t, Y_t(a);\btheta_a^*)].
\end{equation}
From Lemma \ref{lem::convergence-true-derivative-general}, we deduce that 
$$
\frac1T\sum_{t=1}^T\frac1{\pi_t(A_t)}1_{\{A_t = a\}}\nabla\bg(\bX_t, Y_t(a);\btheta_a^*)\xrightarrow{p} \EE\nabla \bg(\bX_t, Y_t(a);\btheta_a^*).
$$
In addition, $\forall i$,
\begin{align}
&\left\|\frac1T\sum_{t=1}^T\frac1{\pi_t(A_t)}1_{\{A_t = a\}}\nabla\bg^{(i)}(\bX_t, Y_t(a);\hat\btheta_a^{(T)}) - \frac1T\sum_{t=1}^T\frac1{\pi_t(A_t)}1_{\{A_t = a\}}\nabla\bg^{(i)}(\bX_t, Y_t(a);\btheta_a^*)\right\|_2\nonumber\\
= & 
\left\|
\frac1T\bigg[
\sum_{t=1}^T\frac{1_{\{A_t = a\}}}{\pi_t(A_t)}
\int_{0}^1\nabla^2\bg^{(i)}(\bX_t, Y_t(a);\btheta_a^* + u(\hat\btheta_a^{(T)} - \btheta_a^*))\mathrm{d}u
\bigg]\cdot (\hat\btheta_a^{(T)} - \btheta_a^*)
\right\|_2\nonumber\\
\leq & \frac1T
\bigg[
\sum_{t=1}^T\frac{1_{\{A_t = a\}}}{\pi_t(A_t)}
\Big\|\int_{0}^1\nabla^2\bg^{(i)}(\bX_t, Y_t(a);\btheta_a^* + u(\hat\btheta_a^{(T)} - \btheta_a^*))\mathrm{d}u\Big\|_2
\bigg]\cdot \|\hat\btheta_a^{(T)} - \btheta_a^*\|_2\nonumber\\
\leq & \frac1T
\bigg[
\sum_{t=1}^T\frac{1_{\{A_t = a\}}}{\pi_t(A_t)}
\Phi(\bX_t, Y_t(a))
\bigg]\cdot \|\hat\btheta_a^{(T)} - \btheta_a^*\|_2 + o_p(1).\label{eq::consistent-var-G-1}
\end{align}
Here the last inequality uses Lemma \ref{lem::consistency-general} and Assumption \ref{aspt:smoothness-general}. Note that 
\begin{align*}
\EE\bigg[
\frac{1_{\{A_t = a\}}}{\pi_t(A_t)}
\Phi(\bX_t, Y_t(a))
\bigg]
&= 
\EE
\bigg[
\EE_{A_t}\bigg[
\frac{1_{\{A_t = a\}}}{\pi_t(A_t)}\Big|\cH_{t-1}, \bX_t
\bigg]
\cdot
\EE_{Y_t(a)}\Big[
\Phi(\bX_t, Y_t(a))
\Big|\cH_{t-1}, \bX_t
\Big]
\bigg]\\
& = 
\EE\Big[
\Phi(\bX_t, Y_t(a))
\Big]<\infty.
\end{align*}
Thus, we have 
$$
\frac1T
\bigg[
\sum_{t=1}^T\frac{1_{\{A_t = a\}}}{\pi_t(A_t)}
\Phi(\bX_t, Y_t(a))
\bigg] = O_p(1)
$$
Combine this and Lemma \ref{lem::consistency-general}, we deduce that the RHS of (\ref{eq::consistent-var-G-1}) is $o_p(1)$. Summarizing the above analysis leads to (\ref{eq::var-estimator-component-1-general}).

Next we prove that as $T\rightarrow \infty$, 
\begin{equation}\label{eq::var-estimator-component-2-general}
    \hat\bI_{a, T}\xrightarrow{p}\bar\bI_a.
\end{equation}
Denote $\bZ_t = \frac{1}{\pi_t(A_t)}1_{\{A_t = a\}}\bg(\bX_t, Y_t; \btheta_a^*)$, and $\hat\bZ_t^{(T)} = \frac{1}{\pi_t(A_t)}1_{\{A_t = a\}}\bg(\bX_t, Y_t; \hat \btheta_a^{(T)})$. Then (\ref{eq::var-estimator-component-2-general}) is equivalent to 
$$
\frac1{S_{a, T}^2}\sum_{t=1}^T\hat\bZ_t^{(T)}\hat\bZ_t^{(T), \top}\xrightarrow{p} \bar \bI_a,
$$
and in order to show this, it suffices to prove 
\begin{align}
    \frac1{S_{a, T}^2}\sum_{t=1}^T \hat\bZ_t^{(T)}\hat\bZ_t^{(T), \top}  - \frac1{S_{a, T}^2}\sum_{t=1}^T \bZ_t\bZ_t^{\top} &= o_p(1),\label{eq::var-estimator-component-2-1-general}\\
    \frac1{S_{a, T}^2}\sum_{t=1}^T \Big(\bZ_t\bZ_t^{\top} - \EE[\bZ_t\bZ_t^{\top}|\cH_{t-1}]\Big) &= o_p(1),\label{eq::var-estimator-component-2-2-general} \\
    \frac1{S_{a, T}^2}\sum_{t=1}^T \EE[\bZ_t\bZ_t^{\top}|\cH_{t-1}]&\xrightarrow{p}\bar\bI_a.\label{eq::var-estimator-component-2-3-general}
\end{align}
Below we check these facts one by one.

\textbf{Check (\ref{eq::var-estimator-component-2-1-general}):} For convenience, we define $\bg_t^* = \bg(\bX_t, Y_t(a);\btheta_a^*)$, and $\hat\bg_t^{(T)} = \bg(\bX_t, Y_t(a);\hat\btheta_a^{(T)})$. Then 
\begin{align}
\frac1{S_{a, T}^2}\sum_{t=1}^T \hat\bZ_t^{(T)}\hat\bZ_t^{(T), \top}  - \frac1{S_{a, T}^2}\sum_{t=1}^T \bZ_t\bZ_t^{\top}
= \frac1{S_{a, T}^2}\sum_{t=1}^T 
\frac{1_{\{A_t = a\}}}{\pi_t(A_t)^2}\big(\hat\bg_t^{(T)}\hat\bg_t^{(T), \top} - \bg_t^*\bg_t^{*, \top}\big) = \bI_1' + \bI_2' + \bI_3'. \label{eq::consistent-var-part1-decomposition}
\end{align}
Here
\begin{align*}
\bI_1'& = \frac1{S_{a, T}^2}\sum_{t=1}^T 
\frac{1_{\{A_t = a\}}}{\pi_t(A_t)^2}\bg_t^* \big(\hat\bg_t^{(T)} - \bg_t^*\big)^\top,\\
\bI_2'& = \frac1{S_{a, T}^2}\sum_{t=1}^T 
\frac{1_{\{A_t = a\}}}{\pi_t(A_t)^2}\big(\hat\bg_t^{(T)} - \bg_t^*\big) \big(\hat\bg_t^{(T)} - \bg_t^*\big)^\top,\\
\bI_3'& =  \frac1{S_{a, T}^2}\sum_{t=1}^T\frac{1_{\{A_t = a\}}}{\pi_t(A_t)^2}\big(\hat\bg_t^{(T)} - \bg_t^*\big)\bg_t^{*, \top}.
\end{align*}
We first analyze $\bI_1'$. We have
\begin{align}
\|\bI_1'\|_2 &= \frac1{S_{a, T}^2}\left\|\sum_{t=1}^T 
\frac{1_{\{A_t = a\}}}{\pi_t(A_t)^2}\bg(\bX_t, Y_t(a);\btheta_a^*)\big(\hat\btheta_a^{(T)} - \btheta_a^*\big)^\top\cdot \Big[\int_{0}^1\nabla \bg(\bX_t, Y_t(a);\btheta_a^* + u(\hat\btheta_a^{(T)} - \btheta_a^*))\mathrm{d}u\Big]^\top\right\|_2\nonumber\\
&\leq \frac1{S_{a, T}^2}
\bigg[
\sum_{t=1}^T \frac{1_{\{A_t = a\}}}{\pi_t(A_t)^2}\|\bg(\bX_t, Y_t(a);\btheta_a^*)\|_2
\cdot \phi(\bX_t, Y_t(a))
\bigg]\cdot \big\|\hat\btheta_a^{(T)} - \btheta_a^*\big\|_2\label{eq::consistent-var-I1'-1}
\end{align}
where we have used Assumption \ref{aspt:smoothness-general}. Denote $\phi_g^*(\bx, y) = \|\bg(\bx, y;\btheta_a^*)\|_2 \cdot \phi(\bx;y)$. Then From Assumptions \ref{aspt:boundedness-general} and \ref{aspt:smoothness-general} and Cauchy-Schwarz Inequality, we have 
$$
\EE[\phi_g^*(\bX_t, Y_t(a))|\bX_t]\leq \sqrt{M_2 M_2'} \quad a.e. \bX_t.
$$
Thus,
\begin{align}
&\frac1{S_{a, T}^2}\EE
\bigg[
\sum_{t=1}^T \frac{1_{\{A_t = a\}}}{\pi_t(A_t)^2}\phi_g^*(\bX_t, Y_t(a))
\bigg]\nonumber\\
= & \frac1{S_{a, T}^2}\sum_{t=1}^T\EE\left[
\EE\bigg[\frac{1_{\{A_t = a\}}}{\pi_t(A_t)^2}\Big|\cH_{t-1}, \bX_t\bigg]
\cdot\EE\Big[
\phi_g^*(\bX_t, Y_t(a))
\Big|\cH_{t-1}, \bX_t
\Big]
\right]\nonumber\\
\leq & \frac1{S_{a, T}^2}\sum_{t=1}^T\EE\left[
\frac1{\pi_t(a)}\cdot \sqrt{M_2M_2'}
\right]\nonumber\\
\leq &\frac{\sqrt{M_2M_2'}}{S_{a, T}^2}\left[
\sum_{t=1}^T\frac1{r_{a, t}}\EE\big[\rho(a,\bX_t)\big] +\sum_{t=1}^T\frac1{r_{a, t}} \EE\bigg[\frac{r_{a, t}}{\pi_t(a)} - \rho(a,\bX_t)\bigg]
\right]\nonumber\\
= & \sqrt{M_2M_2'}\cdot \EE\bigg[\frac1{\bar \pi(a|\bX_t)}\bigg] + o_p(1) = O_p(1).\label{eq::consistent-var-I1'-2}
\end{align}
Here the second-to-last equality uses {scaled inverse-propensity convergence} as well as Toeplitz Lemma, similar to the proof of Lemma \ref{lem::asymptotic-normality-true-G-general}. Combining (\ref{eq::consistent-var-I1'-1}) and (\ref{eq::consistent-var-I1'-2}), we deduce that $\|\bI_1'\|_2 = o_p(1)$. Using the same arguments, we also have $\|\bI_3'\|_2 = o_p(1)$.

The analysis for $\bI_2'$ is also similar. In fact similar to (\ref{eq::consistent-var-I1'-1}) we obtain that
$$
\|\bI_2'\| \leq \frac1{S_{a, T}^2}
\bigg[\sum_{t=1}^T \frac{1_{\{A_t = a\}}}{\pi_t(A_t)^2}\phi(\bX_t, Y_t(a))^2
\bigg]\cdot \big\|\hat\btheta_a^{(T)} - \btheta_a^*\big\|_2^2.
$$
Similar to (\ref{eq::consistent-var-I1'-2}) we will have 
$$
\frac1{S_{a, T}^2}
\bigg[\sum_{t=1}^T \frac{1_{\{A_t = a\}}}{\pi_t(A_t)^2}\phi(\bX_t, Y_t(a))^2
\bigg] = O_p(1).
$$
Combining the consistency of $\hat\btheta_a^{(T)}$, we deduce that $\|\bI_2'\|_2 = o_p(1)$. 

Finally, we plug in the analysis for $\bI_1'$, $\bI_2'$ and $\bI_3'$ to (\ref{eq::consistent-var-part1-decomposition}) and obtain the desired result.

\textbf{Check (\ref{eq::var-estimator-component-2-2-general}):} $\forall \bc\in\RR^d$, $\forall \delta>0$,
\begin{align}
& \PP\left(\bigg|\frac1{S_{a, T}^2}\sum_{t=1}^T[\bc^\top\bZ_t\bZ_t^\top\bc - \EE[\bc^\top\bZ_t\bZ_t^\top\bc|\cH_{t-1}]]\bigg|>\delta\right)\nonumber\\
\leq & \frac1{\delta^2{S_{a, T}^4}}\EE\left(\sum_{t=1}^T[\bc^\top\bZ_t\bZ_t^\top\bc - \EE[\bc^\top\bZ_t\bZ_t^\top\bc|\cH_{t-1}]]\right)^2\nonumber\\
 = & \frac1{\delta^2{S_{a, T}^4}}\sum_{t=1}^T\EE\left(\bc^\top\bZ_t\bZ_t^\top\bc - \EE[\bc^\top\bZ_t\bZ_t^\top\bc|\cH_{t-1}]\right)^2\nonumber\\
 \leq & \frac1{\delta^2{S_{a, T}^4}}\sum_{t=1}^T\EE\left(\bc^\top\bZ_t\bZ_t^\top\bc\right)^2\nonumber\\
 = & \frac1{\delta^2{S_{a, T}^4}} \sum_{t=1}^T \EE\left[
 \frac{1_{\{A_t = a\}}}{\pi_t(A_t)^4}\cdot \big(\bc^\top \bg(\bX_t, Y_t(a);\btheta_a^*)\big)^4
 \right]\nonumber\\
 = & \frac1{\delta^2{S_{a, T}^4}} \sum_{t=1}^T \EE\left[
 \EE\bigg[\frac{1_{\{A_t = a\}}}{\pi_t(A_t)^4}\Big|\cH_{t-1},\bX_t\bigg]
 \cdot \EE\Big[\big(\bc^\top \bg(\bX_t, Y_t(a);\btheta_a^*)\big)^4\Big|\cH_{t-1},\bX_t\Big]
 \right]\nonumber\\
 = & \frac1{\delta^2{S_{a, T}^4}} \sum_{t=1}^T \EE\left[
 \frac{1}{\pi_t(a)^3}
 \cdot \EE\Big[\big(\bc^\top \bg(\bX_t, Y_t(a);\btheta_a^*)\big)^4\Big|\cH_{t-1},\bX_t\Big]
 \right]\nonumber\\
 \leq & \frac1{\delta^2{S_{a, T}^4}}\bigg(\sum_{t=1}^T\pi_{\min,t}^{-3}\bigg)\cdot \EE[\big(\bc^\top \bg(\bX_t, Y_t(a);\btheta_a^*)\big)^2]
 \to 0
\end{align}
Here the first inequality is because of Chebyshev's Inequality. The second equality is due to the following fact: Let $v_{t}' = \bc^\top\bZ_t\bZ_t^\top\bc$. Then for $t_1<t_2$,
\begin{align*}
   & \EE (v_{t_1}' - \EE[v_{t_1}'|\cH_{t_1-1}])(v_{t_2}' - \EE[v_{t_2}'|\cH_{t_2-1}]) \\
   = & \EE\big[ \EE[(v_{t_1}' - \EE[v_{t_1}'|\cH_{t_1-1}])(v_{t_2}' - \EE[v_{t_2}'|\cH_{t_2-1}])\big|\cH_{t_2-1}]\big]\\
   = & \EE\big[(v_{t_1}' - \EE[v_{t_1}'|\cH_{t_1-1}])\cdot \EE[v_{t_2}' - \EE[v_{t_2}'|\cH_{t_2-1}]\big|\cH_{t_2-1}]\big]\\
   = & \EE\big[ (v_{t_1}' - \EE[v_{t_1}'|\cH_{t_1-1}])\cdot 0\big] = 0.
\end{align*}
The last convergence uses Assumption \ref{aspt:boundedness-general}. 

\textbf{Check (\ref{eq::var-estimator-component-2-3-general}):} In fact, (\ref{eq::var-estimator-component-2-3-general}) is a direct consequence of the analysis in (\ref{eq::cond-var-general}).
}

{
\subsection{Proof of Proposition \ref{prop::convergence-mab-no-pi-min}}\label{apdx::proof-prop::convergence-mab-no-pi-min}
Define $\cH_t^0:= \sigma(\{A_\tau, Y_\tau\}_{\tau = 1}^t)$ for $t\geq 1$. For any $a\in\cA$, $t\geq 1$ define $D_{a, t} = 1_{\{A_t = a\}}(Y_t(a) - \mu_a^*)$, $S_{a, t} = \sum_{\tau\leq t}D_{a, \tau}$. By unconfoundedness, 
$$
\EE[D_{a, t}|\cH_{t-1}^0]  = 1_{\{A_t = a\}}\cdot \EE\big[Y_t(a)-\mu_a^*|\cH_{t-1}^0\big] = 0.
$$
From the subgaussian assumption, $\forall \lambda$, $\forall \tau\geq 1$, $\EE\big[\exp(\lambda D_{a, \tau})|\cH_{\tau-1}^0, A_\tau\big]\leq \exp\big(\frac{\lambda^2\sigma_Y^2}{2}1_{\{A_\tau = a\}}\big)$. Iterating over $\tau$ using the tower property, we obtain that
$$
\EE\left[\exp(\lambda S_{a, t})|A_{1:t}\right]\leq \exp\left(\frac{\lambda^2\sigma_Y^2}{2}N_{a, t}\right).
$$
Apply the Chernoff bound condition on $A_{1:t}$, we obtain that $\forall x>0$, $\PP(S_{a, t}\geq x|A_{1:t})\leq \exp\big(-\lambda x +\frac{\lambda^2\sigma_Y^2}{2}N_{a, t}\big)$.
Thus, given any $n\in\NN$, on the event $\{N_{a, t} = n\}$, by letting $\lambda = \frac{x}{\sigma_Y^2n}$, we deduce that 
$$
\PP(S_{a, t}\geq x|A_{1:t}, N_{a, t} = n)\leq \exp\left(-\frac{x^2}{2\sigma_Y^2n}\right).
$$
Similarly, we combine the analysis of the event $\{S_{a, t}\leq -x\}$, and the fact that $S_{a, t} = N_{a, t}(\hat\mu_{a, t} - \mu_a^*)$, then we obtain
\begin{equation}\label{eq::mab-concentration-00}
   \PP(|\hat\mu_{a, t} - \mu_a^*|\geq \delta|A_{1:t}, N_{a, t} = n)\leq 2\exp\left(-\frac{n\delta^2}{2\sigma_Y^2}\right). 
\end{equation}
Therefore $\forall m\geq 1$, 
\begin{align}
&\PP(|\hat\mu_{a, t} - \mu_a^*|\geq \delta)
\leq \PP(|\hat\mu_{a, t} - \mu_a^*|\geq \delta|N_{a, t}\geq m) + \PP(N_{a, t}< m)\nonumber\\
\leq & 2\exp\left(-\frac{m\delta^2}{2\sigma_Y^2}\right) + \PP(N_{a, t}< m).\label{eq::mab-eps-greedy-concentration-0}
\end{align}

$\bullet$ \textbf{$\epsilon$-greedy.} Without loss of generality, we assume that the $\epsilon$-greedy policy is realized via an independent external randomization sequence $\{B_t, U_t\}_{t\geq 1}$, where $B_t\sim \mathrm{Bernoulli}(\epsilon_t)$ denotes the indicator of exploration at time $t$, $U_t\sim \mathrm{Unif}(\cA)$ denotes the arm to select if the algorithm explores, all independent. Choose $m_t = \frac1K\sum_{\tau\leq t} \epsilon_\tau$, and denote $M_{a, t} = \sum_{\tau\leq t}E_{a, \tau}$, where $E_{a, \tau} = 1_{\{B_\tau = 1, U_\tau = a\}}$. Then we have $N_{a, t}\geq M_{a,t}$ a.s. and that $\EE M_{a, t} = m_t$. Using the Chernoff bound we obtain that 
$$
\PP(N_{a, t}\leq \frac12m_t)\leq \PP(M_{a, t}\leq \frac12m_t)\leq \exp\left(-\frac18 m_t\right).
$$
Now we plug $m = \lfloor m_t/2\rfloor + 1$, $\delta = \Delta/3$ into (\ref{eq::mab-eps-greedy-concentration-0}) and get
$$
\PP(|\hat\mu_{a, t} - \mu_a^*|\geq \delta)\leq 2\exp\left(-\frac{m_t\Delta^2}{36\sigma_Y^2}\right) + \exp\left(-\frac18 m_t\right)\leq 3\exp\left(-c\sum_{\tau\leq t}\epsilon_\tau\right).
$$
Here $c = \frac1K \min\big\{\frac18, \frac{\Delta^2}{36\sigma_Y^2}\big\}$. This indicates that $\forall t\geq 1$, 
$$
\PP(\hat a_t\neq a^*)\leq 3K\exp\left(-c\sum_{\tau< t}\epsilon_\tau\right).
$$
Here $\hat a_t := \argmax_a \hat\mu_{a, t-1}$.

Now for the optimal arm $a^*$, 
\begin{align*}
\EE\left|\frac1{\pi_t(a|\cH_{t-1}^0)} - 1\right|
&\leq \left|\frac1{1-\frac{K-1}{K}\epsilon_t} - 1\right| + \bigg(\frac{K}{\epsilon_t} + 1\bigg)\PP(\hat a_t\neq a^*)\\
& \leq O(\epsilon_t) + 3K\bigg(\frac{K}{\epsilon_t} + 1\bigg)\exp\left(-c\sum_{\tau< t}\epsilon_\tau\right) = o(1).
\end{align*}
For any suboptimal arm $a\neq a^*$, 
\begin{align*}
\EE\left|\frac{\epsilon_t}{\pi_t(a|\cH_{t-1}^0)} - K\right|
&\leq \left|\frac{\epsilon_t}{1-\frac{K-1}{K}\epsilon_t} - K\right|\PP(\hat a_t\neq a^*)\\
& \leq O(\epsilon_t + 1) \PP(\hat a_t\neq a^*) = o(1).
\end{align*}
Thus the conclusion follows.

$\bullet$ \textbf{UCB.} Similar to the $\epsilon$-greedy case, we assume that the policy is realized via an independent external randomization sequence $\{B_t, U_t\}_{t\geq 1}$, where $B_t\sim \mathrm{Bernoulli}(K\pi_{\min,t})$ denotes the indicator of exploration at time $t$, $U_t\sim \mathrm{Unif}(\cA)$ denotes the arm to select if the algorithm explores, all independent. Define $m_t = \sum_{\tau\leq t} \pi_{\min,\tau}$, and denote $M_{a, t} = \sum_{\tau\leq t}E_{a, \tau}$, where $E_{a, \tau} = 1_{\{B_\tau = 1, U_\tau = a\}}$. Then we have $N_{a, t}\geq M_{a,t}$ a.s. and that $\EE M_{a, t} = m_t$. Using the Chernoff bound we obtain that 
\begin{equation}\label{eq::mab-arm-pull-control-no-min-prob}
\PP(N_{a, t}\leq \frac12m_t)\leq \PP(M_{a, t}\leq \frac12m_t)\leq \exp\left(-\frac18 m_t\right).    
\end{equation}
Now we plug $m = \lfloor m_t/2\rfloor + 1$, $\delta = \Delta/5$ into (\ref{eq::mab-eps-greedy-concentration-0}) and get
\begin{equation}\label{eq::mab-convergence-n-1}
    \PP(|\hat\mu_{a, t} - \mu_a^*|\geq \delta)\leq 2\exp\left(-\frac{m_t\Delta^2}{100\sigma_Y^2}\right) + \exp\left(-\frac18 m_t\right)\leq 3\exp\left(-c'\sum_{\tau\leq t}\pi_{\min,\tau}\right).
\end{equation}
Here $c' = \min\big\{\frac18, \frac{\Delta^2}{100\sigma_Y^2}\big\}$. 
In addition, $\exists T_0\in\NN$ s.t. $\forall t\geq T_0$, $\sqrt{C_{t+1}/(m_{t}/2)}\leq \Delta/4$. Thus, $\forall t\geq T_0$, 
\begin{equation}\label{eq::mab-convergence-n-2}
\PP\left(\sqrt{\frac{C_{t+1}}{N_{a, t}}}\geq\frac{\Delta}{4}\right)\leq \PP\left(N_{a, t}\leq \frac12 m_t\right)\leq \exp\left(-\frac18 m_t\right).
\end{equation}
Combining (\ref{eq::mab-convergence-n-1}) and (\ref{eq::mab-convergence-n-2}), taking a union bound over all actions, we deduce that $\forall t> T_0$,
$$
\PP(\tilde a_{t}\neq a^*)\leq 4K\exp\left(-c'\sum_{\tau< t}\pi_{\min,\tau}\right).
$$
Here $\tilde a_t := \argmax_a \hat\mu_{a, t-1} + \sqrt{C_t/N_{a, t-1}}$.

Now for the optimal arm $a^*$, 
\begin{align*}
\EE\left|\frac1{\pi_t(a|\cH_{t-1}^0)} - 1\right|
&\leq \left|\frac1{1-(K-1)\pi_{\min,t}} - 1\right| + \bigg(\frac1{\pi_{\min,t}} + 1\bigg)\PP(\tilde a_t\neq a^*)\\
& \leq O(\pi_{\min,t}) + 4K\bigg(\frac1{\pi_{\min,t}} + 1\bigg)\exp\left(-c'\sum_{\tau< t}\pi_{\min,\tau}\right) = o(1).
\end{align*}
For any suboptimal arm $a\neq a^*$, 
\begin{align*}
\EE\left|\frac{\pi_{\min,t}}{\pi_t(a|\cH_{t-1}^0)} - 1\right|
&\leq \left|\frac{\pi_{\min,t}}{1-(K-1)\pi_{\min,t}} - 1\right|\PP(\tilde a_t\neq a^*)\\
& \leq O(\pi_{\min,t} + 1) \PP(\tilde a_t\neq a^*) = o(1).
\end{align*}
Thus the conclusion follows.

$\bullet$ \textbf{TS.} To analyze the TS algorithm, we need to further quantify each $\mu_{a, t}^{\mathrm{post}}$ in addition to $\hat\mu_{a, t}$. From (\ref{eq::mab-concentration-00}) we deduce that 
\begin{align}
&\PP\left(\big|\mu_{a, t}^{\mathrm{post}} - \mu_a^*\big|\geq \delta\big|A_{1:t}, N_{a, t} = n\right)\nonumber\\
= & \PP\left(\left|\bigg(\frac1{\sigma_0^2} +\frac{N_{a, t}}{\sigma^2}\bigg)^{-1}\bigg(\frac{\mu_0}{\sigma_0^2} +\frac{N_{a, t}\hat\mu_{a, t}}{\sigma^2}\bigg) - \mu_a^*\right|\geq \delta\bigg|A_{1:t}, N_{a, t} = n \right)\nonumber\\
\leq & \PP\left(
\big|\hat\mu_{a, t}- \mu_a^*\big|\geq \bigg(\frac{\sigma^2}{n\sigma_0^2} + 1\bigg)\delta - \frac{\sigma^2}{n\sigma_0^2} \big|\mu_a^* - \mu_0\big|
\bigg|A_{1:t}, N_{a, t} = n \right)\nonumber\\
\leq & 2\exp\left(
-\frac{n}{2\sigma_Y^2}\bigg[\bigg(\frac{\sigma^2}{n\sigma_0^2} + 1\bigg)\delta - \frac{\sigma^2}{n\sigma_0^2}\big|\mu_a^* - \mu_0\big|\bigg]_+^2
\right).\label{eq::TS-post-mean-control-no-min-prob}
\end{align}

Now we proceed to prove (\ref{eq::mab-arm-pull-control-no-min-prob}) for the TS policy using a similar argument as in $\epsilon$-greedy and UCB. Specifically, define 
$$
q_t (a|\cH_{t-1}):= \frac{\pi_t^{TS}(a|\cH_{t-1}) - \pi_{\min,t}}{1-K\pi_{\min,t}}.
$$
Then it is easy to verify $q_t(\cdot|\cH_{t-1})$ defines a probability measure on $\cA$. In addition,
$$
\pi_t^{TS}(a|\cH_{t-1}) = K\pi_{\min}\mathrm{Unif}(\cA) + (1-K\pi_{\min}) q_t (a|\cH_{t-1}).
$$
Without loss of generality, assume that the action selection in TS is proceeded recursively as follows: for each $t$, condition on $\cH_{t-1}$, sample $V_t\sim q_t (\cdot|\cH_{t-1})$, and independently sample $B_t\sim \mathrm{Bernoulli}(K\pi_{\min,t})$, $U_t\sim \mathrm{Unif}(\cA)$. Now set 
$$
A_t = 
\begin{cases}
U_t\quad& \text{if } B_t = 1,\\
V_t\quad& \text{if } B_t = 0.
\end{cases}
$$
Then $A_t\sim \pi_t^{TS}(a|\cH_{t-1})$. In addition, $E_{a, t}:=1_{\{B_t = 1, U_t = a\}}\leq 1_{\{A_t = a\}}$, a.s. Thus $M_{a, t} := \sum_{\tau\leq t} E_{a, t}\leq N_{a, t}$ a.s. and we have $\EE M_{a, t} = \sum_{\tau\leq t}\pi_{\min,\tau}: = m_t$. Similar to the UCB case we obtain (\ref{eq::mab-arm-pull-control-no-min-prob}). 

Now fix a suboptimal arm $a\neq a^*$. Define the event $\cE_{a, t} = \cE_{a, t}^{(1)}\cap \cE_{a, t}^{(2)}$, where $\cE_{a, t}^{(1)} = \{\mu_{a^*, t}^{\mathrm{post}} \geq \mu_{a, t}^{\mathrm{post}} + 3\Delta/5\}$, $\cE_{a, t}^{(2)} = \big\{\max\{(\sigma_a^{\mathrm{post}})^2, (\sigma_{a^*}^{\mathrm{post}})^2\}\leq (1/\sigma_0^2 + m_t/(2\sigma^2))^{-1}\big\}$, both events are $\cH_{t}^0$-measurable. On $\cE_{a, t}$, when $\mu_a'$ and $\mu_{a^*}'$ are sampled jointly from the posterior up to time $t$, we have
\begin{align}
\PP_t^{\mathrm{post}}\left(\mu_a'\geq \mu_{a^*}' - \Delta/5|\cH_t^0\right)
& = \Phi\left(\Delta/5; \mu_{a^*, t}^{\mathrm{post}} - \mu_{a, t}^{\mathrm{post}}, (\sigma_a^{\mathrm{post}})^2 + (\sigma_{a^*}^{\mathrm{post}})^2 \right)\nonumber\\
&\leq \exp\left(
-\frac{\big(\mu_{a^*, t}^{\mathrm{post}} - \mu_{a, t}^{\mathrm{post}} - \Delta/5\big)^2}{2\big[(\sigma_a^{\mathrm{post}})^2 + (\sigma_{a^*}^{\mathrm{post}})^2\big]}
\right)\nonumber\\
& \leq \exp\left(
-\frac{\Delta^2}{25}\bigg(\frac{1}{\sigma_0^2} + \frac{m_t}{2\sigma^2}\bigg)
\right).\label{eq::TS-arm-a-opt-prob}
\end{align}
Here, $\Phi(z;\mu, \sigma^2)$ is defined as the CDF of $\cN(\mu, \sigma^2)$ evaluated at $z$.
Next, to prove that $\cE_{a, t}$ is a high probability event, we upper bound the probability of $\cE_{a, t}^{(1), c}$ and $\cE_{a, t}^{(2), c}$. First, from the definition of $\Delta$ we have 
\begin{equation}\label{eq::E-a-t-1-prob-0}
\PP(\cE_{a, t}^{(1), c})\leq \PP(|\mu_{a, t}^{\mathrm{post}} - \mu_a^*|\geq \Delta/5) + \PP(|\mu_{a^*, t}^{\mathrm{post}} - \mu_{a^*}^*|\geq \Delta/5).
\end{equation}
When $n\geq \frac{10\sigma^2}{\Delta\sigma_0^2}|\mu_a^* - \mu_0|$, we have $\frac{\Delta}{10} \geq \frac{\sigma^2}{n\sigma_0^2}\big|\mu_a^* - \mu_0\big|$ and thus by choosing $\delta = \Delta/5$ in (\ref{eq::TS-post-mean-control-no-min-prob}) we deduce that 
$$
\PP\left(\big|\mu_{a, t}^{\mathrm{post}} - \mu_a^*\big|\geq \Delta/5\big|A_{1:t}, N_{a, t} = n\right)\leq 2\exp\left(-\frac{\Delta^2n}{200\sigma_Y^2}\right).
$$
Thus $\forall m\geq \frac{10\sigma^2}{\Delta\sigma_0^2}|\mu_a^* - \mu_0|$,
\begin{align}
\PP\left(\big|\mu_{a, t}^{\mathrm{post}} - \mu_a^*\big|\geq \Delta/5\right)
&\leq \PP\left(\big|\mu_{a, t}^{\mathrm{post}} - \mu_a^*\big|\geq \Delta/5\big| N_{a, t} \geq m\right) + \PP(N_{a, t}< m)\nonumber\\
&\leq 2\exp\left(-\frac{\Delta^2m}{200\sigma_Y^2}\right) + \PP(N_{a, t}< m).\label{eq::TS-arm-mean-diff-bound-no-min-prob}
\end{align}
Define $m = \lfloor m_t/2\rfloor + 1$. Define $T_0' = \min\{t:m_t\geq \frac{20\sigma^2}{\Delta\sigma_0^2}\big|\mu_a^* - \mu_0\big|\}$. Then $\forall t\geq T_0'$, from (\ref{eq::mab-arm-pull-control-no-min-prob}) and (\ref{eq::TS-arm-mean-diff-bound-no-min-prob}), we deduce that 
$$
\PP\left(\big|\mu_{a, t}^{\mathrm{post}} - \mu_a^*\big|\geq \Delta/5\right)\leq 3\exp(-c''m_t),
$$
where $c'' = \min\{\frac18, \frac{\Delta^2}{400\sigma_Y^2}\}$.
Because this holds for any $a\in\cA$, we plug into (\ref{eq::E-a-t-1-prob-0}) and deduce that 
\begin{equation}\label{eq::E-a-t-1-prob}
\PP(\cE_{a, t}^{(1), c})\leq 6\exp(-c''m_t).
\end{equation}
At the same time, it is easy to obtain from (\ref{eq::mab-arm-pull-control-no-min-prob}) that 
\begin{equation}
\PP(\cE_{a, t}^{(2), c})\leq 2\exp\left(-\frac18 m_t\right).
\end{equation}
Combine the above fact, we obtain that $\PP(\cE_{a, t}^{c})\leq 8\exp(-c_1 m_t)$, where $c_1 = \min\{1/8, c''\}$.

Now define $\cE_t = \cap_{a\neq a^*} \cE_{a, t}$. Then $\forall t\geq T_0'$, $\PP(\cE_t^c)\leq 8K\exp(-c_1 m_t)$. In addition, on $\cE_t$, we deduce from (\ref{eq::TS-arm-a-opt-prob}) that $\forall a\neq a^*$, 
$$
\PP_t^{\mathrm{post}}\left(\mu_a'\geq \mu_{a^*}' - \Delta/5|\cH_t^0\right)\leq \exp\left(-\frac{\Delta^2m_t}{50\sigma^2}\right),
$$
and thus
$$
\PP_t^{\mathrm{post}}\left(A_{t+1}\neq a^*|\cH_t^0\right)\leq K\exp\left(-\frac{\Delta^2m_t}{50\sigma^2}\right).
$$
From exploration condition (ii), $\exists T_1'$ s.t. $\forall t\geq T_1'$, $K\exp\left(-\frac{\Delta^2m_t}{50\sigma^2}\right)<\pi_{\min,t+1}$. Using the property of the mapping $\mathrm{clip}(\cdot)$, when $t> T_2':=\max\{T_0', T_1'\}$, on $\cE_{t-1}$, 
$$
\pi_t^{TS}(a|\cH_{t-1})=
\begin{cases}
 \pi_{\min,t}\quad &\text{for }\forall a\neq a^*,\\
 1-(K-1)\pi_{\min,t}\quad &\text{for } a= a^*.
\end{cases}
$$
In addition, $\PP(\cE_{t-1}^{c})\leq 8K\exp(-c_1 m_{t-1})$.
Therefore $\forall a\neq a^*$, when $t> T_2'$, $\forall a\neq a^*$,
$$
\EE\left|\frac{\pi_{\min,t}}{\pi_t^{TS}(a|\cH_{t-1})} - 1\right|\leq O(\pi_{\min,t} + 1)\PP(\cE_{t-1}^c) = o(1).
$$
and for $a = a^*$,
$$
\EE\left|\frac{1}{\pi_t^{TS}(a|\cH_{t-1})} - 1\right|\leq O(\pi_{\min,t}) + (\frac{1}{\pi_{\min,t}} + 1)\PP(\cE_{t-1}^c) = o(1).
$$
Thus the conclusion follows.
}

\subsection{Analysis of MAB algorithms with a minimum sampling probability}\label{apdx::proof-lem::convergence-mab}

\begin{proposition}\label{lem::convergence-mab-pi-min}
Assume that the suboptimal gap $\Delta = \mu_{a^*}^* - \max_{a' \neq a^*}\mu_{a'}^* > 0$, and that \\
$\sup_{a\in\cA}\mathrm{Var}(Y(a))\leq \sigma_Y^2$ for a universal constant $\sigma_Y^2$. Suppose the exploration schedule satisfies $\pi_{\min,t}\equiv \pi_{\min}$, $\forall t\geq 1$. For the UCB policy, further assume that the sequence $\{C_t\}_{t\geq 1}$ satisfies $\lim_{t\rightarrow\infty}\frac{C_t}{t} = 0$, e.g., $C_t = 2\log t$ \citep{auer2010ucb}. 
Then the $\epsilon$-greedy, UCB, and Thompson sampling policies defined in \eqref{eq::policy-eps-greedy-mab}, \eqref{eq::policy-ucb-mab}, and \eqref{eq::policy-ts-mab-no-min-prob}, respectively, satisfy the scaled inverse-propensity convergence condition in Definition~\ref{aspt:policy_convergence}, with 
\begin{equation}
r_{a,t}\equiv 1,
\qquad
\rho(a,x) = 
\begin{cases}
\{1-(K-1)\pi_{\min}\bigr\}^{-1}, & a=a^*,\\
\pi_{\min}^{-1}, & a\neq a^*.
\end{cases}
\end{equation}
\end{proposition}

\emph{Proof of Proposition \ref{lem::convergence-mab-pi-min}.} We first present the following lemma. Its proof is in Appendix \ref{apdx::proof-lem::policy-with-no-contexts}.

\begin{lemma}\label{lem::policy-with-no-contexts}
Consider a setting where the behavior policy does not use the contexts, i.e.,
$$
\pi = \{
\pi_t(\cdot|\cH_{t-1}^0)\}_{t\geq 1},
$$ 
where $\cH_t^0:= \sigma(\{A_\tau, Y_\tau\}_{\tau = 1}^t)$ for $t\geq 1$. Assume that the potential outcomes have uniformly bounded variance, i.e., $\sup_{a\in\cA}\mathrm{Var}(Y(a))\leq \sigma_Y^2$ for some constant $\sigma_Y^2$. Suppose Assumption \ref{aspt:min-sampling-prob} is satisfied for all arm $a\in\cA$, and we assume unconfoundedness: $\forall t\in[T]$,
\begin{equation}\label{eq::unconfoundedness-no-context}
A_t\perp \{Y_t(a)\}_{a\in\cA}|\cH_{t-1}^0.
\end{equation}
Then as $t\rightarrow \infty$,
\begin{itemize}
    \item[(i)] For any deterministic sequence $\{C_t\}_{t\geq 1}$ such that $\lim_{t\rightarrow\infty}\frac{C_t}{t} = 0$,
    $\frac{C_t}{N_{a, t}}\xrightarrow{p} 0$ for any $a\in\cA$;
    \item[(ii)] $\hat \mu_{a, t}\xrightarrow{p} \mu_a^*$  for any $a\in\cA$.
\end{itemize}
\end{lemma}

\textbf{The $\epsilon$-greedy algorithm.} Consider the $\epsilon$-greedy policy defined by (\ref{eq::policy-eps-greedy-mab}). Define 
\begin{align*}
\pi^{(1)}_t(a|\{\hat\mu_{i, t-1}\}_{i\in\cA}) := \pi_t^{\epsilon\text{-greedy}}(a|\mathcal H_{t-1})=
\begin{cases}
    1-\frac{K-1}{K}\epsilon,\quad &\text{if }i = \argmax_i \hat\mu_{i, t-1},\\
    \frac1K\epsilon, \quad &\text{otherwise.}
\end{cases}
\end{align*}
Then $\pi^{(1)}$ satisfies Assumption \ref{aspt:min-sampling-prob} with $\pi_{\min} = \epsilon/K$. Under the setting of Proposition \ref{lem::convergence-mab-pi-min}, the bounded variance condition and the unconfoundedness condition in Lemma \ref{lem::policy-with-no-contexts} is also satisfied. Thus, according to part (i) of Lemma \ref{lem::policy-with-no-contexts}, condition (i) of Theorem \ref{lem::statistics-converge-implies-policy-converge} is satisfied with $\bbeta_t = (\hat\mu_{a, t})_{a\in\cA}$, $\bbeta^* = (\mu^*_{a})_{a\in\cA}$. In addition, condition (ii) of Theorem \ref{lem::statistics-converge-implies-policy-converge} holds for $\pi^{(1)}$ as long as $\Delta>0$. From Theorem \ref{lem::statistics-converge-implies-policy-converge}, we deduce that {scaled inverse-propensity convergence as in Definition \ref{aspt:policy_convergence}} holds.

\textbf{The UCB algorithm.} Consider the clipped UCB algorithm defined by (\ref{eq::policy-ucb-mab}). Let 
\begin{align*}
\pi^{(2)}(a|\{\hat\mu_{i, t-1}, C_t/N_{i, t-1}\}_{i\in\cA}) &:= \pi_t^{\text{UCB}}(a|\mathcal H_{t-1})\\
&=
\begin{cases}
    1-(K-1)\pi_{\min},\quad &\text{if }i = \argmax_i \left\{\hat\mu_{i, t-1} + \sqrt{\frac{C_t}{N_{i, t-1}}}\right\},\\
    \pi_{\min}, \quad &\text{otherwise.}
\end{cases}
\end{align*}
Then $\pi^{(2)}$ satisfies Assumption \ref{aspt:min-sampling-prob}, and similar to the previous case, the bounded variance condition and the unconfoundedness condition in Lemma \ref{lem::policy-with-no-contexts} also hold. According to part (i) and (ii) of Lemma \ref{lem::policy-with-no-contexts}, condition (i) of Theorem \ref{lem::statistics-converge-implies-policy-converge} is satisfied with $\bbeta_t = (\hat\mu_{a, t}, C_{t+1}/N_{a, t})_{a\in\cA}$, $\bbeta^* = (\mu^*_{a}, 0)_{a\in\cA}$. Also, condition (ii) of Theorem \ref{lem::statistics-converge-implies-policy-converge} holds for $\pi^{(2)}$ as long as $\Delta>0$. From Theorem \ref{lem::statistics-converge-implies-policy-converge}, we deduce that {scaled inverse-propensity convergence as in Definition \ref{aspt:policy_convergence}} holds.

\textbf{The TS algorithm.} Define 
\begin{align*}
\bar \pi^{(3)}(a|\bmu_{t-1}^{\mathrm{post}}, \bSigma_{t-1}^{\mathrm{post}}) =\bar\pi_t^{\text{TS}}(a|\mathcal H_{t-1}),
\end{align*}
where the dependence of $\bar \pi^{(3)}$ on $\bmu_{t-1}^{\mathrm{post}}$ and $\bSigma_{t-1}^{\mathrm{post}}$ is shown in Section \ref{sec::MAB-policy-convergence}. In addition, define 
$$
\pi^{(3)}(a|\bmu_{t-1}^{\mathrm{post}}, \bSigma_{t-1}^{\mathrm{post}})  = \left[\mathrm{Clip} \big(\bar \pi^{(3)}(i|\bmu_{t-1}^{\mathrm{post}}, \bSigma_{t-1}^{\mathrm{post}})_{i\in\cA}\big)\right]^{(a)},
$$
Here for a vector $\bv\in\RR^K$, $[\bv]^{(a)}$ denotes its $a$-th entry. From (\ref{eq::policy-ts-mab-no-min-prob}), we deduce that 
$$
\pi^{(3)}(a|\bmu_{t-1}^{\mathrm{post}}, \bSigma_{t-1}^{\mathrm{post}}) = \pi_t^{\mathrm{TS}}(a|\cH_{t-1}).
$$
We now show that conditions (i) and (ii) of Theorem \ref{lem::statistics-converge-implies-policy-converge} holds with $\pi^{(3)}$ as the behavior policy and with statistics $\bbeta_t = (\bmu_{t}^{\mathrm{post}}, \bSigma_{t}^{\mathrm{post}})$. First, $\pi^{(3)}$ satisfies Assumption \ref{aspt:min-sampling-prob}, the bounded variance condition and the unconfoundedness condition in Lemma \ref{lem::policy-with-no-contexts}. From Lemma \ref{lem::policy-with-no-contexts}, we obtain that $\forall a\in\cA$, $1/N_{a, t}\xrightarrow{p} 0$ and $\hat\mu_{a, t}\xrightarrow{p} \mu_a^*$. Combining these results together with (\ref{eq::ts-mab-posterior-mean-var}) and the continuous mapping theorem, we deduce that $\forall a\in\cA$, 
$$
\mu_{a, t}^{\mathrm{post}}\xrightarrow{p} \mu_a^*,\quad \sigma_{a, t}^{\mathrm{post}}\xrightarrow{p} 0.
$$
Note that entrywise convergence in probability implies joint convergence in probability in a finite dimensional Euclidean space, and thus from (\ref{eq::ts-mab-posterior-mean-var-joint}) we have 
$$
\bmu_{t}^{\mathrm{post}}\xrightarrow{p} \bmu^*,\quad \bSigma_{ t}^{\mathrm{post}}\xrightarrow{p} \mathbf{0}_{K\times K},
$$
where $\bmu^* = (\mu_a^*)_{a\in\cA}$. This implies condition (i) of Theorem \ref{lem::statistics-converge-implies-policy-converge} holds with the statistics $\bbeta_t$ and limit $\bbeta^* = (\bmu^*, \mathbf{0}_{K\times K})$.

In addition, it is not difficult to verify that $\bar\pi^{(3)}$ is continuous at $\bbeta^*$ as long as the suboptimality gap $\Delta>0$. Also,  $\mathrm{Clip}(\cdot)$ is continuous in its first argument. Therefore, the composite mapping $\pi^{(3)}$ is continuous at $\bbeta^*$, and condition (ii) of Theorem \ref{lem::statistics-converge-implies-policy-converge} holds.

In summary, we have verified both conditions (i) and (ii) of Theorem \ref{lem::statistics-converge-implies-policy-converge}. We deduce from the theorem that {scaled inverse-propensity convergence in Definition \ref{aspt:policy_convergence}} holds.


\subsection{Proof of Theorem \ref{lem::statistics-converge-implies-policy-converge}}\label{apdx::proof-lem::statistics-converge-implies-policy-converge}
    
    Fix any $a\in\cA$. Denote $f_{\bx}(\bbeta):=\pi(a|\bx, \bbeta)$. Then $\forall \epsilon, \delta>0$, 
\begin{align}
    &\PP(|\pi_t(a|\bX_t, \cH_{t-1}) - \bar\pi(a|\bX_t)|>\epsilon)\\
& = \PP(|\pi_t(a|\bX_t, \bbeta_{t-1}) - \pi(a|\bX_t, \bbeta^*)|>\epsilon)\nonumber\\
& = \PP(|f_{\bX_t}(\hat\bbeta_{t-1}) - f_{\bX_t}(\bbeta^*)|>\epsilon)\nonumber\\
& \leq \PP(\|\hat\bbeta_{t-1}-\bbeta^*\|_2\geq \delta) + \PP(\bbeta^*\in \cD(f_{\bX_t})) + \PP(\bbeta^*\in\cB_{\epsilon, \delta}(\bX_t)),\label{eq::policy-converge-decomposition}
\end{align}
Here $\forall \bx$, we define 
\begin{align*}
\cD(f_{\bx})&:= \{\bbeta: f_{\bx}\text{ is discontinuous at }\bbeta\},\\
\cB_{\epsilon, \delta}(\bx)& := \{\bbeta\notin \cD(f_{\bx}): \exists \bbeta',\text{ s.t.} \|\bbeta' - \bbeta\|_2<\delta, |f_{\bx}(\bbeta) - f_{\bx}(\bbeta')|>\epsilon\}.
\end{align*}
Notice that:
\begin{itemize}
    \item From condition (i), for the first term on the RHS of (\ref{eq::policy-converge-decomposition}), we have $\limsup_{t\rightarrow \infty} \PP(\|\hat\bbeta_{t-1}-\bbeta^*\|_2\geq \delta) = 0$,
    \item From condition (ii), for the second term on the RHS of (\ref{eq::policy-converge-decomposition}), we have $\PP(\bbeta^*\in \cD(f_{\bX_t})) = 0$ $\forall t$,
    \item Due to $\bX_t$ are i.i.d., the last two terms on the RHS of (\ref{eq::policy-converge-decomposition}) does not depend on time $t$.
\end{itemize}
Plugging in the above into (\ref{eq::policy-converge-decomposition}), we obtain that $\forall \epsilon, \delta>0$,
\begin{equation}\label{eq::policy-converge-decomposition-limit}
    \limsup_{t\rightarrow \infty}\PP(|\pi_t(a|\bX_t, \cH_{t-1}) - \bar\pi(a|\bX_t)|>\epsilon)\leq \PP(\bbeta^*\in\cB_{\epsilon, \delta}(\bX_t)).
\end{equation}
Finally, by definition of continuity, we have that $\forall \bx$, 
$$
\lim_{\delta\rightarrow 0} 1_{\{\bbeta^*\in\cB_{\epsilon, \delta}(\bx)\}} = 0.
$$
From the dominated convergence theorem, 
$$
\lim_{\delta\rightarrow 0}\PP(\bbeta^*\in \cB_{\epsilon, \delta}(\bX_t)) = \lim_{\delta\rightarrow 0}\EE1_{\{\bbeta^*\in \cB_{\epsilon, \delta}(\bX_t)\}} = 0.
$$
Plugging the above into (\ref{eq::policy-converge-decomposition-limit}) shows $\pi_t(a\mid\bX_t,\cH_{t-1})\xrightarrow{p}\bar\pi(a\mid\bX_t):=\pi(a\mid\bX_t,\bbeta^*)$. {It remains to upgrade this to the $L_1$ convergence of the inverse propensity required by Definition~\ref{aspt:policy_convergence} (with $r_{a,t}\equiv1$ and $\rho(a,\bx)=1/\bar\pi(a\mid\bx)$). By condition~(iii), $\pi_t(a\mid\bX_t,\cH_{t-1})\ge\pi_{\min}>0$, and the same bound passes to the limit $\bar\pi(a\mid\bX_t)\ge\pi_{\min}$; since $u\mapsto1/u$ is $\pi_{\min}^{-2}$-Lipschitz on $[\pi_{\min},1]$,
\[
    \Bigl|\tfrac{1}{\pi_t(a\mid\bX_t,\cH_{t-1})}-\tfrac{1}{\bar\pi(a\mid\bX_t)}\Bigr|\le\pi_{\min}^{-2}\bigl|\pi_t(a\mid\bX_t,\cH_{t-1})-\bar\pi(a\mid\bX_t)\bigr|\le 2\pi_{\min}^{-1},
\]
which converges to $0$ in probability and is uniformly bounded. Thus $\EE\bigl|1/\pi_t(a\mid\bX_t,\cH_{t-1})-1/\bar\pi(a\mid\bX_t)\bigr|\to0$, which is the claim.}

\subsection{Proof of Proposition \ref{theorem:convergence-of-Ridge}}\label{apdx::proof-thm::convergence-of-Ridge}

To explicitly show the dependence of our behavior policy on the smooth allocation family, we denote the clipped policy~\eqref{eq::clipped-allocation} by
\begin{align}
    \pi^{\gamma}(a \mid \bX_t, \hat\bbeta_{t-1}) = \bigl[\mathrm{Clip}\bigl(\brho_\gamma(s_{\hat\bbeta_{t-1}}(\bX_t));\pi_{\min}\bigr)\bigr]_a, \quad \text{where } s_{\bbeta}(\bx) = (\bbeta_1^\top \bx, \ldots, \bbeta_K^\top \bx),
\end{align}
which satisfies $\pi^{\gamma}(a\mid\bX_t,\hat\bbeta_{t-1})\ge\pi_{\min}$ for all scores.

\subsubsection{Stochastic approximation setup}

Define $\bX_{t, a} = \bX_t 1_{\{A_t = a\}}$, $Y_{t, a} = Y_t 1_{\{A_t = a\}}$, and
\begin{equation}
    \bPhi_{t, a} = \frac{1}{t-1} \left(\lambda I + \sum_{i=1}^{t-1} \bX_{i, a} \bX_{i, a}^\top\right), \quad \bvarphi_{t, a} = \frac{1}{t-1} \left(\sum_{i=1}^{t-1} \bX_{i, a} Y_{i, a}\right), \label{eq:sample-cov-revised}
\end{equation}
so that $\hat \bbeta_{t, a}^{\Ridge} = \bPhi_{t, a}^{-1} \bvarphi_{t, a}$. Let $\bm{\psi}_{t, a} = (\bvarphi_{t, a}^\top, \Vec(\bPhi_{t, a})^\top)^\top$ and $\bm{\psi}_{t} = (\bm{\psi}_{t, 1}^\top, \ldots, \bm{\psi}_{t, K}^\top)^\top$. The recursion takes the standard stochastic approximation form
\begin{align}
    \bm{\psi}_{t+1} = \bm{\psi}_{t} + \tfrac{1}{t} \left(\bm{g}(\bm{\psi}_{t}) - \bm{\psi}_{t} + \bm{M}_t \right),
\end{align}
where $\bg(\bm{\psi})$, $\bm{M}_t$ are defined as in the standard decomposition:
\begin{align}
    \bg_a(\bm{\psi}) &= (\boldsymbol{m}_a(\bm{\psi}), \mathrm{vec}(\bH_a(\bpsi))), \qquad
    \bM_{t, a} =
    \begin{pmatrix}
        \bX_{t, a} Y_{t,a} \\
        \Vec(\bX_{t, a}\bX_{t, a}^\top)
    \end{pmatrix} - \bg_a(\bm{\psi}_{t}),
\end{align}
and 
\begin{align*}
\boldsymbol{m}_a(\bm{\psi}) &= \EE[\pi^{\gamma}(a|\bX, \bbeta(\bpsi))\cdot \bX y(\bX, a)],\\
\bH_a(\bpsi) &= \EE[\pi^{\gamma}(a|\bX, \bbeta(\bpsi))\cdot \bX \bX^\top].
\end{align*}
Here $\bbeta(\bpsi)$ is a stacked vector of $\bbeta_a(\bpsi) = \bPhi_a^{-1}\bphi_a$ over $a\in\cA$.

Convergence follows from verifying the standard conditions on a high-probability event:
\begin{itemize}
    \item[(C.1)] $\sup_t \EE[\|\bm{\psi}_t\|^2_2] < \infty$;
    \item[(C.2)] $\|\bg(\bm{\psi}) - \bg(\bm{\psi}') \| \leq \kappa \|\bm{\psi} - \bm{\psi}' \|$ for some $\kappa < 1$;
    \item[(C.3)] $\EE[\bm{M}_t \mid \cH_{t-1}] = \mathbf{0}$ and $\EE[\|\bm{M}_t\|^2_2 \mid \cH_{t-1}] \leq C(1+ \|\bm{\psi}_t\|^2_2)$ a.s.
\end{itemize}

\subsubsection{Almost-sure confinement under the global floor}


Define the constants
\begin{align}\label{eq:constants-def}
    C_\varphi := M R_{y} + 4M\sigma_{\eta}\sqrt{dK}, \qquad
    R_{\bbeta} := \frac{2 C_\varphi}{\pi_{\min}\sigma_{\min}^2} + 1, \qquad
    R := \sqrt{K}\,M R_{\bbeta},
\end{align}
so that the \emph{self-consistency condition}
\begin{align}\label{eq:self-consistency}
    \frac{2 C_\varphi}{\pi_{\min} \sigma_{\min}^2} < R_{\bbeta}
\end{align}
holds by construction. 

Define the deterministic set
\[
\mathcal K
:=
\left\{
\bpsi=(\bvarphi_a^\top,\mathrm{vec}(\bPhi_a)^\top)_{a\in\mathcal A}^\top:
\lambda_{\min}(\bPhi_a)\ge \frac12\pi_{\min}\sigma_{\min}^2,\ 
\|\bvarphi_a\|_2\le C_{\varphi},\ \forall a\in\mathcal A
\right\}.
\]

For $t\ge1$ define the good event $\cE_t := \cE_{1,t} \cap \cE_{2,t}$, where
\begin{align}
    \cE_{1, t} &:= \left\{\lambda_{\min}(\bPhi_{\tau, a}) \geq \pi_{\min} \sigma_{\min}^2 - M^2 \sqrt{\frac{16\log(\tau K d)}{\tau-1}}, \;\; \forall\, \tau \geq t,\; \forall\, a \in \cA \right\}, \label{eq:good-event-1-revised} \\
    \cE_{2, t} &:= \left\{\|\bvarphi_{\tau, a}\|_2 \leq C_\varphi, \;\; \forall\, \tau \geq t,\; \forall\, a \in \cA\right\}. \label{eq:good-event-2-revised}
\end{align}

\begin{lemma}[Almost-sure confinement]\label{lem::good-event-master}
    There exists $t_{\mathrm{conc}} \in \NN$ (depending on $\pi_{\min}$, $\sigma_{\min}$, $M$, $K$, $d$) such that $\PP(\cE_{t_0}^c) \leq \frac{2}{t_0-1}$ for every $t_0 \geq t_{\mathrm{conc}}$. Consequently $\PP\bigl(\bigcup_{t_0\ge t_{\mathrm{conc}}}\cE_{t_0}\bigr)=1$: almost surely, $\lambda_{\min}(\bPhi_{t,a})\ge\tfrac12\pi_{\min}\sigma_{\min}^2$ and $\|\bvarphi_{t,a}\|_2\le C_\varphi$ for all $a\in\cA$ and all $t$ beyond some finite (random) time, and hence $\bm\psi_t\in\cK$ eventually almost surely.
\end{lemma}

\begin{proof}
We first prove that $\cE_{2,t_0}$ holds with probability $\ge 1-1/(t_0-1)$. For $a \in \cA$ and $\tau \geq t_0$,
\begin{align}
    \left\|\bvarphi_{\tau, a}\right\|_2
    \leq M R_{y} + \frac{1}{\tau-1} \left\|\sum_{k = 1}^{\tau-1} \bX_{k,a} \eta_k\right\|_2,
    \qquad \eta_k := Y_k - y(\bX_k, A_k).
\end{align}
For the noise term, apply Kolmogorov's inequality coordinate-wise with a dyadic peeling argument: for the block $\sB_m = \{n : 2^m t_0 - 1 \leq n \leq 2^{m+1}t_0 - 1\}$,
\begin{align}
    \PP\!\left(\exists n \in \sB_m : \left|\tfrac{1}{n-1}\textstyle\sum_{k=1}^{n-1}(\bX_{k,a}\eta_k)_i\right| > \varepsilon\right)
    \leq \frac{8M^2\sigma_\eta^2}{2^m t_0\,\varepsilon^2}.
\end{align}
Summing over $m \geq 0$ and union-bounding over $i \in [d]$, $a \in [K]$ with $\varepsilon = 4M\sigma_\eta\sqrt{dK}$ gives $\PP(\cE_{2,t_0}^c) \leq 1/(t_0-1)$.

We then prove that $\cE_{1,t_0}$ holds with probability $\ge 1-1/(t_0-1)$. {Because  $\pi_t(a\mid\bX_t)\ge\pi_{\min}$ a.s.,}
\begin{align}\label{eq:conditioning-bound-clip}
    \EE[\bX_{k,a}\bX_{k,a}^{\top} \mid \cH_{k-1}]
    = \EE[\pi_k(a\mid\bX_k)\bX_k\bX_k^\top \mid \cH_{k-1}]
    \succeq \pi_{\min}\,\EE[\bX_k\bX_k^\top]\succeq \pi_{\min}\sigma_{\min}^2\bI.
\end{align}
The differences $\bX_{k,a}\bX_{k,a}^\top - \EE[\bX_{k,a}\bX_{k,a}^\top \mid \cH_{k-1}]$ satisfy $\|\cdot\|_2 \leq M^2$, so by Corollary~\ref{cor::matrix-azuma}, with probability $\geq 1 - 1/(t_0-1)$, simultaneously for all $a \in \cA$ and $\tau \geq t_0$,
\begin{align}\label{eq:concentration-unconditional}
    \left\|\frac{1}{\tau-1} \sum_{k = 1}^{\tau-1} \left(\bX_{k,a}\bX_{k,a}^{\top} - \EE[\bX_{k,a}\bX_{k,a}^{\top} \mid \cH_{k-1}]\right)\right\|_2 \leq M^2\sqrt{\frac{16\log(K\tau d)}{\tau-1}},
\end{align}
which together with (\ref{eq:conditioning-bound-clip}) yields $\lambda_{\min}(\bPhi_{\tau,a})\ge\pi_{\min}\sigma_{\min}^2 - M^2\sqrt{16\log(K\tau d)/(\tau-1)}$, i.e.\ $\cE_{1,t_0}$. A union bound gives $\PP(\cE_{t_0}^c)\le 2/(t_0-1)$.

Now choose $t_{\mathrm{conc}}$ so that $M^2\sqrt{16\log(K\tau d)/(\tau-1)} \leq \frac{1}{2}\pi_{\min}\sigma_{\min}^2$ for all $\tau \geq t_{\mathrm{conc}}$. Then on $\cE_{t_0}$ (any $t_0\ge t_{\mathrm{conc}}$), for all $\tau\ge t_0$,
\begin{align}\label{eq:eigenvalue-lower}
    \lambda_{\min}(\bPhi_{\tau,a}) \geq \tfrac{1}{2}\pi_{\min}\sigma_{\min}^2,
    \qquad
    \|\hat\bbeta_{\tau,a}\|_2\le\|\bPhi_{\tau,a}^{-1}\|_2\|\bvarphi_{\tau,a}\|_2\le\frac{2C_\varphi}{\pi_{\min}\sigma_{\min}^2}\overset{(\ref{eq:self-consistency})}{<}R_{\bbeta},
\end{align}
so $\bm\psi_\tau\in\cK$. Finally, the events $\cE_{t_0}$ increase in $t_0$ and $\PP(\cE_{t_0})\ge 1-2/(t_0-1)\to1$, so $\PP(\bigcup_{t_0\ge t_{\mathrm{conc}}}\cE_{t_0})=1$.
\end{proof}

\paragraph{Contraction property of $\bg$}
Note that the set $\cK$ is closed and convex, and by Lemma~\ref{lem::good-event-master} the iterates $\{\bm\psi_\tau\}$ lie in $\cK$ for all $\tau$ beyond some finite (random) time, almost surely.
\begin{lemma}\label{lem::contraction-revised}
    For $\gamma$ satisfying (\ref{eq:Ridge-Lip-condition}), the map $\bg$ is a contraction on $\cK$:
    \begin{equation}
        \|\bm{g}(\bm{\psi}) - \bm{g}(\bm{\psi}') \|_2 \leq \kappa \|\bm{\psi} - \bm{\psi}' \|_2, \quad \kappa < 1, \qquad \forall\, \bm{\psi}, \bm{\psi}' \in \cK.
    \end{equation}
    {Moreover, $\bg(\cK)\subseteq\cK$, so $\bg$ has a unique fixed point $\bm\psi^*\in\cK$.}
\end{lemma}

\begin{proof}
    Fix $\bm\psi\in\cK$. By definition of $\cK$, $\|\bPhi_{a}^{-1}\|_2 \leq 2/(\pi_{\min}\sigma_{\min}^2)$ and $\|\bvarphi_{a}\|_2 \leq C_\varphi$, so $\|\bbeta_a\|_2=\|\bPhi_a^{-1}\bvarphi_a\|_2\le 2C_\varphi/(\pi_{\min}\sigma_{\min}^2)<R_{\bbeta}$ by (\ref{eq:self-consistency}); hence the scores satisfy $\|s_{\bbeta}(\bX)\|_2\le R$, and $\pi^{\gamma}(a\mid\bX,\bbeta)\ge\pi_{\min}$ by the global clip~\eqref{eq::clipped-allocation}. Thus we have 
    \begin{align*}
    \|\boldsymbol{m}_{a}(\bm{\psi}) -\boldsymbol{m}_{a}(\bm{\psi}')\|_2& \leq \EE\left[|\pi^{\gamma}(a \mid \bX_t, \bbeta(\bm{\psi})) - \pi^{\gamma}(a \mid \bX_t, \bbeta(\bm{\psi}'))|\right]\\
    & \leq L_{\brho}(\gamma, R) M(MR_y + M\sigma_\eta)\cdot \|\bbeta(\bm{\psi}) - \bbeta(\bm{\psi}')\|_2, \\
    \|\bH_{a}(\bm{\psi}) - \bH_{a}(\bm{\psi}')\|_2 &\leq \EE\left[|\pi^{\gamma}(a \mid \bX_t, \bbeta(\bm{\psi})) - \pi^{\gamma}(a \mid \bX_t, \bbeta(\bm{\psi}'))|\right]\\
    &\leq L_{\brho}(\gamma, R) M^3\|\bbeta(\bm{\psi}) - \bbeta(\bm{\psi}')\|_2.
    \end{align*}
    Here we have used the fact that the scores are bounded, $\{\brho_{\gamma}\}_\gamma$ is a smooth allocation family, the clip operator is 1-Lipshitz due to being a projection, and the boundedness assumptions.
    Also, using $\|\bPhi_a^{-1}\|_2, \|\bPhi_a^{'-1}\|_2 \leq 2/(\pi_{\min}\sigma_{\min}^2)$ and $\|\bvarphi_a'\|_2 \leq C_\varphi$, we have $\forall a\in\cA$,
    \begin{align}
        \|\bbeta_a(\bm{\psi}) - \bbeta_a(\bm{\psi}')\|_2 \leq \frac{2}{\pi_{\min}\sigma_{\min}^2} \|\bvarphi_a - \bvarphi_a'\|_2 + \frac{4 C_\varphi}{\pi^2_{\min}\sigma_{\min}^4}  \|\bPhi_a - \bPhi'_{a}\|_F .
    \end{align}
    Combining the above, the Lipschitz constant of $\bg$ is at most
    \begin{align}
        L_{\brho}(\gamma, R) \cdot \sqrt{2K}M^2 \cdot (R_{y} + \sigma_\eta + M)\max\!\left\{\frac{2}{\pi_{\min}\sigma_{\min}^2},\; \frac{4 C_\varphi}{\pi^2_{\min}\sigma_{\min}^4}\right\},
    \end{align}
    which is $< 1$ under (\ref{eq:Ridge-Lip-condition}).
    
    Finally, $\bg(\cK)\subseteq\cK$: for $\bm\psi\in\cK$, $a\in\cA$,  $\|\boldsymbol{m}_a(\bpsi)\|_2=\|\EE[\pi^{\gamma}(a\mid\bX,\bbeta(\bm\psi))
    \bX Y(a)]\|_2\le M(R_y+\sigma_\eta)\le C_\varphi$, and $\bH_a(\bpsi)=\EE[\pi^{\gamma}(a\mid\bX,\bbeta(\bm\psi))\bX\bX^\top]\succeq\pi_{\min}\sigma_{\min}^2\bI\succeq\tfrac12\pi_{\min}\sigma_{\min}^2\bI$. Thus $\bg$ maps the closed set $\cK$ into itself and is a contraction there, so by Banach's fixed-point theorem (Theorem~\ref{thm:banach}) it has a unique fixed point $\bm\psi^*\in\cK$.
\end{proof}

On $\cK$, the moment bounds $\|\bPhi_{\tau,a}\|_2 \leq M^2$ and $\|\bvarphi_{\tau,a}\|_2 \leq C_\varphi$ hold trivially, and $\EE[\|\bm{M}_{\tau,a}\|_2^2 \mid \cH_{\tau-1}] \leq M^2(R_y + \sigma_\eta)^2 + M^4$, verifying the moment bound (C.3) along the iterates.

\subsubsection{Finishing the proof}

Fix $\gamma\ge\gamma_0$, where $\gamma_0$ is the threshold from the contraction condition~\eqref{eq:Ridge-Lip-condition}; this fixes $R_{\bbeta}$, the clip floor $\pi_{\min}$, and $t_{\mathrm{conc}}$. We prove $\hat\bbeta_t^{\Ridge}\to\bar\bbeta_\gamma$ \emph{almost surely}, which gives convergence in probability for this fixed $\gamma$.

By Lemma~\ref{lem::good-event-master}, the event $G:=\bigcup_{t_0\ge t_{\mathrm{conc}}}\cE_{t_0}$ has $\PP(G)=1$, and on $\cE_{t_0}$ the iterates satisfy $\bm\psi_\tau\in\cK$ for all $\tau\ge t_0$, where (Lemma~\ref{lem::contraction-revised}) $\bg$ is a $\kappa$-contraction on $\cK$ with unique fixed point $\bm\psi^*=\bm\psi^*_\gamma\in\cK$ and the conditional moment bound (C.3) holds.

Write $\be_t:=\bm\psi_t-\bm\psi^*$. The recursion (\ref{eq:sample-cov-revised}) and $\bg(\bm\psi^*)=\bm\psi^*$ give
\[
    \be_{t+1}=\Bigl(1-\tfrac1t\Bigr)\be_t+\tfrac1t\bigl(\bg(\bm\psi_t)-\bg(\bm\psi^*)\bigr)+\tfrac1t\bM_t .
\]
Fix $t_0\ge t_{\mathrm{conc}}$ and let $\sigma:=\inf\{t\ge t_0:\bm\psi_t\notin\cK\}$ (with $\inf\emptyset=\infty$); then $\{\sigma>t\}$ is $\cH_{t-1}$-measurable and $\{\sigma=\infty\}\supseteq\cE_{t_0}$. On $\{\sigma>t\}$, $\bm\psi_t\in\cK$ gives $\|\bg(\bm\psi_t)-\bg(\bm\psi^*)\|_2\le\kappa\|\be_t\|_2$ (Lemma~\ref{lem::contraction-revised}, using $\bm\psi^*\in\cK$) and $\EE[\|\bM_t\|_2^2\mid\cH_{t-1}]\le C$, with $\|\be_t\|_2$ bounded on $\cK$. For the stopped process $\tilde\be_t:=\be_{t\wedge\sigma}$, using $(1-\tfrac{1-\kappa}{t})^2\le 1-\tfrac{2(1-\kappa)}{t}+\tfrac{(1-\kappa)^2}{t^2}$, on the whole probability space
\[
    \EE\bigl[\|\tilde\be_{t+1}\|_2^2\mid\cH_{t-1}\bigr]
    \le\|\tilde\be_t\|_2^2-\tfrac{2(1-\kappa)}{t}\|\tilde\be_t\|_2^2\,1_{\{\sigma>t\}}+\tfrac{C'}{t^2},
    \qquad t\ge t_0,
\]
where $C'$ absorbs the bounded $O(1/t^2)$ terms. By the Robbins--Siegmund theorem (Theorem~\ref{thm:robbins-siegmund}, with $V_t=\|\tilde\be_t\|_2^2$, $\beta_t=0$, $\xi_t=C'/t^2$, $\zeta_t=\tfrac{2(1-\kappa)}{t}\|\tilde\be_t\|_2^2\,1_{\{\sigma>t\}}$), $\|\tilde\be_t\|_2^2$ converges a.s.\ and $\sum_t \tfrac{1}{t}\|\tilde\be_t\|_2^2\,1_{\{\sigma>t\}}<\infty$ a.s. On $\{\sigma=\infty\}$ this reads $\sum_t \tfrac{1}{t}\|\be_t\|_2^2<\infty$, which with $\sum_t t^{-1}=\infty$ forces $\be_t\to0$. Hence $\bm\psi_t\to\bm\psi^*$ on $\cE_{t_0}$. Taking the union over $t_0\ge t_{\mathrm{conc}}$, $\bm\psi_t\to\bm\psi^*$ almost surely on $G$, i.e.\ almost surely.

Since $\hat\bbeta_{t,a}^{\Ridge} = \bPhi_{t,a}^{-1}\bvarphi_{t,a}$ is continuous in $\bm{\psi}_t$, the continuous mapping theorem gives $\hat\bbeta_t^{\Ridge} \xrightarrow{a.s.} \bar\bbeta_\gamma$, hence $\hat\bbeta_t^{\Ridge}\xrightarrow{p}\bar\bbeta_\gamma$; {scaled inverse-propensity convergence} then follows from Theorem~\ref{lem::statistics-converge-implies-policy-converge}. \qed

\subsection{Proof of Proposition \ref{theorem:convergence-of-SGD}}\label{apdx::proof-thm::convergence-of-SGD}

Similar to the above, for simplicity, we denote the clipped policy~\eqref{eq::clipped-allocation} by
\begin{align*}
    \pi^{\gamma}(a \mid \bX_t, \hat\bbeta_{t-1}) = \bigl[\mathrm{Clip}\bigl(\brho_\gamma(s_{\hat\bbeta_{t-1}}(\bX_t));\pi_{\min}\bigr)\bigr]_a, \quad \text{where } s_{\bbeta}(\bx) = (\bbeta_1^\top \bx, \ldots, \bbeta_K^\top \bx).
\end{align*}

The argument mirrors the proof of Proposition~\ref{theorem:convergence-of-Ridge}. Rewrite the update~\eqref{eq::SGD-update} in Robbins--Monro form:
\[
    \hat \bbeta_{t, a}^{\SGD} = \hat \bbeta_{t-1, a}^{\SGD} + \eta_t\,\tilde\bh_a(\hat\bbeta_{t-1}^{\SGD}) + \bM_{t,a}.
\]
The mean field and the martingale-difference noise are
\[
    \tilde{\bh}_a(\bbeta) := \EE\!\left[1_{\{A_t = a\}}\bh(\bX_t, Y_t(a); \bbeta_a) \mid \hat\bbeta_{t-1}^{\SGD} = \bbeta\right],
\]
\[
    \bM_{t,a} := \eta_t\bigl(1_{\{A_t=a\}}\bh(\bX_t, Y_t(a);\hat\bbeta_{t-1,a}^{\SGD}) - \tilde\bh_a(\hat\bbeta_{t-1}^{\SGD})\bigr).
\]
Stacking over $a\in\cA$ gives $\hat\bbeta_t^{\SGD}=\hat\bbeta_{t-1}^{\SGD}+\eta_t\,\tilde\bh(\hat\bbeta_{t-1}^{\SGD})+\bM_t$, with limiting o.d.e.\ $\dot\bbeta=\tilde\bh(\bbeta)$ and equilibrium set $\{\bbeta:\tilde\bh(\bbeta)=0\}$.

\emph{Step 1: $\tilde\bh$ is Lipschitz on $\cB_{\bbeta}$.} Fix $a$ and write $\pi^\gamma_i := \pi^{\gamma}(a\mid\bX_t,\bbeta_i)$ and $\bh_i := \bh(\bX_t, Y_t(a);\bbeta_{i,a})$. Split
\[
    \pi^\gamma_1\bh_1 - \pi^\gamma_2\bh_2
    = (\pi^\gamma_1 - \pi^\gamma_2)\,\bh_1 + \pi^\gamma_2\,(\bh_1 - \bh_2).
\]
Taking expectations and applying the triangle and Cauchy--Schwarz inequalities,
\[
    \|\tilde{\bh}_a(\bbeta_1) - \tilde{\bh}_a(\bbeta_2)\|_2
    \le \sqrt{\EE\bigl[(\pi^\gamma_1 - \pi^\gamma_2)^2\bigr]}\,\sqrt{\EE\|\bh_1\|_2^2}
      + \EE\|\bh_1 - \bh_2\|_2.
\]
The smooth-allocation Lipschitz bound, as in the proof of Proposition~\ref{theorem:convergence-of-Ridge}, gives the pointwise estimate
\[
    |\pi^\gamma_1 - \pi^\gamma_2|\le L_{\brho}(\gamma, \sqrt{K}MR_{\bbeta})\,M\,\|\bbeta_1-\bbeta_2\|_2,
\]
which for Boltzmann exploration reduces to the $2/\gamma$ bound used above. With the growth bound $\EE\|\bh_1\|_2^2\le\rho_{\bh}(1+R_{\bbeta}^2)$ and the Lipschitz bound $\EE\|\bh_1-\bh_2\|_2\le L_{\bh}\|\bbeta_1-\bbeta_2\|_2$,
\[
    \|\tilde{\bh}_a(\bbeta_1) - \tilde{\bh}_a(\bbeta_2)\|_2
    \le \Bigl(L_{\brho}(\gamma, \sqrt{K}MR_{\bbeta})\,M\,\sqrt{\rho_{\bh}(1+R_{\bbeta}^2)} + L_{\bh}\Bigr)\|\bbeta_1-\bbeta_2\|_2.
\]
Hence $\tilde\bh$ is Lipschitz with constant $L_{\tilde\bh}=\sqrt{|\cA|}\bigl(L_{\brho}(\gamma, \sqrt{K}MR_{\bbeta})\,M\,\sqrt{\rho_{\bh}(1+R_{\bbeta}^2)} + L_{\bh}\bigr)$.

\emph{Step 2: martingale-difference noise.} By construction $\EE[\bM_{t,a}/\eta_t\mid\cH_{t-1}]=0$. By $(a+b)^2\le 2a^2+2b^2$ and the growth bound (ii), whenever $\hat\bbeta_{t-1}^{\SGD}\in \cB_{\bbeta}$, 
\[
    \EE\bigl[\|\bM_t/\eta_t\|_2^2 \mid \cH_{t-1}\bigr]
    \le 2\sum_{a\in\cA}\EE\bigl[\|\bh(\bX_t, Y_t(a); \hat\bbeta_{t-1,a}^{\SGD})\|_2^2\mid\cH_{t-1}\bigr]
    \le 2|\cA|\,\rho_{\bh}\bigl(1+R_{\bbeta}^2\bigr).
\]

\emph{Step 3: a.s.\ eventual confinement.} The iterates enter and remain in $\cB_{\bbeta}$ almost surely. For $t_0\in\NN$, set
\[
    \sigma_{t_0}:=\inf\{t\ge t_0:\hat\bbeta_t^{\SGD}\notin\cB_{\bbeta}\}, \qquad \inf\emptyset=\infty .
\]
Condition (i) of Proposition~\ref{theorem:convergence-of-SGD} states $\limsup_t\|\hat\bbeta_{t,a}^{\SGD}\|_2<R_{\bbeta}$ a.s.\ for every $a$. On the probability-one event where this holds, each path satisfies $\sup_a\|\hat\bbeta_{t,a}^{\SGD}\|_2<R_{\bbeta}$ for all $t\ge t_0(\omega)$ with some finite $t_0(\omega)$. Hence $\hat\bbeta_t^{\SGD}\in\cB_{\bbeta}$ for all $t\ge t_0(\omega)$, so $\sigma_{t_0(\omega)}(\omega)=\infty$. Therefore
\[
    \PP\Bigl(\textstyle\bigcup_{t_0\in\NN}\{\sigma_{t_0}=\infty\}\Bigr)=1 .
\]
Proposition~\ref{prop::SGD-boundedness} verifies Condition (i) for the example scores: the global clip induces an inward Lyapunov drift, yielding $\limsup_t\|\hat\bbeta_{t,a}^{\SGD}\|_2\le R_0<R_{\bbeta}$ a.s.

\emph{Step 4: unique equilibrium.} The dissipativity assumption supplies $\mu>0$ with
\[
    \langle\bbeta-\bbeta',\,\tilde\bh(\bbeta)-\tilde\bh(\bbeta')\rangle\le-\mu\|\bbeta-\bbeta'\|_2^2, \qquad \bbeta,\bbeta'\in\cB_{\bbeta}.
\]
Two equilibria in $\cB_{\bbeta}$ give $0\le-\mu\|\bbeta-\bbeta'\|_2^2$, so the equilibrium in $\cB_{\bbeta}$ is unique; call it $\bar\bbeta=\bar\bbeta_\gamma$, with $\tilde\bh(\bar\bbeta)=0$.

\emph{Step 5: convergence at fixed $\gamma$.} Fix $\gamma\ge\gamma_0$, which fixes $R_{\bbeta}$, the clip floor $\pi_{\min}$, and $\bar\bbeta$. Set $\be_t:=\hat\bbeta_t^{\SGD}-\bar\bbeta$. Since $\tilde\bh(\bar\bbeta)=0$,
\[
    \be_t=\be_{t-1}+\eta_t\,\bd_{t-1}+\bM_t, \qquad \bd_{t-1}:=\tilde\bh(\hat\bbeta_{t-1}^{\SGD})-\tilde\bh(\bar\bbeta).
\]

Fix $t_0\in\NN$ and let $\sigma:=\sigma_{t_0}$ from Step 3. The event $\{\sigma\ge t\}=\{\hat\bbeta_s^{\SGD}\in\cB_{\bbeta},\ t_0\le s\le t-1\}$ is $\cH_{t-1}$-measurable. Define the stopped error $\tilde\be_t:=\be_{t\wedge\sigma}$, so that $\tilde\be_t-\tilde\be_{t-1}=(\be_t-\be_{t-1})\,1_{\{\sigma\ge t\}}$ for $t>t_0$.

On $\{\sigma\ge t\}$ the iterate $\hat\bbeta_{t-1}^{\SGD}\in\cB_{\bbeta}$. Expand $\|\be_t\|_2^2$ and take $\EE[\cdot\mid\cH_{t-1}]$; the cross terms in $\bM_t$ vanish in conditional mean, giving
\[
    \EE\bigl[\|\be_t\|_2^2\mid\cH_{t-1}\bigr]
    =\|\be_{t-1}\|_2^2+2\eta_t\langle\be_{t-1},\bd_{t-1}\rangle+\eta_t^2\|\bd_{t-1}\|_2^2+\EE\bigl[\|\bM_t\|_2^2\mid\cH_{t-1}\bigr].
\]
Dissipativity (Step 4) gives $\langle\be_{t-1},\bd_{t-1}\rangle\le-\mu\|\be_{t-1}\|_2^2$. Lipschitzness (Step 1) gives $\|\bd_{t-1}\|_2\le L_{\tilde\bh}D$, where $D:=\sup_{\bbeta\in\cB_{\bbeta}}\|\bbeta-\bar\bbeta\|_2\le 2\sqrt{K}R_{\bbeta}$. Step 2 gives $\EE[\|\bM_t\|_2^2\mid\cH_{t-1}]\le 2|\cA|\rho_{\bh}(1+R_{\bbeta}^2)\,\eta_t^2$. Hence, with $C':=L_{\tilde\bh}^2D^2+2|\cA|\rho_{\bh}(1+R_{\bbeta}^2)$,
\[
    \EE\bigl[\|\be_t\|_2^2\mid\cH_{t-1}\bigr]\le\|\be_{t-1}\|_2^2-2\mu\eta_t\|\be_{t-1}\|_2^2+C'\eta_t^2 \qquad\text{on } \{\sigma\ge t\}.
\]

Since $1_{\{\sigma\ge t\}}$ is $\cH_{t-1}$-measurable and $\tilde\be_t=\tilde\be_{t-1}$ on $\{\sigma<t\}$,
\[
    \EE\bigl[\|\tilde\be_t\|_2^2\mid\cH_{t-1}\bigr]\le\|\tilde\be_{t-1}\|_2^2-2\mu\eta_t\|\tilde\be_{t-1}\|_2^2\,1_{\{\sigma\ge t\}}+C'\eta_t^2 .
\]
Apply the Robbins--Siegmund theorem (Theorem~\ref{thm:robbins-siegmund}) with $V_t=\|\tilde\be_t\|_2^2$, $\beta_t=0$, $\xi_t=C'\eta_t^2$, and $\zeta_t=2\mu\eta_t\|\tilde\be_{t-1}\|_2^2\,1_{\{\sigma\ge t\}}$. Since $\sum_t\eta_t^2<\infty$, $\|\tilde\be_t\|_2^2$ converges a.s.\ and
\[
    \sum_t\eta_t\|\tilde\be_{t-1}\|_2^2\,1_{\{\sigma\ge t\}}<\infty \qquad\text{a.s.}
\]
On $\{\sigma_{t_0}=\infty\}$ this reads $\sum_t\eta_t\|\be_{t-1}\|_2^2<\infty$, which with $\sum_t\eta_t=\infty$ forces $\be_t\to0$. Since $t_0\in\NN$ was arbitrary and $\PP(\bigcup_{t_0}\{\sigma_{t_0}=\infty\})=1$ by Step 3, $\be_t\to0$ a.s., i.e.\ $\hat\bbeta_t^{\SGD}\to\bar\bbeta$ almost surely.

Therefore $\hat\bbeta_t^{\SGD}\xrightarrow{p}\bar\bbeta_\gamma$, and scaled inverse-propensity convergence (Definition~\ref{aspt:policy_convergence}) follows from Theorem~\ref{lem::statistics-converge-implies-policy-converge}. The four conditions are verified for the example score functions in Proposition~\ref{prop::SGD-boundedness}.

\subsection{Proof of Proposition \ref{prop::SGD-boundedness}}
    \label{appx::prop::SGD-boundedness}

We verify the four conditions of Proposition~\ref{theorem:convergence-of-SGD} for the three example scores. Each is affine in $\bbeta$,
\[
    \bh(\bx,y;\bbeta_a)=\bb_a(\bx,y)-\bA(\bx)\bbeta_a,
\]
with $\bA(\bx)=\bx\bx^\top$ (Example~\ref{ex::misspecified-linear-bandits}), $\bA(\bx)=\bx\bx^\top-\bSigma_e$ (Example~\ref{ex::bandits-noisy-contexts}), and $\bA(\bx)=\bI$ (Example~\ref{ex::ope}).
For Example~\ref{ex::ope}, $\bh$ is action-specific and scalar-valued; the same argument applies with $\bh_a$ replacing a common $\bh$.

For the Lipschitz condition: Since $\bh(\bx,y;\bbeta_a)-\bh(\bx,y;\bbeta_a')=-\bA(\bx)(\bbeta_a-\bbeta_a')$, the map $\bh$ is $L_{\bh}$-Lipschitz in $\bbeta$ with $L_{\bh}=\sup_{\|\bx\|_2\le M}\|\bA(\bx)\|_2<\infty$ under Assumption~\ref{aspt:Ridge_Conv}.
The growth and boundedness conditions, existence of equilibrium and the dissipativity condition are verified in order.

Now we first verify the growth condition of Proposition~\ref{theorem:convergence-of-SGD}, namely
        \begin{align}
            \EE[\|\bh(\bX_t, Y_t(a); \bbeta)\|_2^2] \leq \rho_{\bh}(1+\|\bbeta\|_2^2).
        \end{align}
        This assumption is immediately satisfied from the almost surely boundedness of $\|\bX_t\|_2$ and the boundedness of $\EE[Y_t^2]$ given by Assumption \ref{aspt:Ridge_Conv} for all three cases.

        \textbf{Example \ref{ex::bandits-noisy-contexts}.} Write 
        \begin{align}
            \|\bh(\bX_t, Y_t(a); \bbeta)\| \leq |Y_t(a)| \|\bX_t\| + (\|\bX_t\|^2 + \|\Sigma_e\|) \|\bbeta\|.
        \end{align}
        Squaring and taking expectation gives
        \begin{align}
&\EE\|\bh(\bX_t,Y_t(a);\bbeta)\|^2\\
\le& 2\EE\big[Y_t(a)^2\|\bX_t\|^2\big]+2\EE\big[(\|\bX_t\|^2+\|\Sigma_e\|)^2\big]\|\bbeta\|^2\\
\le& 2M^2(R_y^2 + \sigma_{\eta}^2)+2(M^2 + \|\bSigma_e\|)^2\|\bbeta\|^2 \\
\le& \max\{2M^2(R_y^2 + \sigma_{\eta}^2),\ 2(M^2 + \|\bSigma_e\|)^2\} (1+\|\bbeta\|^2).
        \end{align}
Example \ref{ex::misspecified-linear-bandits} is a special case of \ref{ex::bandits-noisy-contexts} with $\Sigma_e = 0$.

\textbf{Example \ref{ex::ope}: $\bh_a(\bx, y; \bbeta) = \pi^e(a \mid \bx)y-\bbeta$.} Here $\|\bh_a(\bX_t, Y_t(a); \bbeta)\| \leq |\pi^e(a \mid \bX_t)| |Y_t(a)| + \|\bbeta\|$. Since $0\leq\pi^e(\cdot \mid \bx) \leq 1$ for all $\bx$, 
\begin{align}
    \EE \|\bh_a(\bX_t, Y_t(a);\bbeta)\|^2 \leq 2 \EE[Y_t(a)^2] + 2\|\bbeta\|^2 \leq 2\max\{R_y^2 + \sigma_{\eta}^2, 1\} (1+\|\bbeta\|_2^2).
\end{align}

For the a.s.\ eventual confinement via inward drift: The affine mean field is $\tilde\bh_a(\bbeta)=\bvarphi_a(\bbeta)-\bSigma_a(\bbeta)\bbeta_a$, with
\[
\bvarphi_a(\bbeta)=\EE[\pi^{\gamma}(a\mid\bX_t,\bbeta)\,\bb_a(\bX_t,Y_t)], \qquad
    \bSigma_a(\bbeta)=\EE[\pi^{\gamma}(a\mid\bX_t,\bbeta)\,\bA(\bX_t)].
\]
The global clip gives $\bSigma_a(\bbeta)\succeq\underline c\,\bI$ uniformly in $\bbeta$, with $\underline c=\pi_{\min}\sigma_{\min}^2$ (Example~\ref{ex::misspecified-linear-bandits}), $\underline c=\pi_{\min}$ (Example~\ref{ex::ope}), and $\underline c=\lambda_{\min}\bigl(\EE[\pi^{\gamma}(a\mid\bX_t,\bbeta)(\bX_t\bX_t^\top-\bSigma_e)]\bigr)>0$ (Example~\ref{ex::bandits-noisy-contexts}, under the condition of the dissipativity step). Since $|\pi^{\gamma}|\le1$ and $\EE|Y_t|\le R_y+\sigma_\eta$,
        \[
            \|\bvarphi_a(\bbeta)\|_2\le\EE\|\bb_a(\bX_t,Y_t)\|_2\le B_\varphi, \qquad
            B_\varphi:=M(R_y+\sigma_\eta)
        \]
        ($B_\varphi=R_y+\sigma_\eta$ for Example~\ref{ex::ope}, where $\bb_a=\pi^e(a\mid\bX_t)Y_t$). Hence
        \[
            \langle\bbeta_a,\tilde\bh_a(\bbeta)\rangle
            =\langle\bbeta_a,\bvarphi_a(\bbeta)\rangle-\bbeta_a^\top\bSigma_a(\bbeta)\bbeta_a
            \le B_\varphi\|\bbeta_a\|_2-\underline c\,\|\bbeta_a\|_2^2,
        \]
        which is strictly negative for $\|\bbeta_a\|_2>R_0:=B_\varphi/\underline c$: the mean field points inward outside the deterministic ball of radius $R_0$. With the growth bound above and $\sum_t\eta_t^2<\infty$, the Lyapunov-drift stability theorem for stochastic approximation \citep[Ch.~3]{borkar2008stochastic} gives
        \[
            \limsup_{t\to\infty}\|\hat\bbeta_{t,a}^{\SGD}\|_2\le R_0\quad\text{a.s.}, \qquad a\in\cA.
        \]
        Thus, for any deterministic $R_{\bbeta}>R_0$, we have $\limsup_{t\to\infty}\|\hat\bbeta_{t,a}^{\SGD}\|_2\le R_0<R_{\bbeta}$ a.s., which is the confinement condition of Proposition~\ref{theorem:convergence-of-SGD}.

We next verify that there exists $R_{\bbeta}$ such that the limiting ODE has an equilibrium in \(\cB_{\bbeta}:= \{\bbeta\in\RR^{d_\beta}: \max_a\|\bbeta_a\|_2\leq R_{\bbeta}\} \). By the uniform curvature bound established above, there exists \(c>0\) such that $\bSigma_a(\bbeta)\succeq c \bI$ uniformly over \(a\in\mathcal A\) and \(\bbeta\in \cB_{\bbeta}\). Moreover, the preceding moment bounds imply that there exists a finite constant \(B_\phi\) such that $\sup_{a\in\mathcal A}\sup_{\bbeta\in \cB_{\bbeta}}\|\phi_a(\bbeta)\|_2\le B_\phi$. For each \(a\in\mathcal A\), define
\[
\bT_a(\bbeta)=\bSigma_a(\bbeta)^{-1}\bphi_a(\bbeta),
\qquad
\bT(\bbeta)=\bigl(\bT_a(\bbeta)\bigr)_{a\in\mathcal A}.
\]
Then, for every \(\bbeta\in \cB_{\bbeta}\),
\[
\|\bT_a(\bbeta)\|_2
\le
\|\bSigma_a(\bbeta)^{-1}\|_2\|\bphi_a(\bbeta)\|_2
\le
\frac{B_\phi}{c}.
\]
Choose \(R_{\bbeta}>B_\phi/c\). It follows that \(\bT\) maps the compact convex set $\cB_{\bbeta}$ into itself. In addition, \(\bT\) is continuous on $\cB_{\bbeta}$, since the clipped smooth-allocation policy is continuous in \(\bbeta\), the action-score map is continuous, and \(\bSigma_a(\bbeta)\) is uniformly nonsingular. Therefore, by Brouwer's fixed point theorem, there exists \(\bar\bbeta_\gamma\in \cB_{\bbeta}\) such that
\[
\bT(\bar\bbeta_\gamma)=\bar\bbeta_\gamma .
\]
Then, \(\bar\bbeta_\gamma\in \cB_{\bbeta}\) is an equilibrium of the limiting ODE.

It remains to verify the dissipativity condition. With the mean field $\tilde\bh_a(\bbeta)=\bvarphi_a(\bbeta)-\bSigma_a(\bbeta)\bbeta_a$ defined above, decompose
\[
    \tilde\bh_a(\bbeta)-\tilde\bh_a(\bbeta')
    = -\bSigma_a(\bbeta)(\bbeta_a-\bbeta_a')
      + \bigl[\bvarphi_a(\bbeta)-\bvarphi_a(\bbeta')\bigr]
      - \bigl[\bSigma_a(\bbeta)-\bSigma_a(\bbeta')\bigr]\bbeta_a'.
\]
Since the policy is clipped at the floor $\pi_{\min}>0$, we have $\bSigma_a(\bbeta)\succeq\underline{c}\,\bI$ uniformly in $\bbeta$, where
\[
    \underline{c}=\pi_{\min}\sigma_{\min}^2 \ \text{(Example~\ref{ex::misspecified-linear-bandits})}, \qquad
    \underline{c}=\pi_{\min} \ \text{(Example~\ref{ex::ope})},
\]
and $\underline{c}=\lambda_{\min}\bigl(\EE[\pi^{\gamma}(a\mid\bX_t,\bbeta)(\bX_t\bX_t^\top-\bSigma_e)]\bigr)$ for Example~\ref{ex::bandits-noisy-contexts}, which is positive when $\bSigma_e$ is small relative to $\bSigma_S$. (For Example~\ref{ex::misspecified-linear-bandits} and~\ref{ex::ope} this uses $\pi^{\gamma}(a\mid\bx,\bbeta)\ge\pi_{\min}$, and for Example~\ref{ex::misspecified-linear-bandits}, together with $\EE[\bX_t\bX_t^\top]\succeq\sigma_{\min}^2\bI$.) The last two terms depend on $\bbeta$ only through the allocation map, so each is bounded by $C\,L_{\brho}(\gamma,\sqrt{K}MR_{\bbeta})\|\bbeta-\bbeta'\|_2$ with $C=C(M,R_y,R_{\bbeta},|\cA|)$. Summing over $a$,
\[
    \langle\bbeta-\bbeta',\,\tilde\bh(\bbeta)-\tilde\bh(\bbeta')\rangle
    \le -\bigl(\underline{c}-C\,L_{\brho}(\gamma,\sqrt{K}MR_{\bbeta})\bigr)\|\bbeta-\bbeta'\|_2^2.
\]
This is the dissipativity condition with $\mu=\underline{c}-C\,L_{\brho}(\gamma,\sqrt{K}MR_{\bbeta})>0$ once $\gamma$ is large enough that $C\,L_{\brho}(\gamma,\sqrt{K}MR_{\bbeta})<\underline{c}$, which holds for all $\gamma\ge\gamma_0$ since $L_{\brho}(\gamma,\sqrt{K}MR_{\bbeta})\to0$. 

Summarizing from the above, all four conditions of Proposition~\ref{theorem:convergence-of-SGD} hold, and its conclusion follows.

\subsection{Proof of Theorem \ref{theorem:necessary-condition-for-Ridge-convergence}}\label{apdx::proof-thm::necessary-condition-for-Ridge-convergence}

We prove by contradiction. Suppose $\hat{\bm{\beta}}_{t}^{\Ridge}$ converge in probability, then there exists a random vector $\bar{\bbeta}$ such that $\hat{\bm{\beta}}_{t}^{\Ridge} \xrightarrow{p} \bar{\bbeta}$. Recall the definition of $\hat{\bm{\beta}}_{t}^{\Ridge}$ in (\ref{eq:ridge-estimator}):
\begin{align}
    \hat{\bm{\beta}}_{t, a}^{\Ridge} = \left(\lambda 
    \bI + \sum_{i=1}^{t-1} \bX_{i, a} \bX_{i, a}^\top\right)^{-1} \sum_{i=1}^{t-1} \bX_{i, a} Y_{i, a}.
\end{align}
Here we define $\bX_{t, a} = \bX_t 1_{\{A_t = a\}}$, and $Y_{t, a} = Y_t 1_{\{A_t = a\}}$.


We first note that $\bSigma_{a}(\bm{\beta}) = \EE_{\bX_t, A_t \sim \pi(\cdot \mid \bX_t, \bm{\beta})}[ 1_{\{A_t = a\}} \bX_t \bX_t^\top]$ is a continuous function of $\bm{\beta}$ given any $a \in \mathcal{A}$, because we assume that $\pi(a \mid \bx, \bm{\beta})$ is a continuous function of $\bm{\beta}$ given any $\bx \in \mathcal{X}$. From the continuous mapping theorem, we have $\bSigma_{a}(\hat{\bm{\beta}}_{t-1}^{\Ridge}) \xrightarrow{p} \bSigma_a(\bar{\bm{\beta}})$, and
\begin{align}
    \frac{1}{t} \left(\lambda \bI + \sum_{\tau=1}^{t-1} \bSigma_{a}(\hat{\bm{\beta}}_{\tau, a}^{\Ridge}) \right) \xrightarrow{p} \bSigma_a(\bar{\bm{\beta}}).
\end{align}
Further, since $\bX_{t, a}\bX_{t, a}^\top - \bSigma_{a}(\hat{\bm{\beta}}_{t}^{\Ridge})$ is a zero-mean martingale difference sequence with bounded second moment, by the law of large numbers (Lemma \ref{lem::law-of-large-numbers-for-martingale-difference-sequence})\citep{chow1967strong}, we have that 
\begin{align}
    \frac{1}{t} \sum_{\tau=1}^{t-1} \left(\bX_{\tau, a}\bX_{\tau, a}^\top - \bSigma_{a}(\hat{\bm{\beta}}_{\tau, a}^{\Ridge})\right) \xrightarrow{p} 0.
\end{align}
Combining the above, we deduce that 
\begin{align}
    \frac{1}{t} \left(\lambda \bI + \sum_{\tau=1}^{t-1} \bX_{\tau, a}\bX_{\tau, a}^\top \right) \xrightarrow{p} \bSigma_a(\bar{\bm{\beta}}).
\end{align}

Similar argument gives that 
\begin{align}
    \frac{1}{t} \sum_{\tau=1}^{t-1} \left(\bX_{\tau, a}Y_{\tau, a}\right) \xrightarrow{p} \bm{\varphi}_a(\bar{\bm{\beta}}).
\end{align}

Now we are ready to show the contradiction. First, from the assumptions, we have $\pi(a \mid \bx, \bbeta)>0$ for all $a$, $\bx$, $\bbeta$. Thus $\forall \bbeta$, 
\begin{align*}
\bSigma_{\bar a}(\bbeta) & = \EE_{\bX_t, A\sim \pi(\cdot|\bX_t, \bbeta)}[1_{\{A = \bar a\}}\bX_t\bX_t^\top]\\
& = \EE_{\bX_t}\left[\EE_{A\sim \pi(\cdot|\bX_t, \bbeta)}[1_{\{A = \bar a\}}\bX_t\bX_t^\top]\Big|\bX_t\right]\\
& = \EE_{\bX_t}\left[\bX_t\bX_t^\top\cdot \pi(\bar a \mid \bX_t, \bbeta)\right]\succ \mathbf{0}.
\end{align*}
The last expression above is true because otherwise, there exists a constant vector $\bv\in\RR^d$ s.t. $\bv^\top \EE_{\bX_t}[\bX_t\bX_t^\top\cdot \pi(\bar a \mid \bX_t, \bbeta)]\bv = 0$, which is equivalent to $\EE (\bv^\top \bX_t)^2\cdot \pi(\bar a \mid \bX_t, \bbeta) = 0$, and implies that $(\bv^\top \bX_t)^2\cdot \pi(\bar a \mid \bX_t, \bbeta) = 0$ almost surely. Because $\pi$ is always positive, it has to be that $\bv^\top \bX_t = 0$ almost surely, which contradicts Assumption \ref{aspt:Ridge_Conv}.

Since the mapping $\bM \mapsto \bM^{-1}$ is continuous on the set of positive definite matrices, by the continuous mapping theorem, we deduce that 
\begin{align}
    \left(\frac{1}{t} \left(\lambda \bI + \sum_{\tau=1}^{t-1} \bX_{\tau, {\bar a}}\bX_{\tau, {\bar a}}^\top \right) \right)^{-1} \xrightarrow{p} \left(\bSigma_{{\bar a}}(\bar{\bm{\beta}})\right)^{-1}.
\end{align}
Further, we have 
\begin{align}
    \hat{{\bbeta}}_{t, \bar a}^{\Ridge} \xrightarrow{p} \left(\bSigma_{\bar a}(\bar{\bm{\beta}})\right)^{-1} \bm{\varphi}_{\bar a}(\bar{\bm{\beta}}).
\end{align}
At the same time, because we also assume that   $\hat{\bm{\beta}}_{t}^{\Ridge} \xrightarrow{p} \bar{\bbeta}$, it has to be true that 
$$
\left(\bSigma_{\bar a}(\bar{\bm{\beta}})\right)^{-1} \bm{\varphi}_{\bar a}(\bar{\bm{\beta}}) = \bar{\bbeta}_{\bar a},\quad a.s.
$$
which contradicts (\ref{eq:necessary-condition-for-Ridge-convergence}).




\section{Proofs of Auxiliary Lemmas}\label{appB}

    \subsection{Proof of Lemma \ref{lem::clip-l2-projection}}\label{apdx::proof-lem::clip-l2-projection}

{
    We first show that $\operatorname{Clip}(\bm{\pi})$ is the $L_2$ projection of $\bm{\pi}$ onto the set $\{ \bm{\pi} \in [0, 1]^{|\mathcal{A}|} \mid \sum_{a} \bm{\pi}_a = 1, \bm{\pi}_a \geq \pi_{\min}\}$.

    The optimization problem is given by
    \begin{align}
        \min_{\bm{\pi}' \in [0, 1]^{|\mathcal{A}|}} \frac{1}{2}\|\bm{\pi} - \bm{\pi}'\|_2^2 \quad \text{s.t.} \quad \sum_{a} \bm{\pi}'_a = 1, \quad \bm{\pi}'_a \geq \pi_{\min}.
    \end{align}
    The Lagrangian is given by
    \begin{align}
        \mathcal{L}(\bm{\pi}', \nu, \mu) = \frac{1}{2}\|\bm{\pi} - \bm{\pi}'\|_2^2 + \nu \left(1 - \sum_{a} \bm{\pi}'_a\right) + \bm{\mu}^\top (\bm{\pi}' - \pi_{\min}).
    \end{align}
    The KKT conditions are given by
    \begin{align}
        \nabla_{\bm{\pi}'} \mathcal{L}(\bm{\pi}', \nu, \bm{\mu}) = \bm{\pi} - \bm{\pi}' - \nu \mathbf{1} - \bm{\mu} = 0, \\
        \nu \left(1 - \sum_{a} \bm{\pi}'_a\right) = 0, \\
        \bm{\mu} \left(\bm{\pi}' - \pi_{\min}\right) = 0, \\
        \bm{\mu} \geq 0, \quad \bm{\pi}' \geq \pi_{\min}, \quad
        \nu \mathbf{1} = 1.
    \end{align}

    From the optimality condition, we have:
    \begin{align}
        \bm{\pi}'_a=\bm{\pi}_a+\nu+\mu_a
    \end{align}

    Complementary slackness indicates that for each $a$ :
    \begin{itemize}
        \item if $\bm{\pi}'_a>\pi_{\min}$, then $\mu_a=0$, thus:
        \begin{align}
            \bm{\pi}'_a=\bm{\pi}_a+\nu
        \end{align}
        \item if $\bm{\pi}'_a=\pi_{\min}$, then:
        \begin{align}
            \pi_{\min}=\bm{\pi}_a+\nu+\mu_a, \quad \mu_a \geq 0 \Rightarrow \bm{\pi}_a+\nu \leq \pi_{\min}
        \end{align}
    \end{itemize}

    Hence, the solution takes the form:
    \begin{align}
        \bm{\pi}'_a=\max \left(\bm{\pi}_a+\nu, \pi_{\min}\right)
    \end{align}
    with the constraint that $\sum_{a=1}^{|\mathcal{A}|} \bm{\pi}'_a=1$.

    By taking the derivative of $\mathcal{L}(\bm{\pi}', \nu, \bm{\mu})$ w.r.t. $\nu$, we have that $\nu$ is the minimum value that satisfies the KKT conditions.

}

{
    
    \subsection{Proof of Lemma \ref{lem::consistency-general}}\label{apdx::proof-lem::consistency-general}
    
    We first prove that for $R_{\btheta}$ stated in Assumption \ref{aspt:boundedness-general}, 
    \begin{equation}\label{eq::uniform-convergence-general}
        \sup_{\|\btheta\|_2\leq R_{\btheta}}\|\bG_T(\btheta) - \EE\bG_T(\btheta)\|_2\xrightarrow{p} 0
    \end{equation}
    as $T\rightarrow\infty$. In fact, we only need to prove 
    \begin{equation}\label{eq::uniform-convergence-by-entry-general}
        \sup_{\|\btheta\|_2\leq R_{\btheta}}\|\bG_T^{(i)}(\btheta) - \EE\bG_T^{(i)}(\btheta)\|_2\xrightarrow{p} 0
    \end{equation}
    for all $i = 1, \ldots, d$. Here $\bG_T^{(i)}(\btheta)$ denotes the $i$-th entry of $\bG_T(\btheta)$.
    
    Fix any $\epsilon>0$. Let $\bTheta_\epsilon = \{\btheta_j: j = 1, \ldots, N_\epsilon\}$ be an $\epsilon$-net of the set $\bTheta := \overline{\cB(\mathbf{0}, R_{\btheta})}$ with finite cardinality $N_\epsilon$. This means that $\forall \btheta\in\bTheta$, $\exists j\in[N_\epsilon]$ such that $\|\btheta - \btheta_j\|_2\leq \epsilon$. Then from Assumption \ref{aspt:smoothness-general}, we have 
    \begin{align*}
    |\bg^{(i)}(\bX_t, Y_t(a); \btheta) - \bg^{(i)}(\bX_t, Y_t(a); \btheta_j)| 
    &\leq \|\bg(\bX_t, Y_t(a); \btheta) - \bg(\bX_t, Y_t(a); \btheta_j)\|_2\\
    &\leq \left\|\int_{0}^1\nabla \bg(\bX_t, Y_t(a);\btheta_j + u(\btheta - \btheta_j))\mathrm{d}u\cdot (\btheta - \btheta_j)\right\|_2\\
    &\leq \phi(\bX_t, Y_t(a))\cdot \|\btheta - \btheta_j\|_2\\
    &\leq \epsilon\phi(\bX_t, Y_t(a)).
    \end{align*}
    Here $\bg^{(i)}(\bX_t, Y_t(a); \btheta)$ denotes the $i$-th entry of $\bg(\bX_t, Y_t(a); \btheta)$, for any $\btheta$. Thus, 
    \begin{equation}\label{eq::gi-nested-bounds}
        l_j(\bX_t, Y_t(a))\leq \bg^{(i)}(\bX_t, Y_t(a); \btheta)\leq u_j(\bX_t, Y_t(a)),
    \end{equation}
    where we define $l_j(\bX_t, Y_t(a)) = \bg^{(i)}(\bX_t, Y_t(a); \btheta_j) - \epsilon\phi(\bX_t, Y_t(a))$, $u_j(\bX_t, Y_t(a)) = \bg^{(i)}(\bX_t, Y_t(a); \btheta_j) + \epsilon\phi(\bX_t, Y_t(a))$.
    
    Now, notice that $\bG_T(\btheta) = \frac1T\sum_{t=1}^T\bZ_t(\btheta)$, where $\bZ_t(\btheta) = \frac{1}{\pi_t(A_t)}1_{\{A_t = a\}}\bg(\bX_t, Y_t; \btheta)$. We have 
    \begin{align*}
        \EE[\bZ_t(\btheta)|\cH_{t-1}]
        & = \EE_{\bX_t}\big[\EE[\bZ_t(\btheta)|\cH_{t-1}, \bX_t]\big|\cH_{t-1}\big]\\
        & = \EE_{\bX_t}\bigg[\EE_{A_t\sim \pi_t}\Big[\frac{1}{\pi_t(A_t)}1_{\{A_t = a\}}\Big|\cH_{t-1}, \bX_t\Big]\cdot \EE_{Y_t(a)}[\bg(\bX_t, Y_t(a); \btheta)|\cH_{t-1}, \bX_t]\bigg|\cH_{t-1}\bigg]\\
        & = \EE_{\bX_t}\big[\EE_{Y_t(a)}[\bg(\bX_t, Y_t(a); \btheta)|\cH_{t-1}, \bX_t]\big|\cH_{t-1}\big]\\
        & = \EE\bg(\bX_t, Y_t(a); \btheta).
    \end{align*}
    Here in the second equation, we use Assumption \ref{aspt:unconfoundedness}. This implies that 
    \begin{equation}\label{eq::GT-expectation}
        \EE\bG_T(\btheta) =  \EE\bg(\bX_t, Y_t(a); \btheta) = \EE[\bZ_t(\btheta)|\cH_{t-1}],
    \end{equation}
    $\forall t\in[T]$. Combining (\ref{eq::gi-nested-bounds}), we deduce that 
    \begin{align}
    \sup_{\btheta\in\bTheta}\{\bG_T^{(i)}(\btheta) - \EE\bG_T^{(i)}(\btheta)\}
    & = \sup_{\btheta\in\bTheta}\frac1T\sum_{t=1}^T
    \bigg\{\frac{1}{\pi_t(A_t)}1_{\{A_t = a\}}\bg^{(i)}(\bX_t, Y_t; \btheta)\nonumber \\
    &\quad\quad\quad\quad\quad\quad - \EE\bigg[\frac{1}{\pi_t(A_t)}1_{\{A_t = a\}}\bg^{(i)}(\bX_t, Y_t; \btheta)\bigg|\cH_{t-1}\bigg]\bigg\}\nonumber\\
    & \leq \max_{j\in[N_\epsilon]} \frac1T\sum_{t=1}^T
    \bigg\{\frac{1}{\pi_t(A_t)}1_{\{A_t = a\}}u_j(\bX_t, Y_t(a))\nonumber\\
    &\quad\quad\quad\quad\quad\quad\quad - \EE\bigg[\frac{1}{\pi_t(A_t)}1_{\{A_t = a\}}l_j(\bX_t, Y_t(a))\bigg|\cH_{t-1}\bigg]\bigg\}\nonumber\\
    & \leq \Delta_1 + \Delta_2,\label{eq::maxima-mtg-decomposition}
    \end{align}
    where 
    $$
    \Delta_1 := \max_{j\in[N_\epsilon]} \frac1T\sum_{t=1}^T
    \left\{\frac{1}{\pi_t(A_t)}1_{\{A_t = a\}}u_j(\bX_t, Y_t(a)) - \EE\bigg[\frac{1}{\pi_t(A_t)}1_{\{A_t = a\}}u_j(\bX_t, Y_t(a))\bigg|\cH_{t-1}\bigg]\right\},
    $$
    $$
    \Delta_2:= \max_{j\in[N_\epsilon]} \frac1T\sum_{t=1}^T
    \EE\bigg[\frac{1}{\pi_t(A_t)}1_{\{A_t = a\}}\big(u_j(\bX_t, Y_t(a))-l_j(\bX_t, Y_t(a))\big)\bigg|\cH_{t-1}\bigg].
    $$
    We first analyze $\Delta_1$. Define $U_{t, j} = \frac{1_{\{A_t = a\}}}{\pi_t(A_t)}u_j(\bX_t, Y_t(a))$, and $M_{t, j} = U_{t, j} - \EE[U_{t, j}|\cH_{t-1}]$. Then we have 
    \begin{equation}\label{eq::Delta1_bound}
        \Delta_1\leq \sum_{j\in[N_\epsilon]}\left| \frac1T\sum_{t=1}^T
    M_{t, j}\right|.
    \end{equation}
    For any $\epsilon'>0$,
    \begin{align}
    & \PP\left(\bigg|\frac1T\sum_{t=1}^T
    M_{t, j}\bigg|>\epsilon'\right)
    \leq \frac1{T^2\epsilon'^2}\EE\left(\sum_{t=1}^T
    M_{t, j}\right)^2
    = \frac1{T^2\epsilon'^2}\sum_{t=1}^T\EE M_{t, j}^2\nonumber\\
    \leq & \frac1{T^2\epsilon'^2}\sum_{t=1}^T\EE U_{t, j}^2
     = \frac1{T^2\epsilon'^2}\sum_{t=1}^T\EE_{\cH_{t-1}, \bX_t} \EE[U_{t, j}^2|\cH_{t-1}, \bX_t] \nonumber\\
    \leq & \frac1{T^2\epsilon'^2}\sum_{t=1}^T\EE_{\cH_{t-1}, \bX_t}
    \left[
    \EE_{A_t\sim \pi_t}\bigg[\frac{1_{\{A_t = a\}}}{\pi_t(A_t)^2}\bigg|\cH_{t-1}, \bX_t\bigg]
    \cdot
    \EE_{Y_t(a)}\bigg[u_j(\bX_t, Y_t(a))^2\bigg|\cH_{t-1}, \bX_t\bigg]
    \right]\nonumber\\
    = & \frac1{T^2\epsilon'^2}\sum_{t=1}^T\EE_{\cH_{t-1}, \bX_t}
    \left[
   \frac1{\pi_t(a)}\cdot 
    \EE_{Y_t(a)}\bigg[u_j(\bX_t, Y_t(a))^2\bigg|\cH_{t-1}, \bX_t\bigg]
    \right]\nonumber\\
    \leq& \frac1{T^2\epsilon'^2}\sum_{t=1}^T\pi_{\min,t}^{-1}
    \cdot \EE[u_j^2(\bX_t, Y_t(a))]\nonumber\\
    \leq & \frac{\sum_{t=1}^T \pi_{\min,t}^{-1}}{\epsilon'^2T^2}\cdot \EE[\bg^{(i)}(\bX_1, Y_1(a);\btheta_j) + \epsilon \bphi(\bX_1, Y_1(a))]^2\nonumber\\
    \leq & \frac{\sum_{t=1}^T \pi_{\min,t}^{-1}}{\epsilon'^2T^2}\cdot \big\{2\EE[\bg^{(i)}(\bX_1, Y_1(a);\btheta_j)^2] + 2\epsilon^2\EE[\bphi(\bX_1, Y_1(a))]^2\big\}\nonumber\\
    \leq & \frac{\sum_{t=1}^T \pi_{\min,t}^{-1}}{\epsilon'^2T^2} (2M_3 + 2\epsilon^2 M_2')\rightarrow 0.
    \label{eq::single-theta-u-convergence}
    \end{align}
    Here $M_3:= \sup_{\|\btheta\|_2\leq R_{\btheta}}\EE\|\bg(\bX_t, Y_t(a);\btheta)\|_2^2$. Above, the first inequality is due to Chebyshev's inequality. The first equality is due to the following fact: for $t_1<t_2$,
    \begin{align*}
       & \EE [M_{t_1, j}M_{t_2, j}]\\
       = & \EE\big[ \EE[(U_{t_1, j} - \EE[U_{t_1, j}|\cH_{t_1-1}])(U_{t_2, j} - \EE[U_{t_2, j}|\cH_{t_2-1}])\big|\cH_{t_2-1}]\big]\\
       = & \EE\big[ (U_{t_1, j} - \EE[U_{t_1, j}|\cH_{t_1-1}])\cdot \EE[U_{t_2, j} - \EE[U_{t_2, j}|\cH_{t_2-1}]\big|\cH_{t_2-1}]\big]\\
       = & \EE\big[ (U_{t_1, j} - \EE[U_{t_1, j}|\cH_{t_1-1}])\cdot 0\big] = 0.
    \end{align*}
    The second to last inequality is due to the fact that $(a+b)^2\leq 2(a^2+b^2)$ for any $a, b\in\RR$. The final inequality uses Assumption \ref{aspt:boundedness-general} and Assumption \ref{aspt:smoothness-general}, which implies $M_2'<\infty$ and $M_3<\infty$.
    
    Combining (\ref{eq::Delta1_bound}) and (\ref{eq::single-theta-u-convergence}), we obtain that $\Delta_1\xrightarrow{p}0$.
    
    Next, we analyze $\Delta_2$. Note that $\forall j$, 
    \begin{align*}
    &\frac1T\sum_{t=1}^T
    \EE\bigg[\frac{1}{\pi_t(A_t)}1_{\{A_t = a\}}\big(u_j(\bX_t, Y_t(a))-l_j(\bX_t, Y_t(a))\big)\bigg|\cH_{t-1}\bigg]\\
    =& \EE_{\bX_t}
    \left[
    \EE_{A_t}\bigg[\frac{1_{\{A_t = a\}}}{\pi_t(A_t)}\bigg|\cH_{t-1}, \bX_t\bigg]
    \cdot
    \EE_{Y_t(a)}
    \big[
    u_j(\bX_t, Y_t(a)) - l_j(\bX_t, Y_t(a))|\cH_{t-1}, \bX_t
    \big]
    \bigg|\cH_{t-1}\right]\\
    = & \EE_{\bX_t}
    \left[
    \EE_{Y_t(a)}
    \big[
    u_j(\bX_t, Y_t(a)) - l_j(\bX_t, Y_t(a))|\cH_{t-1}, \bX_t
    \big]
    \Big|\cH_{t-1}\right]\\
    = & \EE[u_j(\bX_t, Y_t(a)) - l_j(\bX_t, Y_t(a))].
    \end{align*}
    Thus,
    \begin{align*}
    \Delta_2\leq & \max_j\frac1T\sum_{t=1}^T\EE[u_j(\bX_t, Y_t(a)) - l_j(\bX_t, Y_t(a))]\\
    = & 2\epsilon \EE[\phi(\bX_1, Y_1(a))]\\
    \leq & 2\epsilon \sqrt{\EE[\phi(\bX_1, Y_1(a))^2]} \leq 2\epsilon \sqrt{M_2'}.
    \end{align*}
    
    Plugging in the analysis of $\Delta_1$ and $\Delta_2$ into (\ref{eq::maxima-mtg-decomposition}), we obtain that $\forall \epsilon>0$, 
    $$
    \PP\left(\sup_{\btheta\in\bTheta}\{\bG_T^{(i)}(\btheta) - \EE\bG_T^{(i)}(\btheta)\}>2\epsilon \sqrt{M_2'}\right)\rightarrow 0.
    $$
    This implies that $\forall \epsilon>0$, 
    $$
    \PP\left(\sup_{\btheta\in\bTheta}\{\bG_T^{(i)}(\btheta) - \EE\bG_T^{(i)}(\btheta)\}>\epsilon\right)\rightarrow 0.
    $$
    Similarly, we have 
    $$
    \PP\left(\sup_{\btheta\in\bTheta}\{-\bG_T^{(i)}(\btheta) + \EE\bG_T^{(i)}(\btheta)\}>\epsilon\right)\rightarrow 0.
    $$
    Therefore, we have proved (\ref{eq::uniform-convergence-by-entry-general}). Further, (\ref{eq::uniform-convergence-general}) is proved.
    
    Now we return to prove the main results. Define 
    $$
    \hat\btheta_a^{(T)} = \argmin_{\|\btheta\|_2\leq R_{\btheta}}\|\bG_T(\btheta)\|_2,
    $$
    and 
    $$\tilde\btheta_a^{(T)} = \btheta_a^* - [\nabla\bar \bG(\btheta_a^*)]^{-1}\bG_T(\btheta_a^*).$$ 
    Here $\bar\bG(\btheta) := \EE \bG_T(\btheta) = \EE \bg(\bX_t, Y_t(a);\btheta)$ due to (\ref{eq::GT-expectation}). From Assumption \ref{aspt:smoothness-general}, $\nabla\bar \bG(\btheta_a^*)$ is invertible. Combining Lemma \ref{lem::asymptotic-normality-true-G-general}, we have 
    \begin{equation}\label{eq::exist-z-estimator-consistency}
        \tilde\btheta_a^{(T)} - \btheta_a^* = O_p(1/b_{a, T}).
    \end{equation}
    In addition, from Taylor expansion, we have 
    \begin{align}
    \bG_T(\tilde\btheta_a^{(T)})
    & = \bG_T(\btheta_a^*) + \nabla\bG_T(\btheta_a^*)(\tilde\btheta_a^{(T)} - \btheta_a^*) + o_p(\|\tilde\btheta_a^{(T)} - \btheta_a^*\|_2)\nonumber\\
    & = \bG_T(\btheta_a^*) - \nabla\bG_T(\btheta_a^*)[\nabla\bar \bG(\btheta_a^*)]^{-1}\bG_T(\btheta_a^*) + o_p(1/b_{a, T})\nonumber\\
    & = \bG_T(\btheta_a^*) - (1 + o_p(1))\bG_T(\btheta_a^*) + o_p(1/b_{a, T})\nonumber\\
    & = o_p(1/b_{a, T}).\label{eq::exist-z-estimator}
    \end{align}
    Here the second equality is from the definition of $\tilde\btheta_a^{(T)}$ and due to (\ref{eq::exist-z-estimator-consistency}). The third equality is obtained from the following two facts: From Lemma \ref{lem::convergence-true-derivative-general}, 
    $\nabla\bG_T(\btheta_a^*)\xrightarrow{p}\EE\nabla\bg(\bX_t, Y_t(a);\btheta_a^*)$; From (\ref{eq::GT-expectation}) and Assumption \ref{aspt:smoothness-general}, $\nabla \bar \bG(\btheta_a^*)  = \frac{\partial}{\partial \btheta}\EE\bG_T(\btheta)|_{\btheta =\btheta_a^*}= \frac{\partial}{\partial \btheta}\EE \bg(\bX_t, Y_t(a);\btheta)|_{\btheta =\btheta_a^*} = \EE \frac{\partial}{\partial \btheta}\bg(\bX_t, Y_t(a);\btheta)|_{\btheta =\btheta_a^*}$ nonsingular. Here the exchangeability between expectation and differentiation can be obtained by standard arguments using dominated convergence theorem, see e.g. Theorem 16.8 in \citep{billingsley1995probability}. The last equality is obtained from Lemma \ref{lem::asymptotic-normality-true-G-general}.
    
    Combining (\ref{eq::exist-z-estimator-consistency}) (which implies $\PP(\|\tilde\btheta_a^{(T)}\|_2>R_{\btheta})\rightarrow 0$) and (\ref{eq::exist-z-estimator}), we deduce that 
    \begin{align*}
        \|\bG_T(\hat\btheta_a^{(T)})\|_2 &= \|\bG_T(\hat\btheta_a^{(T)})\|_21_{\{\|\tilde\btheta_a^{(T)}\|_2\leq R_{\btheta}\}} + \|\bG_T(\hat\btheta_a^{(T)})\|_21_{\{\|\tilde\btheta_a^{(T)}\|_2> R_{\btheta}\}}\\
        & \leq \|\bG_T(\tilde\btheta_a^{(T)})\|_2 + o_p(1/b_{a, T}) = o_p(1/b_{a, T}).
    \end{align*}
    Thus we have proved the first part of Lemma \ref{lem::consistency-general}. 
    
    Next, for any sequence $\hat\btheta_a^{(T)}$ such that $\hat\btheta_a^{(T)}\leq R_{\btheta}$ and (\ref{eq::estimating-equation-general}) holds, we proceed to prove $\hat\btheta_a^{(T)}\xrightarrow{p}\btheta_a^*$. According to Assumption \ref{aspt:identifiability-general}, for any $\epsilon>0$, we have 
    $$
    \delta_\epsilon:=\inf_{\|\btheta - \btheta_a^*\|_2>\epsilon}\|\EE\bg(\bX_t, Y_t(a); \btheta)\|_2>0.
    $$
    From (\ref{eq::GT-expectation}), we deduce that 
    $$
    \inf_{\|\btheta - \btheta_a^*\|_2>\epsilon}\|\EE\bG_T( \btheta)\|_2 = \inf_{\|\btheta - \btheta_a^*\|_2>\epsilon}\|\EE\bg(\bX_t, Y_t(a);\btheta)\|_2\geq \delta_{\epsilon}.
    $$
    Combine (\ref{eq::uniform-convergence-general}), we have 
    \begin{align*}
    \inf_{\btheta\in\bTheta, \|\btheta - \btheta_a^*\|_2>\epsilon} \|\bG_T(\btheta)\|_2\geq \inf_{\|\btheta - \btheta_a^*\|_2>\epsilon}\|\EE\bG_T( \btheta)\|_2 - \sup_{\btheta\in\bTheta}\|\bG_T(\btheta) - \EE\bG_T(\btheta)\|_2\geq \delta_{\epsilon} - o_p(1).
    \end{align*}
    Here $\bTheta = \overline{\cB(\mathbf{0}, R_{\btheta})}$. This implies 
    \begin{equation}\label{eq::consistency1}
    \lim_{T\rightarrow\infty}\PP\left(\inf_{\btheta\in\bTheta, \|\btheta - \btheta_a^*\|_2>\epsilon} \|\bG_T(\btheta)\|_2\leq \frac12\delta_{\epsilon}\right)=0.
    \end{equation}
    At the same time, from the property of $\hat\btheta_a^{(T)}$, we have for any $\epsilon'>0$,
    $$
    \lim_{T\rightarrow\infty}\PP\left(\|\bG_{T}(\hat\btheta_a^{(T)})\|_2>\epsilon'/b_{a, T}\right)=0.
    $$
    Thus, 
    \begin{equation}\label{eq::consistency2}
    \lim_{T\rightarrow\infty}\PP\left(\|\bG_{T}(\hat\btheta_a^{(T)})\|_2>\frac12\delta_{\epsilon}\right)= 0.
    \end{equation}
    We combine (\ref{eq::consistency1}) and (\ref{eq::consistency2}) and get
    \begin{align*}
    \PP\big(\|\hat\btheta_a^{(T)} - \btheta_a^*\|_2>\epsilon\big)
    & =  \PP\left(\|\hat\btheta_a^{(T)} -\btheta_a^*\|_2>\epsilon, \|\bG_{T}(\hat\btheta_a^{(T)})\|_2\leq \frac12\delta_{\epsilon}\right)\\
    &\quad + \PP\left(\|\hat\btheta_a^{(T)} - \btheta_a^*\|_2>\epsilon, \|\bG_{T}(\hat\btheta_a^{(T)})\|_2>\frac12\delta_{\epsilon}\right)\\
    &\leq \PP\left(\inf_{\btheta\in\bTheta, \|\btheta - \btheta_a^*\|_2>\epsilon} \|\bG_T(\btheta)\|_2\leq \frac12\delta_{\epsilon}\right) + \PP\left(\|\bG_{T}(\hat\btheta_a^{(T)})\|_2>\frac12\delta_{\epsilon}\right)\\
    &\rightarrow 0
    \end{align*}
    as $T\rightarrow \infty$. As the above holds for any $\epsilon>0$, the consistency of $\hat\btheta_a^{(T)}$ is proved.
    }

    {
    \subsection{Proof of Lemma \ref{lem::convergence-true-derivative-general}}\label{apdx::proof-lem::convergence-true-derivative-general}
    
    Recall that $\nabla\bG_T(\btheta_a^*) = \frac1T\sum_{t=1}^T\bV_t$, where\\ $\bV_t:=\frac{1}{\pi_t(A_t)}1_{\{A_t = a\}}\nabla\bg(\bX_t, Y_t(a);\btheta_a^*)$. We have for any nonrandom vectors $\bc, \bc'\in\RR^d$, 
    \begin{align*}
    \EE[c^\top\bV_t c'|\cH_{t-1}]
    & = \EE_{\bX_t}\left[\EE_{A_t\sim \pi_t, Y_t(a)}[\bc^\top\bV_t\bc'|\cH_{t-1}, \bX_t]\bigg|\cH_{t-1}\right]\\
    & = \EE_{\bX_t}\bigg[\EE_{A_t\sim \pi_t}\bigg[\frac{1}{\pi_t(A_t)}1_{\{A_t = a\}}\bigg|\cH_{t-1}, \bX_t\bigg]\cdot\\
    &\quad \quad\quad \quad\EE_{Y_t(a)}[\bc^\top\nabla\bg(\bX_t, Y_t(a);\btheta_a^*)\bc'|\cH_{t-1}, \bX_t]\bigg|\cH_{t-1}\bigg]\\
    & = \EE_{\bX_t}\left[1\cdot \EE_{Y_t(a)}[\bc^\top\nabla\bg(\bX_t, Y_t(a);\btheta_a^*)\bc'|\cH_{t-1}, \bX_t]\bigg|\cH_{t-1}\right]\\
    & = \bc^\top\EE[\nabla\bg(\bX_t, Y_t(a);\btheta_a^*)]\bc'.
    \end{align*}
    Here the second inequality uses Assumption \ref{aspt:unconfoundedness}. From the above we deduce that $\forall \delta>0$,
    \begin{align}
    &\PP\left(|\bc^\top[\nabla\bG_T(\btheta_a^*) - \EE\nabla\bg(\bX_t, Y_t(a);\btheta_a^*)]\bc'|>\delta\right)
    =  \PP\left(\bigg|\frac1T\sum_{t=1}^T[\bc^\top\bV_t\bc' - \EE[\bc^\top\bV_t\bc'|\cH_{t-1}]]\bigg|>\delta\right)\nonumber\\
    \leq & \frac1{\delta^2T^2}\EE\left(\sum_{t=1}^T[\bc^\top\bV_t\bc' - \EE[\bc^\top\bV_t\bc'|\cH_{t-1}]]\right)^2
     = \frac1{\delta^2T^2}\sum_{t=1}^T\EE\left(\bc^\top\bV_t\bc' - \EE[\bc^\top\bV_t\bc'|\cH_{t-1}]\right)^2\nonumber\\
     \leq & \frac1{\delta^2T^2}\sum_{t=1}^T\EE\left(\bc^\top\bV_t\bc' \right)^2
     =\frac1{\delta^2T^2}\sum_{t=1}^T\EE\left[\frac{1_{\{A_t = a\}}}{\pi_t(A_t)^2}\cdot \big(\bc^\top \nabla \bg(\bX_t, Y_t(a);\btheta_a^*)\bc\big)^2\right]\nonumber\\
     = & \frac1{\delta^2T^2}\sum_{t=1}^T \EE_{\cH_{t-1}, \bX_t}
     \left[
     \EE_{A_t}\bigg[\frac{1_{\{A_t = a\}}}{\pi_t(A_t)^2}\bigg|\cH_{t-1}, \bX_t\bigg]
     \cdot 
     \EE_{Y_t(a)}
     \Big[
     \big(\bc^\top \nabla \bg(\bX_t, Y_t(a);\btheta_a^*)\bc\big)^2
     \Big|\cH_{t-1}, \bX_t\Big]
     \right]\nonumber\\
     = & 
     \frac1{\delta^2T^2}\sum_{t=1}^T \EE_{\cH_{t-1}, \bX_t}
     \left[
     \frac{1}{\pi_t(a)}
     \cdot 
     \EE_{Y_t(a)}
     \Big[
     \big(\bc^\top \nabla \bg(\bX_t, Y_t(a);\btheta_a^*)\bc\big)^2
     \Big|\cH_{t-1}, \bX_t\Big]
     \right]\nonumber\\
     \leq & \frac1{\delta^2T^2}\sum_{t=1}^T \pi_{\min,t}^{-1}\cdot \EE\big(\bc^\top \nabla \bg(\bX_1, Y_1(a);\btheta_a^*)\bc\big)^2\to 0.
     \label{eq::convergence-gradient-general-cc'}
    \end{align}
    Here the first inequality is because of Chebyshev's inequality. The second equality is due to the following fact: Let $v_{t} = \bc^\top\bV_t\bc'$. Then for $t_1<t_2$,
    \begin{align*}
       & \EE (v_{t_1} - \EE[v_{t_1}|\cH_{t_1-1}])(v_{t_2} - \EE[v_{t_2}|\cH_{t_2-1}]) \\
       = & \EE\big[ \EE[(v_{t_1} - \EE[v_{t_1}|\cH_{t_1-1}])(v_{t_2} - \EE[v_{t_2}|\cH_{t_2-1}])\big|\cH_{t_2-1}]\big]\\
       = & \EE\big[(v_{t_1} - \EE[v_{t_1}|\cH_{t_1-1}])\cdot \EE[v_{t_2} - \EE[v_{t_2}|\cH_{t_2-1}]\big|\cH_{t_2-1}]\big]\\
       = & \EE\big[ (v_{t_1} - \EE[v_{t_1}|\cH_{t_1-1}])\cdot 0\big] = 0.
    \end{align*}
    The last convergence uses Assumption \ref{aspt:smoothness-general}. 
    
    Finally, because (\ref{eq::convergence-gradient-general-cc'}) holds for any $\bc, \bc'$, we conclude our proof.
    }

    {
    \subsection{Proof of Lemma \ref{lem::2nd-derivative-boundness-general}}\label{apdx::proof-lem::2nd-derivative-boundness-general}
    
    Note that $\nabla^2\bG_T(\btheta) = \frac1T\sum_{t=1}^T\frac{1_{\{A_t = a\}}}{\pi_t(A_t)}\nabla^2\bg(\bX_t, Y_t(a); \btheta)$. Thus,
\begin{align*}
\sup_{\|\btheta - \btheta_a^*\|_2\leq \epsilon_0}\|\nabla^2\bG_T(\btheta)\|_1
& = \sup_{\|\btheta - \btheta_a^*\|_2\leq \epsilon_0}
\left\|
\frac1T\sum_{t=1}^T\frac{1_{\{A_t = a\}}}{\pi_t(A_t)}\nabla^2\bg(\bX_t, Y_t(a); \btheta)
\right\|_1\\
&\leq \sup_{\|\btheta - \btheta_a^*\|_2\leq \epsilon_0}\frac1T\sum_{t=1}^T\frac{1_{\{A_t = a\}}}{\pi_t(A_t)}
\|\nabla^2\bg(\bX_t, Y_t(a); \btheta)\|_1\\
&\leq \sup_{\|\btheta - \btheta_a^*\|_2\leq \epsilon_0}\frac1T\sum_{t=1}^T\frac{1_{\{A_t = a\}}}{\pi_t(A_t)} d^2\sup_{i\in[d]}\|\nabla^2\bg^{(i)}(\bX_t, Y_t(a);\btheta)\|_2\\
&\leq \frac{d^2}{T}\sum_{t=1}^T\frac{1_{\{A_t = a\}}}{\pi_t(A_t)} \Phi(X_t, Y_t(a)).
\end{align*}
Here we have used Assumption \ref{aspt:smoothness-general}.

Note that 
\begin{align*}
\EE\left[
\frac{1_{\{A_t = a\}}}{\pi_t(A_t)} \Phi(X_t, Y_t(a))
\right]
&=
\EE_{\cH_{t-1},\bX_t}
\left[
\EE_{A_t}
\bigg[\frac{1_{\{A_t = a\}}}{\pi_t(A_t)}\bigg|\cH_{t-1},\bX_t\bigg]
\cdot 
\EE_{Y(a)}
\big[\Phi(X_t, Y_t(a))\big|\cH_{t-1},\bX_t\big]
\right]\\
& = \EE\Phi(\bX_1, Y_1(a)),
\end{align*}
thus
$$
\EE\left[\sup_{\|\btheta - \btheta_a^*\|_2\leq \epsilon_0}\|\nabla^2\bG_T(\btheta)\|_1\right]
\leq d^2\EE\Phi(\bX_1, Y_1(a))<\infty.
$$
And the conclusion follows.
    }
    {
    \subsection{Proof of Lemma \ref{lem::asymptotic-normality-true-G-general}}\label{apdx::proof-lem::asymptotic-normality-true-G-general}
    
    We have $\bG_T(\btheta_a^*) = \frac1T\sum_{t=1}^T\bZ_t$, where we define \\
    $\bZ_t:= \frac{1}{\pi_t(A_t)}1_{\{A_t = a\}}\bg(\bX_t, Y_t(a);\btheta_a^*)$. From the Cramer-Wold theorem, in order to show the desired asymptotic normality, it suffices to show that for any $\bc\in\RR^d$, $\bc^\top\cdot b_{a, T}\bG_T(\btheta_a^*) = \bc^\top\cdot\frac{b_{a, T}}{T}\sum_{t=1}^T\bZ_t\xrightarrow{d} \cN(\mathbf{0}, \bc^\top\bar\bI_a\bc)$.
    
    From  \citep{dvoretzky1972asymptotic}, Theorem 2.2, the above asymptotic result can be obtained by ensuring
    \begin{gather}
       \EE [\bZ_t|\cH_{t-1}] = \mathbf{0}\quad  \forall t\in[T],\label{eq::cond-exp-general}\\
       \frac1{S_{a, T}^2}\sum_{t\in[T]}\mathrm{Var}(\bc^\top\bZ_t|\cH_{t-1}) \xrightarrow{p} \bc^\top \bar\bI_a \bc,\label{eq::cond-var-general}\\
       \frac1{S_{a, T}^2}\sum_{t\in[T]}\EE\left[(\bc^\top \bZ_t)^21_{\{|\bc^\top\bZ_t|>{S_{a, T}}\delta\}}\Big|\cH_{t-1}\right]\xrightarrow{p} 0\quad \forall \delta>0.\label{eq::cond-lindeberg-general}
    \end{gather}
    Below we check these facts one by one.
    
    \textbf{Check (\ref{eq::cond-exp-general})}: We have
    \begin{align}
    \EE [\bZ_t|\cH_{t-1}] 
    &= \EE_{\bX_t}\left[\EE_{A_t\sim \pi_t, Y_t(a)}[\bZ_t|\cH_{t-1}, \bX_t]|\cH_{t-1}\right]\nonumber\\
    & = \EE_{\bX_t}\bigg[\EE_{A_t\sim \pi_t}\bigg[\frac{1}{\pi_t(A_t)}1_{\{A_t = a\}}\Big|\cH_{t-1}, \bX_t\bigg]\nonumber\\
    &\quad\quad\quad\quad\cdot \EE_{Y_t(a)}[\bg(\bX_t, Y_t(a);\btheta_a^*)|\cH_{t-1}, \bX_t]\bigg|\cH_{t-1}\bigg]\nonumber\\
    & = \EE_{\bX_t}\left[1\cdot \EE_{Y_t(a)}[\bg(\bX_t, Y_t(a);\btheta_a^*)|\cH_{t-1}, \bX_t]\bigg|\cH_{t-1}\right]\nonumber\\
    & = \EE[\bg(\bX_t, Y_t(a);\btheta_a^*)] = \mathbf{0}.\nonumber
    \end{align}
    Here, the second equality is because of Assumption \ref{aspt:unconfoundedness}. 
    
    \textbf{Check (\ref{eq::cond-var-general})}: Based on (\ref{eq::cond-exp-general}), 
    \begin{align}
    \frac1{S_{a, T}^2}\sum_{t\in[T]}\mathrm{Var}(\bc^\top\bZ_t|\cH_{t-1}) 
    &= \frac1{S_{a, T}^2}\sum_{t\in[T]}\EE [\bc^\top \bZ_t\bZ_t^\top \bc|\cH_{t-1}]\nonumber\\
    & = \frac1{S_{a, T}^2}\sum_{t\in[T]}\bc^\top \EE [ \bZ_t\bZ_t^\top|\cH_{t-1}]\bc.\label{eq::cond-var-1-general}
    \end{align}
    Define $\bV_T = \frac1{S_{a, T}^2}\sum_{t\in[T]}\EE[\bZ_t\bZ_t^\top|\cH_{t-1}]$, $\bM(\bX_t) = \EE[\bg(\bX_t, Y_t(a);\btheta_a^*)\bg(\bX_t, Y_t(a);\btheta_a^*)^\top | \bX_t]$. Then  
    \begin{align}
    \bc^\top\EE [ \bZ_t\bZ_t^\top|\cH_{t-1}] \bc
    & = \bc^\top\EE_{\bX_t}\big[\EE_{A_t\sim \pi_t, Y_t(a)}[\bZ_t\bZ_t^\top|\cH_{t-1}, \bX_t]\big|\cH_{t-1}\big]\bc\nonumber\\
    & = \bc^\top\EE_{\bX_t}\bigg[\EE_{A_t\sim \pi_t}\bigg[\frac{1}{\pi_t^2(A_t)}1_{\{A_t = a\}}\Big|\cH_{t-1}, \bX_t\bigg]\cdot\nonumber\\
    & \quad\quad\EE_{Y_t(a)}[\bg(\bX_t, Y_t(a);\btheta_a^*)\bg(\bX_t, Y_t(a);\btheta_a^*)^\top|\cH_{t-1}, \bX_t]\bigg|\cH_{t-1}\bigg]\bc\nonumber\\
    & = \bc^\top\EE_{\bX_t}\left[\frac{1}{\pi_t(a)}\cdot \bM(\bX_t)\bigg|\cH_{t-1}\right]\bc\label{eq::cond-var-2-general}
    \end{align}
    We deduce that 
    \begin{equation}
        \bc^\top\bV_t \bc = \bc^\top\bA_t \bc + \bc^\top\bB_t \bc,
    \end{equation}
    where
    \begin{align*}
    \bA_T & = \frac1{S_{a, T}^2}\sum_{t\in[T]}\frac{1}{r_{a, t}}\EE\left[\rho(a, \bX_t)\bM(\bX_t)\right] = \bar\bI_a,\\
    \bB_T & = \frac1{S_{a, T}^2}\sum_{t\in[T]}\frac{1}{r_{a, t}}\EE
    \left[
    \bigg(\frac{r_{a, t}}{\pi_t(a)} - \rho(a, \bX_t)\bigg)\bM(\bX_t) \bigg| \cH_{t-1}
    \right].
    \end{align*}
    Define $D_{t} = \bc^\top \bigg(\frac{r_{a, t}}{\pi_t(a)} - \rho(a, \bX_t)\bigg)\bM(\bX_t)\bc$, from {scaled inverse-propensity convergence} as well as Assumption \ref{aspt:boundedness-general}, we deduce that
    $|D_t|\leq M_2\left|\frac{r_{a, t}}{\pi_t(a)} - \rho(a, \bX_t)\right|\cdot \|\bc\|_2^2$ and $\EE|\bD_t|\to 0$. At the same time, we have
    $$
    \EE \big| \EE[D_{t}|\cH_{t-1}]\big|\leq 
    \EE \EE\big[|D_{t}|\big|\cH_{t-1}\big]=\EE |D_{t}|,
    $$
    Thus, we deduce that $\EE[D_t|\cH_{t-1}]\xrightarrow{L_1} 0$.
    Notice that
    $$
    \bc^\top\bB_T\bc = \frac1{S_{a, T}^2}\sum_{t=1}^T\frac1{r_{a, t}}\EE[D_t|\cH_{t-1}] = \sum_{t=1}^T \frac{1/r_{a, t}}{\sum_{t'=1}^T 1/r_{a, t'}}\cdot \EE[D_t|\cH_{t-1}].
    $$
    Using Toeplitz Lemma and the fact that $\lim_{T\to \infty} S_{a, T} = \infty$, we deduce that $\bc^\top\bB_T\bc\xrightarrow{L_1} 0$. Combine the above arguments, and notice that convergence in $L_1$ imply convergence in probability, we conclude (\ref{eq::cond-var-general}).
    
    \textbf{Check (\ref{eq::cond-lindeberg-general})}: We have 
\begin{align*}
&\frac1{S_{a, T}^2}\sum_{t\in[T]}\EE\left[(\bc^\top \bZ_t)^21_{\{|\bc^\top\bZ_t|>{S_{a, T}}\delta\}}\Big|\cH_{t-1}\right]\leq \frac1{S_{a, T}^2}\sum_{t=1}^T \frac{1}{\delta^2S_{a, T}^2}\EE[(\bc^\top \bZ_t)^4|\cH_{t-1}]\\
=& \frac1{\delta^2 S_{a, T}^4}\sum_{t=1}^T\EE
\left[
\frac{1_{\{A_t=a\}}}{\pi_t(A_t)^4}\big(\bc^\top \bg(\bX_t, Y_t(a);\btheta_a^*)\big)^4\bigg|\cH_{t-1}
\right]\\
=& \frac1{\delta^2 S_{a, T}^4}\sum_{t=1}^T\EE
\left[
\EE_{A_t}\bigg[\frac{1_{\{A_t = a\}}}{\pi_t(A_t)^4}\bigg|\cH_{t-1}, \bX_t\bigg]
\cdot
\EE_{Y(a)}
\big[
\big(\bc^\top \bg(\bX_t, Y_t(a);\btheta_a^*)\big)^4
|\cH_{t-1}, \bX_t\big]
\bigg|\cH_{t-1}\right]\\
= & \frac1{\delta^2 S_{a, T}^4}\sum_{t=1}^T
\EE\left[
\frac{1}{\pi_t(a)^3}\cdot 
\EE_{Y(a)}
\big[
\big(\bc^\top \bg(\bX_t, Y_t(a);\btheta_a^*)\big)^4
|\cH_{t-1}, \bX_t\big]
\bigg|\cH_{t-1}\right]\\
\leq & \frac1{\delta^2 S_{a, T}^4}\sum_{t=1}^T \pi_{\min,t}^{-3}\EE\big[
\big(\bc^\top \bg(\bX_t, Y_t(a);\btheta_a^*)\big)^4\big]\to 0.
\end{align*}
Here the first inequality uses Chebyshev's Inequality. The last inequality is due to Assumption \ref{aspt:min-sampling-prob}. The final convergence is deduced from the conditions in Theorem \ref{thm::asymptotic-normality-general}.
    }
    
\subsection{Proof of Lemma \ref{lem::policy-with-no-contexts}}\label{apdx::proof-lem::policy-with-no-contexts}

{
By Chebyshev's inequality, $\forall \delta>0$, 
\begin{align}
    &\PP\bigg(\bigg|N_{i, t} - \sum_{i=1}^t\EE[1_{\{A_\tau = i\}}|\cH_{\tau-1}^0]\bigg|\geq \delta\bigg)\nonumber\\
    \leq & \frac1{\delta^2}\EE\bigg(\sum_{\tau=1}^t\big(1_{\{A_\tau = i\}} - \EE[1_{\{A_\tau = i\}}|\cH_{\tau-1}^0]\big)\bigg)^2\nonumber\\
    =& \frac1{\delta^2}\sum_{\tau=1}^t\EE\big(1_{\{A_\tau = i\}} - \EE[1_{\{A_\tau = i\}}|\cH_{\tau-1}^0]\big)^2\leq \frac{t}{\delta^2}.\label{eq::concentration-Nit}
\end{align}
In addition, given the minimum sampling probability $\pi_{\min}$, we have 
\begin{equation}\label{eq::policy-no-context-sum-e-1}
\sum_{\tau = 1}^t\EE[1_{\{A_\tau = i\}}|\cH_{\tau-1}^0]\geq \pi_{\min} t.
\end{equation}
Thus, by setting $\delta = \pi_{\min}t / 2$ in (\ref{eq::concentration-Nit}), we combine with (\ref{eq::policy-no-context-sum-e-1}) and deduce that 
\begin{equation}\label{eq::Nit-lower-bound}
    \PP\bigg(N_{i, t}\leq \frac{\pi_{\min}t}{2}\bigg)\leq \frac{4}{\pi_{\min}^2t}.
\end{equation}
This implies that 
$$
\PP\bigg(\frac{C_t}{N_{i, t}}\geq \frac{2C_t}{\pi_{\min}t}\bigg)\leq \frac{4}{\pi_{\min}^2t},
$$
which proves statement (i). 

At the same time, $\forall \delta>0$,
\begin{align}
&\PP\bigg(N_{i, t}\bigg|\hat\mu_{i, t} - \mu_i^*\bigg|\geq \delta\bigg)\nonumber\\
    =&\PP\bigg(\bigg|\sum_{i=1}^t 1_{\{A_\tau = i\}}Y_{\tau} - \sum_{i=1}^t\EE[1_{\{A_\tau = i\}}Y_{\tau}|\cH_{\tau-1}^0, A_\tau]\bigg|\geq \delta\bigg)\nonumber\\
    \leq & \frac1{\delta^2}\EE\bigg(\sum_{i=1}^t \big(1_{\{A_\tau = i\}}Y_{\tau} - \EE[1_{\{A_\tau = i\}}Y_{\tau}|\cH_{\tau-1}^0, A_\tau]\big)\bigg)^2\nonumber\\
    =& \frac1{\delta^2}\sum_{\tau=1}^t\EE1_{\{A_\tau = i\}}\big(Y_\tau(i) - \mu_i^*\big)^2\leq \frac{\sigma_Y^2t}{\delta^2}.\label{eq::concentration-hatmuit}
\end{align}
Here we have used the definition of $N_{i, t}$ and $\hat\mu_{i, t}$, as well as the fact that due to the unconfoundedness assumption,
$$
\EE[1_{\{A_\tau = i\}}Y_{\tau}|\cH_{\tau-1}^0, A_\tau] = 1_{{\{A_\tau = i\}}}\mu_i^*.
$$
Combining (\ref{eq::Nit-lower-bound}) and (\ref{eq::concentration-hatmuit}), we obtain that $\forall \delta>0$,
$$
\PP\bigg(|\hat\mu_{i, t} - \mu_i^*|\geq \frac{2\delta}{\pi_{\min} t}\bigg)\leq \frac{\sigma_Y^2t}{\delta^2} + \frac{4}{\pi_{\min}^2 t}.
$$
Let $\delta' = \frac{2\delta}{\pi_{\min} t}$, and we obtain that 
$$
\PP\bigg(|\hat\mu_{i, t} - \mu_i^*|\geq \delta'\bigg)\leq \frac{4\sigma_Y^2}{\delta'^2t} + \frac{4}{\pi_{\min}^2 t}\rightarrow 0
$$
as $t\rightarrow \infty$. Thus, statement (ii) is proved.
}

\section{Additional Technical Lemmas}
\begin{lemma}[Matrix Azuma \citep{tropp2012user}]
    \label{lem::matrix-azuma}

    Consider a finite adapted sequence $\{X_k\}_{k = 1}^t$ of self-adjoint $d\times d$ matrices with respect to the filtration $\{\mathcal{F}_k\}_{k = 1}^t$, and a fixed sequence $\{A_k\}_{k = 1}^t$ of self-adjoint matrices that satisfy
    \begin{align}
        \EE[A_k \mid \mathcal{F}_{k-1}] = 0, \quad \text{and} \quad X_k^2 \preceq A_k^2 \text{ a.s.}
    \end{align}
    Compute the variance parameter 
    \begin{align}
        \sigma_t^2 = \left\|\sum_{i = 1}^t A_i^2 \right\|_2.
    \end{align}
    Then, for all $\epsilon \geq 0$, 
    \begin{align}
        \PP\left( \left\|\sum_{k = 1}^t X_k \right\|_2 \geq \epsilon \right) \leq d \exp\left(-\frac{\epsilon^2}{8\sigma_t^2}\right).
    \end{align}
\end{lemma}

Applying the union bound on all $t' \in [t]$, we have the following corollary.

\begin{corollary}
    \label{cor::matrix-azuma}
    Under the same conditions as in Lemma \ref{lem::matrix-azuma}, we have that with a probability at least $1-\delta/(t-1)$, for all $\tau \in t, t+1, \ldots$,
    \begin{align}
        \frac{1}{\tau}\left\|\sum_{k = 1}^{\tau} X_k \right\|_2 \leq \sqrt{\frac{16 \sigma_{\tau}^2}{\tau^2} \log\left(\frac{\tau d}{\delta} \right)}.
    \end{align}
\end{corollary}
\begin{proof}
    The proof is to apply the union bound on all $\tau \in t, t+1, \ldots$ with each event having a probability of at least $1-\delta/\tau^2$. The total failure probability is at most $\sum_{\tau = t}^{\infty}\frac{1}{\tau^2} \leq \frac{\delta}{t-1}$.
\end{proof}

\begin{lemma}[Law of large numbers for martingale difference sequence \citep{chow1967strong}]
    \label{lem::law-of-large-numbers-for-martingale-difference-sequence}
    Let $Y_n = \sum_{t=1}^n \bX_t$ be a martingale difference sequence, such that 
    \begin{align}
        \sum_{t=1}^{\infty} \EE\left[|\bX_t|^{2\alpha}\right] / k^{1+\alpha} < \infty.
    \end{align}
    Then,
    \begin{align}
        \frac{1}{n} \sum_{t=1}^n Y_t \xrightarrow{a.s.} 0.
    \end{align}
\end{lemma}

\begin{theorem}[Stochastic approximation convergence (Theorem 2 \citep{borkar2008stochastic})]
\label{thm:sac_converge}
    Let $\{x_n\}_{n \geq 0}$ be the stochastic approximation process in $\mathbb{R}^d$ given by 
    \begin{align}
    \bx_{n+1}=\bx_n+a(n)\left[\bh\left(x_n\right)+\bM_{n+1}\right], n \geq 0
    \end{align}
    with prescribed $\bx_0$ with the following assumptions holding:
    \begin{itemize}
        \item The map $\bh: \mathbb{R}^d \mapsto \mathbb{R}^d$ is Lipschitz: $\|\bh(\bx) - \bh(y)\| \leq L \|\bx-\by\|$ for some $0 < L < \infty$.
        \item Stepsizes $\{a(n)\}$ are positive scalars satisfying
        $$
            \sum_n a(n) = \infty, \sum_n a(n)^2 < \infty
        $$
        \item $\{\bM_n\}$ is a martingale difference sequence with respect to the increasing family of $\sigma$-fields 
        $$
            \mathcal{F}_n \stackrel{\text { def }}{=} \sigma\left(\bx_m, \bM_m, m \leq n\right)=\sigma\left(\bx_0, \bM_1, \ldots, \bM_n\right), n \geq 0.
        $$
        That is 
        $$
            \EE[\bM_{n+1} \mid \cF_n] = 0 \text{ a.s. }, n \geq 0.
        $$
        \item Furthermore, $\{\bM_n\}$ are square-integrable with
        $$
            \EE[\|\bM_{n+1}\|^2 \mid \cF_{n}] \leq K (1+ \|\bx_n\|^2) \text{ a.s.,} n \geq 0.
        $$
        for some constant $K$.
        \item {\emph{(Stability.)} The iterates remain bounded almost surely, i.e.\ $\sup_{n\geq 0}\|\bx_n\|<\infty$ a.s.}
    \end{itemize}
    {Then the sequence $\{\bx_n\}$ converges almost surely to a (possibly sample-path-dependent) compact connected internally chain transitive invariant set of the limiting o.d.e.\ $\dot{\bx}(t)=\bh(\bx(t))$, $t\geq 0$. In particular, if this o.d.e.\ has a \emph{globally asymptotically stable equilibrium} $\bx^*$ (equivalently, $\{\bx^*\}$ is its only internally chain transitive invariant set), then $\bx_n\to\bx^*$ almost surely.}
\end{theorem}

\begin{theorem}[Banach fixed-point theorem]\label{thm:banach}
Let $\cK$ be a nonempty closed subset of $\RR^d$ and $T:\cK\to\cK$ a contraction, i.e.\ there is $\kappa\in[0,1)$ with $\|T(\bx)-T(\by)\|_2\le\kappa\|\bx-\by\|_2$ for all $\bx,\by\in\cK$. Then $T$ has a unique fixed point $\bx^*\in\cK$, and $T^n(\bx_0)\to\bx^*$ for every $\bx_0\in\cK$.
\end{theorem}

\begin{theorem}[Robbins--Siegmund \citep{robbins1971convergence}]\label{thm:robbins-siegmund}
Let $\{\cF_n\}_{n\ge0}$ be a filtration and let $V_n,\beta_n,\xi_n,\zeta_n\ge0$ be nonnegative $\cF_n$-measurable random variables with
\[
    \EE[V_{n+1}\mid\cF_n]\le(1+\beta_n)\,V_n+\xi_n-\zeta_n\quad\text{a.s.}, \qquad n\ge0.
\]
Then on the event $\bigl\{\sum_n\beta_n<\infty\bigr\}\cap\bigl\{\sum_n\xi_n<\infty\bigr\}$, the sequence $V_n$ converges almost surely to a finite limit and $\sum_n\zeta_n<\infty$ almost surely.
\end{theorem}

\begin{theorem}[Borkar--Meyn stability criterion (Theorem~7 of \citealp{borkar2008stochastic}; \citealp{borkar2000ode})]\label{thm:borkar-meyn}
Consider the stochastic approximation $\bx_{n+1}=\bx_n+a(n)\bigl(\bh(\bx_n)+\bM_{n+1}\bigr)$ under the Lipschitz map, stepsize, and square-integrable martingale-difference conditions of Theorem~\ref{thm:sac_converge}. For $c\ge1$ define the scaled fields $\bh_c(\bx):=\bh(c\bx)/c$, and suppose $\bh_c\to\bh_\infty$ as $c\to\infty$ uniformly on compact sets, for some continuous $\bh_\infty:\RR^d\to\RR^d$. If the o.d.e.\ $\dot{\bx}(t)=\bh_\infty(\bx(t))$ has the origin as its unique globally asymptotically stable equilibrium, then $\sup_{n\ge0}\|\bx_n\|<\infty$ almost surely.
\end{theorem}

\begin{theorem}[Lyapunov global asymptotic stability {\citep[Theorem~4.1]{khalil2002nonlinear}}]\label{thm:lyapunov-gas}
Let $\bx=\mathbf0$ be an equilibrium of $\dot{\bx}(t)=\bm{f}(\bx(t))$ with $\bm{f}$ locally Lipschitz. Suppose there is a continuously differentiable $V:\RR^d\to\RR$ with $V(\mathbf0)=0$, $V(\bx)>0$ for $\bx\neq\mathbf0$, $V$ radially unbounded ($V(\bx)\to\infty$ as $\|\bx\|_2\to\infty$), and $\dot V(\bx)=\nabla V(\bx)^\top\bm{f}(\bx)<0$ for all $\bx\neq\mathbf0$. Then the origin is globally asymptotically stable.
\end{theorem}

\section{Additional Details on Simulation Studies}
\label{apdx::simulation-details}

\subsection{Environment Settings}
\label{apdx::simulation-details::environment-settings}

We give the details of the simulation environment settings. 

The first type of environments is the noisy contextual linear bandit environment including \texttt{NC-Hard1}, \texttt{NC-Hard2}, \texttt{NC-Gaussian}. Each environment has a ground-truth parameter $\btheta_a^*$. At each time $t$, the following variables are generated:
\begin{equation}\label{eq::model-1}
    \begin{aligned}
        \text{True context: } & \bS_t \sim \cD_S,\\
        \text{Predicted context: } & f(\bX_t) = \bS_t + \bepsilon_t, \text{ where } \bepsilon_t \sim \cD_{\epsilon}(\cdot \mid \bS_t),\\
        \text{Reward: } & Y_t = \langle\btheta_{A_t}^*, \bS_t\rangle + \eta_t, \text{ where } \eta_t \sim \cD_{\eta},
    \end{aligned}.
\end{equation}
where $A_t$ is the algorithm-chosen action. Recall that $\Sigma_S = \EE[\bS_t \bS_t^{\top}]$, $\Sigma_e$ is the covariance matrix of $\bepsilon_t$ (assumed to be independent of $\bS_t$), and $\Sigma_{\eta}$ is the covariance matrix of $\eta_t$.

Both hard environments have one-dimensional context and true parameters ($d = 1$), two actions $|\cA| = 2$, and two contexts $\cS = \{0, -1\}$. The context distribution $\cD_S$ is uniform over $\cS$. \texttt{hard-1} and \texttt{hard-2} have the true parameters $\btheta_0^* = (3, 1)$ and $\btheta_1^* = (-3, -1)$ respectively. They further share the same prediction error distribution $\cD_{\epsilon}(\cdot \mid \bS_t)$, given by (\ref{eq::model-1-error-distribution}). We set the reward noise $\cD_{\eta} = \cN(0, \sigma_{\eta}^2)$.
\begin{align}
\begin{array}{ll}
            \PP(f(\bX_t) = 1 \mid \bS_t = 0) = 2/3 & \quad \PP(f(\bX_t) = -2 \mid \bS_t = 0) = 1/3 \\
            \PP(f(\bX_t) = -2 \mid \bS_t = -1) = 2/3 & \quad \PP(f(\bX_t) = 1 \mid \bS_t = -1) = 1/3.
\end{array}
\label{eq::model-1-error-distribution}
\end{align}

For the \texttt{NC-Gaussian} environment, we randomly sample the true parameters $\btheta_a^*$ from $\cN(0, \Sigma_{\btheta})$ for each $a\in\cA$ independently. We choose $\cD_S = \cN(0, \Sigma_S)$, $\cD_{\epsilon}(\cdot \mid \bS_t) = \cN(0, \Sigma_e)$, $\cD_{\eta} = \cN(0, \sigma_{\eta}^2)$, respectively.

The second type of environments is the misspecified contextual linear bandit environment including \texttt{MC-Polynomial}, \texttt{MC-Neural}. In these two environments, we have $|\cA| = 2$, $d = 1$. The context is sampled from $\cD_X = \cN(0, \Sigma_X)$. In the \texttt{MC-Polynomial} environment, the true reward function is given by 
$$
    y(\bx, a) = \langle\btheta_{a, 1}^*, \bx\rangle + \langle\btheta_{a, 2}^*, \bx^2\rangle + \dots + \langle\btheta_{a, d}^*, \bx^d\rangle,
$$
where $d$ is the degree of the polynomial. The true parameters $\btheta_{a, i}^*$ are randomly sampled from $\cN(0, \Sigma_{\btheta})$ for each $a\in\cA$ and $i\in\{1, 2, \ldots, d\}$ independently. In the \texttt{MC-Neural} environment, the true reward function is given by a two layer neural network with one hidden layer of size $d$.
$$
    y(\bx, a) = \operatorname{ReLU}(\langle\btheta_{a}^*, \bx\rangle),
$$
where $\operatorname{ReLU}(x) = \max(0, x)$ is the ReLU activation function.

The true parameters $\btheta_{a}^*$ are randomly sampled from $\cN(0, \Sigma_{\btheta})$ for each $a\in\cA$ independently. We choose $\cD_X = \cN(0, \Sigma_X)$, $\cD_{\eta} = \cN(0, \sigma_{\eta}^2)$, respectively.

\subsection{Additional Information on OPE}
\label{apdx::simulation-details::ope}

In the OPE setting, we compare the proposed inference method with the CADR (Contextual Adaptive Doubly Robust) method  \citep{bibaut2021post} under various choice of prediction model including linear model, tree-based model, and a dumpy model that always outputs 0. We run CADR on the same dataset collected by Boltzmann exploration w.r.t. Ridge regression in five environments introduced above. 

To implement the CADR method, we define the following functions:
\begin{align}
    \Psi\left(g, Q_Y\right):= \EE_{A_t \sim g(A_t \mid \bX_t)}[Q_Y(A_t, \bX_t)].
\end{align}
\begin{align}
    D^{\prime}(g, \bar{Q})(x, a, y):=\frac{g^*(a \mid x)}{g(a \mid x)}(y-\bar{Q}(a, x))+\int \bar{Q}\left(a^{\prime}, x\right) g^*\left(a^{\prime} \mid x\right) d \mu_{\mathcal{A}}\left(a^{\prime}\right).
\end{align}
\begin{align}
    D(g, \bar Q)(x, a, y) = D'(g, \bar Q)(x, a, y) - \Psi\left(g, \bar Q\right).
\end{align}

Let $g_1, \dots, g_T$ be the logging policy that collects the data, and $g^*$ be the target policy that we aim to evaluate. 

For each step $t = 1, \dots, T$, the CADR method computes the following quantities:
\begin{itemize}
    \item Train $\hat Q_{t-1}: \cX \times \cA \to \RR$ on the dataset $((\bX_s, A_s, Y_s))_{s = 1}^{t-1}$ using the outcome regression estimator.
    \item Set $D'_{t, s} = D(g_s, \hat Q_{t-1})(\bX_t, A_t, Y_t)$ for each $s = t, \dots, T$.
    \item Set 
    \begin{align}
        \hat \sigma_t^2 = \frac{1}{t-1} \sum_{s = 1}^{t-1} \frac{g_t(A_s \mid \bX_s)}{g_s(A_s \mid \bX_s)} (D'_{t, s})^2 - \left(\frac{1}{t-1} \sum_{s = 1}^{t-1} \frac{g_t(A_s \mid \bX_s)}{g_s(A_s \mid \bX_s)} D'_{t, s} \right)^2.
    \end{align}
\end{itemize}
In the end, the CADR method outputs the following estimate:
\begin{align}
    \hat \Psi_T = \frac{\Gamma_T}{T} \sum_{t=1}^T \hat \sigma_t^{-1} D'_{t, t}, \text{ where } \Gamma_T = \left(\frac{1}{T}\sum_{t=1}^T \hat \sigma_t^{-1}\right)^{-1}.
\end{align}
and confidence interval
\begin{align}
    \text{CI}_{\alpha} = [\hat \Psi_T \pm \xi_{1-\alpha/2} \Gamma_T / \sqrt{T}].
\end{align}

\subsection{Additional Results: Variance of the IPW-Z Estimator}
\label{apdx::simulation-details::additional-results}

Figure \ref{fig::variance-misspecified-linear-bandits} reports the variance of the IPW-Z estimator over 10,000 steps for Target \ref{ex::misspecified-linear-bandits}, complementing the coverage results in Figure \ref{fig::results-misspecified-linear-bandits}.
The policies that keep the action-selection probabilities bounded away from zero---\texttt{Random} and the clipped smooth allocation policies \texttt{ridge + Softmax} and \texttt{SGD + Softmax}---have stable variances at comparable levels across all environments, consistent with Section \ref{sec::boltzmann-policy-convergence}, which establishes that preventing the action-selection probabilities from becoming too small controls the variance.
\texttt{MAB-EG}, whose decaying floor $\pi_{\min,t}$ lets the probabilities vanish, attains the smallest variance in some environments (\texttt{NC-Hard2}, \texttt{MS-Neural}) but the largest finite-sample variance, by one to two orders of magnitude, in the environments designed to destabilize the fit (\texttt{NC-Hard1}, \texttt{NC-Gaussian}).
The decaying floor samples \texttt{MAB-EG}'s suboptimal arms at the vanishing rate $r_{a,t}=\pi_{\min,t}$ (Proposition~\ref{prop::convergence-mab-no-pi-min}), which lowers its asymptotic variance but, through larger small-sample weights, widens its finite-sample fluctuations (Figure \ref{fig::results-misspecified-linear-bandits}(b)); the rate $\alpha=1/2$ still meets the conditions of Theorem~\ref{thm::asymptotic-normality-general}.

\begin{figure}[tb]
    \centering
    \includegraphics[width=1\textwidth]{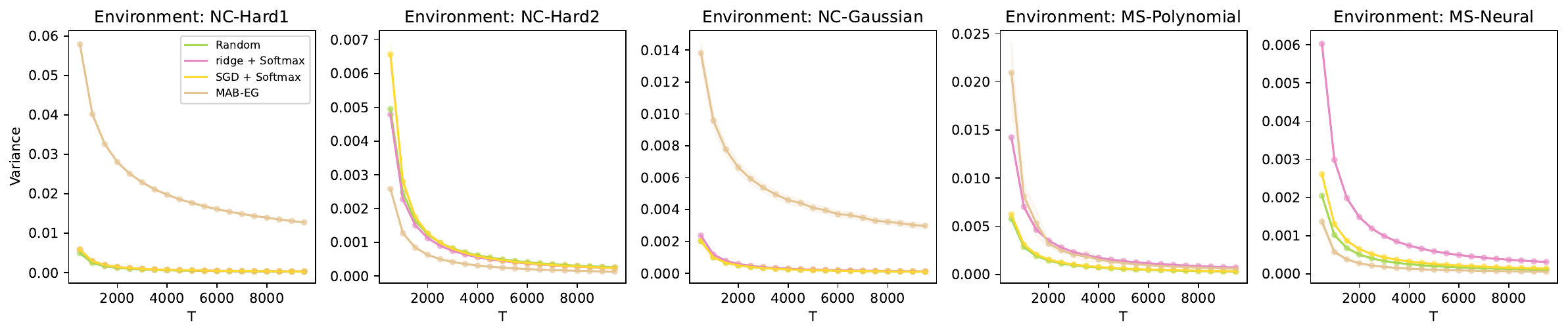}
    \caption{Variance of the IPW-Z estimator over 10,000 steps for Target \ref{ex::misspecified-linear-bandits}. Results averaged across 2,500 Monte Carlo simulations. Shaded bands denote $\pm 2$ Monte Carlo standard errors.}
    \label{fig::variance-misspecified-linear-bandits}
\end{figure}

\subsection{Details of the HeartSteps V1 Real-Data Application}\label{sec::simulations-heartsteps-details}

This subsection records the full construction of the HeartSteps V1 study of Section~\ref{sec::simulations-heartsteps}, so that the application is reproducible and the modeling choices are explicit; for completeness we restate the relevant estimating equations and variance formulas rather than only referencing them.
The environment adapts the semi-synthetic simulator of \citep{guo2024online}, calibrated to the HeartSteps V1 micro-randomized trial \citep{klasnja2015microrandomized, liao2016sample}.

\paragraph{Environment and covariates.}
HeartSteps V1 is a $42$-day mobile health study in which, at up to five decision points per day, the algorithm chooses whether to deliver a contextually tailored physical-activity suggestion ($A_t\in\{0,1\}$) and observes the log step count over the following $30$ minutes as the reward $Y_t$.
Following \citep{guo2024online}, contexts are drawn i.i.d.\ from the pooled empirical distribution of the trial participants: we collect every decision point of all $37$ users into a single bank and, at each time $t$, sample one row with replacement.
The design vector is $\bX_t=(1,\;\texttt{sum\_step},\;\texttt{jbsteps30pre},\;\texttt{dec.temperature},\;\texttt{sd\_steps60},\;\texttt{home.location},\;\texttt{dosage})\in\RR^{7}$, where the first six entries are real logged covariates and $\texttt{dosage}$ is the accumulated treatment burden.
The four continuous covariates ($\texttt{sum\_step}$, $\texttt{jbsteps30pre}$, $\texttt{dec.temperature}$, $\texttt{dosage}$) are standardized to zero mean and unit variance across the pooled bank, and the standardization shift and scale are folded into the reward parameters so that the reward is unchanged; the two indicator covariates are left untransformed.

\paragraph{Outcome and measurement-error model.}
The potential outcome is linear in the unobserved true context $\bS_t$,
\[
  Y_t(a) = \bS_t^\top\btheta_a^* + \eta_t,
\]
with mean-zero reward noise of variance $\sigma_\eta^2 = 1$.
The algorithm does not observe $\bS_t$; it observes the noisy proxy
\[
  \bX_t = \bS_t + \bepsilon_t, \qquad \EE[\bepsilon_t] = \mathbf{0}, \quad \mathrm{Cov}(\bepsilon_t) = \bSigma_e,
\]
where the measurement error perturbs only the (standardized) $\texttt{dosage}$ coordinate, indexed by $j$ with standard basis vector $\be_j$:
\[
  \bepsilon_t = \varepsilon_t\,\be_j, \qquad \varepsilon_t \text{ mean zero with variance } \sigma_e^2,
  \qquad\text{so}\qquad \bSigma_e = \sigma_e^2\,\be_j\be_j^\top .
\]
We report $\sigma_e^2\in\{0.1,0.25,0.5\}$ on the standardized scale, with $\sigma_e^2=0.1$ as the default.
This reflects practice in mobile health, where the burden is itself a noisy prediction rather than a directly observed quantity.

\paragraph{Inferential target.}
The per-action parameter $\btheta_a^*$ is the solution of the moment condition
\[
  \EE\big[\bg(\bX_t, Y_t(a); \btheta_a^*)\big] = \mathbf{0},
  \qquad
  \bg(\bx, y; \btheta) = \bx y - (\bx\bx^\top - \bSigma_e)\,\btheta ,
\]
the noisy-context score of Target~\ref{ex::bandits-noisy-contexts}, which corrects the design for measurement error through the term $\bx\bx^\top-\bSigma_e$.
Writing $\bSigma_S = \EE[\bX_t\bX_t^\top] - \bSigma_e = \EE[\bS_t\bS_t^\top]$ for the true-covariate second moment, which is positive definite, the solution is unique,
\[
  \btheta_a^* = \bSigma_S^{-1}\,\EE[\bX_t\,Y_t(a)] .
\]
The parameters $\btheta_a^*$ are fixed at the generalized-estimating-equation fit of \citep{liao2016sample} on the real outcomes.
Our target is the \emph{dosage-by-treatment interaction}, the contrast on the dosage coordinate,
\[
  \tau^* = \be_j^\top(\btheta_1^* - \btheta_0^*) = \theta_{1,j}^* - \theta_{0,j}^* ,
\]
which measures how the accumulated burden moderates the effect of sending a suggestion, capturing habituation in just-in-time interventions.

\paragraph{IPW-Z estimator.}
Given a dataset $\{(\bX_t, A_t, Y_t)\}_{t=1}^T$ collected by a behavior policy $\pi$, we form the inverse-probability-weighted Z-estimator.
For a fixed action $a$, define
\[
  \bZ_t(\btheta) = \frac{1_{\{A_t = a\}}}{\pi_t(A_t)}\,\bg(\bX_t, Y_t; \btheta),
  \qquad
  \bG_T(\btheta) = \frac1T\sum_{t=1}^T \bZ_t(\btheta),
\]
and take $\hat\btheta_a^{(T)}$ to be an approximate root, $\bG_T(\hat\btheta_a^{(T)}) \approx \mathbf{0}$.
The inverse weights $1/\pi_t(A_t)$ remove the bias from adaptive sampling, so that $\bG_T(\btheta_a^*)$ is an average of a martingale-difference sequence with mean $\mathbf{0}$, which underlies consistency.
The contrast is estimated by $\hat\tau = \hat\theta_{1,j}^{(T)} - \hat\theta_{0,j}^{(T)}$.

\paragraph{Asymptotic variance and its estimate.}
By Theorem~\ref{thm::asymptotic-normality-general}, $b_{a,T}(\hat\btheta_a^{(T)} - \btheta_a^*)\xrightarrow{d}\cN(\mathbf{0},\bSigma_a^*)$, where for the bounded-propensity policies used here $b_{a,T} = \sqrt{T}$, and
\[
  \bSigma_a^* = \bJ_a^{-1}\,\bar\bI_a\,\bJ_a^{-1},
  \qquad
  \bJ_a = \EE\big[\nabla\bg(\bX_t, Y_t(a); \btheta_a^*)\big] = -\bSigma_S,
  \qquad
  \bar\bI_a = \EE\big[\rho(a,\bX_t)\,\bg\bg^\top\big],
\]
since $\nabla\bg(\bx,y;\btheta) = -(\bx\bx^\top - \bSigma_e)$.
We do not assume the bread $\bSigma_S$ is known; we estimate it from the observed contexts by the marginal plug-in
\[
  \widehat\bSigma_S = \frac1T\sum_{t=1}^T \bX_t\bX_t^\top - \bSigma_e ,
\]
which is consistent for $\bSigma_S = \EE[\bX_t\bX_t^\top] - \bSigma_e$.
The key point is that the contexts $\bX_t$ are \emph{exogenous} to the behavior policy. The algorithm selects the action $A_t$, not the context, so $\widehat\bSigma_S$ is an average of i.i.d.\ terms and converges at the $\sqrt{T}$ rate without any inverse-probability weighting, in contrast to the outcome moment $\bG_T$ whose adaptivity-induced bias is what the weights correct.
By Slutsky's theorem, substituting $\widehat\bSigma_S$ for $\bSigma_S$ in the sandwich below leaves the limiting distribution, and hence the asymptotic coverage, unchanged.
In our study $\widehat\bSigma_S$ is computed from all $T=10{,}000$ collected contexts, so its estimation error is $O_p(T^{-1/2})$, about $1\%$ in relative terms. This error enters $\sqrt{T}\,(\hat\btheta_a^{(T)}-\btheta_a^*)$ only through a higher-order term of order $O_p(T^{-1/2})$ that vanishes asymptotically, so at this horizon the plug-in is effectively equivalent to using the true $\bSigma_S$.
The meat $\bar\bI_a$ is estimated by Monte Carlo from $n=2{,}000$ fresh draws $(\bS_m, \bX_m = \bS_m + \bepsilon_m)$ of the context distribution: with
\[
  \bh_m = (\bX_m\bX_m^\top - \bSigma_e)\,\hat\btheta_a^{(T)} - \bX_m\bS_m^\top\,\hat\btheta_a^{(T)},
\]
we set
\[
  \hat{\bar\bI}_a = \frac1n\sum_{m=1}^n \frac{1}{\pi(a\mid\bX_m)}\Big(\bh_m\bh_m^\top + \sigma_\eta^2\,\bX_m\bX_m^\top\Big),
  \qquad
  \widehat\bSigma_a = \widehat\bSigma_S^{-1}\,\hat{\bar\bI}_a\,\widehat\bSigma_S^{-1} .
\]
The first term of $\hat{\bar\bI}_a$ captures the measurement-error contribution and the second the reward noise.
This Monte Carlo form of the meat $\hat{\bar\bI}_a$ uses the true context $\bS_m$ and so is available in simulation; on real data one would instead use the data-based estimator of Proposition~\ref{thm::consistent-var-estimator-general},
\[
  \widehat\bSigma_a = \hat{\dot\bG}_{a,T}^{-1}\,\hat\bI_{a,T}\,\hat{\dot\bG}_{a,T}^{-1},
  \quad
  \hat{\dot\bG}_{a,T} = \frac1T\sum_{t=1}^T \frac{1_{\{A_t=a\}}}{\pi_t(A_t)}\nabla\bg(\bX_t, Y_t; \hat\btheta_a^{(T)}),
  \quad
  \hat\bI_{a,T} = \frac1{S_{a,T}^2}\sum_{t=1}^T \frac{1_{\{A_t=a\}}}{\pi_t(A_t)^2}\bg\bg^\top,
\]
which targets the same $\bSigma_a^*$.
The variance of $\hat\btheta_a^{(T)}$ is estimated by $\widehat\bSigma_a/T$, and because the per-action estimators are constructed from disjoint time points and are asymptotically independent (Theorem~\ref{thm::asymptotic-normality-joint-general}), the contrast variance is the sum of the two per-action variances on coordinate $j$,
\[
  \widehat{\mathrm{Var}}(\hat\tau) = \frac1T\,\be_j^\top\big(\widehat\bSigma_1 + \widehat\bSigma_0\big)\be_j .
\]
A level-$(1-\alpha)$ confidence interval for $\tau^*$ is $\hat\tau \pm z_{1-\alpha/2}\sqrt{\widehat{\mathrm{Var}}(\hat\tau)}$, with $z_{1-\alpha/2}$ the standard normal quantile.

\paragraph{Behavior policies and configuration.}
We collect data with the same four behavior policies as in the synthetic studies: \texttt{Random}; \texttt{ridge + Softmax}, softmax (Boltzmann) exploration with the ridge estimator and inverse temperature $\gamma=1$; \texttt{MAB-EG}, the $\varepsilon$-greedy policy on the empirical per-arm means; and \texttt{SGD + Softmax}, softmax exploration with the same inverse temperature on a stochastic-gradient estimator.
All three exploiting policies---\texttt{MAB-EG} and the two softmax policies (\texttt{ridge + Softmax}, \texttt{SGD + Softmax})---clip their action-selection probabilities to a fixed floor $\pi_{\min}=0.2$; only \texttt{Random}, which already samples uniformly, is unclipped.
Both softmax policies assign, before clipping,
\[
  \pi_t(a\mid\bX_t) \propto \exp\big(\gamma\,\langle\hat\bbeta_{a,t},\bX_t\rangle\big),
\]
with $\hat\bbeta_{a,t}$ the ridge estimate (for \texttt{ridge + Softmax}) or the SGD estimate (for \texttt{SGD + Softmax}); we set $\gamma=1$ to keep the allocation smooth and clip the result at $\pi_{\min}=0.2$ to bound the inverse-propensity weights $1/\pi_t(A_t)$.
The clipping is essential for \texttt{SGD + Softmax}: the stochastic-gradient score is unregularized and, unlike the ridge score, would otherwise saturate the softmax at the calibrated reward scale, driving the propensities toward $0$ and breaking Conditions (i)--(ii) of Theorem~\ref{thm::asymptotic-normality-general}.
Each configuration uses horizon $T=10{,}000$, $2{,}000$ Monte Carlo replications, and coverage checks every $1{,}000$ steps.

\paragraph{Remarks on the modeling choices.}
Two design choices are needed for the variance estimator to be well behaved at this scale, and are worth recording.
First, pooling the contexts across users (rather than replaying a single user's trajectory) avoids rank-deficient designs that arise because some indicator covariates are constant within a user, which would make $\bSigma_S$ singular.
Second, standardizing the continuous covariates removes the near-collinearity between the $\texttt{dosage}$ covariate and the intercept, which otherwise makes the dosage coordinate of $\bSigma_a^* = \bSigma_S^{-1}\bar\bI_a\bSigma_S^{-1}$ ill-conditioned; standardization reduces the condition number of $\bSigma_S$ by roughly two orders of magnitude.
Neither choice affects the inferential target $\tau^*$, which remains the GEE-calibrated dosage-by-treatment interaction.

\section{Examples}\label{seq::examples}


In this section, we revisit the three examples from Section \ref{sec::problem-setup} and establish statistical inference guarantees by applying Theorem \ref{thm::asymptotic-normality-general} and \ref{thm::asymptotic-normality-joint-general} to their corresponding score functions.

Throughout this section, we assume that, for every action \(a\in\cA\), the behavior policy
satisfies scaled inverse-propensity convergence with rate \(\{r_{a,t}\}_{t\ge1}\) and limit
function \(\rho(a,\cdot)\), together with the remaining policy-side conditions of Theorem \ref{thm::asymptotic-normality-joint-general}, including Assumption \ref{aspt:min-sampling-prob}. We use the notation
\[
S_{a,T}^{2}:=\sum_{t=1}^{T}\frac{1}{r_{a,t}},
\qquad
b_{a,T}:=\frac{T}{S_{a,T}},
\qquad
\bD_T:=\operatorname{Diag}
\bigl(b_{1,T}\bI_d,\ldots,b_{K,T}\bI_d\bigr)
\]
from Section~\ref{sec::inference-guarantee}. 

\subsection{Misspecified Linear Bandits}\label{sec::ex1}

We first consider Example \ref{ex::misspecified-linear-bandits}, where the target parameter $\btheta_a^*$ is defined by (\ref{eq::theta-a-*}) with the score function $\bg$ in (\ref{eq::target-parameter-misspecified-linear-bandits}). A natural estimator for $\btheta_a^*$ that satisfies (\ref{eq::estimating-equation-general}) is 
\begin{equation}
\hat\btheta_a^{(T)} = \left(\frac1T\sum_{t=1}^T\frac{1_{\{A_t = a\}}}{\pi_t(A_t)}\bX_t\bX_t^\top\right)^{-1}\left(\frac1T\sum_{t=1}^T\frac{1_{\{A_t = a\}}}{\pi_t(A_t)}\bX_tY_t\right).
\end{equation}

In this setting, Assumptions \ref{aspt:identifiability-general}-\ref{aspt:smoothness-general}, which form part of the conditions for Theorem \ref{thm::asymptotic-normality-joint-general}, are implied by the following regularity assumption.
\begin{assumption}\label{aspt:ex1}
There exist constants $R_\theta, M, M_Y>0$ such that $\sup_{a\in\cA}\|\btheta_a^*\|_2< R_\theta$, $\|\bX_t\|_2\leq M$ a.s., and $\sup_{a\in\cA}
\EE\!\left[|Y_t(a)|^4\mid\bX_t\right]
\le M_Y$ a.s. Moreover, $\bSigma_X:= \EE[\bX_t\bX_t^\top]$ is invertible. 
\end{assumption}

With Assumption \ref{aspt:ex1} in place, together with the remaining conditions of Theorem \ref{thm::asymptotic-normality-joint-general}, the joint asymptotic normality of $\hat\btheta^{(T)} = \big((\hat\btheta_1^{(T)})^\top, \ldots, (\hat\btheta_K^{(T)})^\top\big)^\top$ and a consistent estimator for its asymptotic variance follow directly from Theorem \ref{thm::asymptotic-normality-joint-general} and Proposition \ref{thm::consistent-var-estimator-general}, as stated below.

\begin{corollary}\label{cor::ex1}
Suppose Assumptions \ref{aspt:unconfoundedness} and \ref{aspt:ex1} hold, and suppose the common
policy-side conditions stated at the beginning of this appendix hold for every
\(a\in\cA\). Then
\[
\bD_T\bigl(\widehat{\btheta}^{(T)}-\btheta^*\bigr)
\xrightarrow{d}
\cN\!\left(
\boldsymbol 0,
\operatorname{diag}
\bigl(\bSigma_1^*,\ldots,\bSigma_K^*\bigr)
\right),
\]
where $\bSigma_a^*
=
\bSigma_X^{-1}\bar \bI_a^{\mathrm{lin}}\bSigma_X^{-1}$, and
$\bar \bI_a^{\mathrm{lin}}
:=
\EE\!\left[
\rho(a,\bX_t)
\bigl(Y_t(a)-\bX_t^\top\btheta_a^*\bigr)^{2}
\bX_t\bX_t^\top
\right]$.

A consistent estimator of \(\bSigma_a^*\) is $\widehat\bSigma_{a,T}
=
\widehat \bJ_{a,T}^{-1}
\widehat \bI_{a,T}
\widehat \bJ_{a,T}^{-1}$,
where
\begin{equation}\label{eq::estimator-asympt-var-ex1}
\widehat \bJ_{a,T}
:=
\frac{1}{T}
\sum_{t=1}^{T}
\frac{1_{\{A_t=a\}}}{\pi_t(A_t)}
\bX_t\bX_t^\top, \quad \widehat \bI_{a,T}
:=
\frac{1}{S_{a,T}^{2}}
\sum_{t=1}^{T}
\frac{1_{\{A_t=a\}}}{\pi_t(A_t)^2}
\bigl(Y_t-\bX_t^\top\widehat{\btheta}_a^{(T)}\bigr)^2
\bX_t\bX_t^\top.  
\end{equation}
\end{corollary}

\begin{remark}\label{rmk::ex1-simple-var-est}
In Corollary \ref{cor::ex1}, a simpler consistent estimator for the asymptotic variance is obtained by replacing \(\widehat \bJ_{a,T}\) by $\hat\bSigma_X := \frac1T\sum_{t=1}^T\bX_t\bX_t^\top$. both matrices consistently estimate $\bSigma_X$.
\end{remark}

In the special case where the behavior policy is an $\epsilon$-greedy policy combined with the weighted online LS estimator, Corollary \ref{cor::ex1} yields the same asymptotic normality result as \citep{chen2021statistical}; see also the discussion following Theorem \ref{thm::asymptotic-normality-general}. The variance estimator used in \citep{chen2021statistical} corresponds to the simplified plug-in form described in Remark \ref{rmk::ex1-simple-var-est}. 

\subsection{Bandits with Noisy Contexts}\label{sec::ex2}

We now revisit Example \ref{ex::bandits-noisy-contexts}. Recall that the potential outcome satisfies $Y_t(a) = \bS_t^\top \btheta_a^* + \eta_t$ with unobserved state $\bS_t$, while the observed context is the noisy proxy $\bX_t = \bS_t + \bepsilon_t$. The noise $\bepsilon_t$ satisfies $\EE[\bepsilon_t|\bS_t] = \mathbf{0}$, and $\mathrm{Var}(\bepsilon_t|\bS_t) = \bSigma_e$ for a constant matrix $\bSigma_e\in\RR^{d\times d}$. No parametric assumptions are imposed on the distribution of $\bepsilon_t$. In addition, assume $\mathbb{E}[\eta_t\mid\bS_t,\bX_t]=0$ and define $\sigma_\eta^2=\mathrm{Var}(\eta_t\mid\bS_t,\bX_t)$.

The score function reduces to (\ref{eq::target-parameter-bandits-noisy-contexts}). An estimator of $\btheta_a^*$ satisfying (\ref{eq::estimating-equation-general}) is
\begin{equation}
\hat\btheta_a^{(T)} = \left[\frac1T\sum_{t=1}^T\frac{1_{\{A_t = a\}}}{\pi_t(A_t)}(\bX_t\bX_t^\top-\bSigma_e)\right]^{-1}\left(\frac1T\sum_{t=1}^T\frac{1_{\{A_t = a\}}}{\pi_t(A_t)}\bX_tY_t\right).
\end{equation}

To infer the target parameter $\btheta_a^*$ via Theorem \ref{thm::asymptotic-normality-joint-general}, we impose the following regularity condition, which guarantees Assumptions \ref{aspt:identifiability-general}–\ref{aspt:smoothness-general} required for the theorem. 

\begin{assumption}\label{aspt:ex2}
There exist constants $R_\theta, M, M_\eta > 0$ such that $\sup_a\|\btheta_a^*\|_2< R_\theta$, \\$\max\{\|\bX_t\|_2, \|\bS_t\|_2\}\leq M$ a.s., $\EE[\eta_t^4|\bS_t, \bX_t]\leq M_\eta$ a.s. In addition, $\bSigma_S: = \EE[\bS_t\bS_t^\top]$ is invertible.
\end{assumption}

The following corollary, obtained from Theorem \ref{thm::asymptotic-normality-joint-general} and Proposition \ref{thm::consistent-var-estimator-general}, provides inference for $\{\btheta_a^*\}_{a\in\cA}$ in this setting. 

\begin{corollary}\label{cor::ex2}
Suppose Assumptions~\ref{aspt:unconfoundedness} and~\ref{aspt:ex2} hold, and suppose the
common policy-side conditions stated at the beginning of this appendix hold for every
\(a\in\cA\). Then
\[
\bD_T\bigl(\widehat{\btheta}^{(T)}-\btheta^*\bigr)
\xrightarrow{d}
\cN\!\left(
\boldsymbol 0,
\operatorname{diag}
\bigl(\bSigma_1^*,\ldots,\bSigma_K^*\bigr)
\right),
\]
where $\bSigma_a^*=
\bSigma_S^{-1}\bar \bI_a^{\mathrm{NC}}\bSigma_S^{-1}$, 
with $\bar \bI_a^{\mathrm{NC}}
:=
\EE\!\left[
\rho(a,\bX_t)
\left(
\bh_a(\bX_t,\bS_t)\bh_a(\bX_t,\bS_t)^\top
+
\sigma_{\eta}^2
\bX_t\bX_t^\top
\right)
\right]$, and $\bh_a(\bx,\bs)
:=
(\bx\bx^\top-\bSigma_e)\btheta_a^*
-
\bx\bs^\top\btheta_a^*$.

A consistent estimator of \(\bSigma_a^*\) is $\widehat\bSigma_{a,T}
=
\widehat \bJ_{a,T}^{-1}
\widehat \bI_{a,T}
\widehat \bJ_{a,T}^{-1}$,
where
\begin{equation}\label{eq::estimator-asympt-var-ex2}
\widehat \bJ_{a,T}
:=
\frac{1}{T}
\sum_{t=1}^{T}
\frac{1_{\{A_t=a\}}}{\pi_t(A_t)}
\bigl(\bX_t\bX_t^\top-\bSigma_e\bigr),\quad
\widehat \bI_{a,T}
:=
\frac{1}{S_{a,T}^{2}}
\sum_{t=1}^{T}
\frac{1_{\{A_t=a\}}}{\pi_t(A_t)^2}
\left[
\bigl(\bX_t\bX_t^\top-\bSigma_e\bigr)
\widehat{\btheta}_a^{(T)}
-
\bX_tY_t
\right]^{\otimes2}. 
\end{equation}
\end{corollary}

\subsection{Off-policy Evaluation (OPE) with Adaptively Collected Data} 

We now consider Example \ref{ex::ope}, where the target parameter $V^* = \sum_{a\in\cA}\btheta_a^*$, and $\btheta_a^* = \EE [\pi^e(a|\bX_t) Y_t]$ solves (\ref{eq::theta-a-*}) with an arm-specific score function (\ref{eq::target-parameter-ope}). We estimate $V^*$ by 
$\hat V^{(T)} = \sum_{a\in\cA}\hat\btheta_a^{(T)}$,
where 
\begin{equation}
\hat\btheta_a^{(T)} = \left(\frac1T\sum_{t=1}^T\frac{1_{\{A_t = a\}}}{\pi_t(A_t)}\right)^{-1}\left(\frac1T\sum_{t=1}^T\frac{1_{\{A_t = a\}}}{\pi_t(A_t)}\pi^e(a|\bX_t) Y_t\right)
\end{equation}
satisfies (\ref{eq::estimating-equation-general}). 

The asymptotic distribution of $\hat V^{(T)}$ can be derived from the joint distribution of $\{\hat\btheta_a^{(T)}\}_{a\in\cA}$ via Theorem \ref{thm::asymptotic-normality-joint-general}, enabling statistical inference. Specifically, we impose the following condition, which encompasses Assumptions \ref{aspt:identifiability-general}–\ref{aspt:smoothness-general} required for the theorem. 

\begin{assumption}
\label{aspt:ex3}
There exist constants \(R_\theta,M_Y>0\) such that $\sup_{a\in\cA}|\btheta_a^*|<R_\theta$\\
and $\sup_{a\in\cA}
\EE\!\left[|Y_t(a)|^4\mid\bX_t\right]
\le M_Y$ a.s.
\end{assumption}

\begin{corollary}[Off-policy evaluation]\label{cor:ope-example}
Suppose Assumptions~\ref{aspt:boundedness-general} and~\ref{aspt:ex3} hold, and suppose the policy-side conditions of Theorem~\ref{thm::asymptotic-normality-joint-general} hold for every action \(a\in\cA\). Let $\bar \bI_a^{\mathrm{OPE}}
:=
\EE\!\left[
\rho(a,\bX_t)
\left(
\pi^e(a\mid\bX_t)Y_t(a)-\btheta_a^*
\right)^{2}
\right]$. Then Theorem~\ref{thm::asymptotic-normality-joint-general} gives
\[
\operatorname{Diag}(b_{1,T},\ldots,b_{K,T})
\bigl(\widehat\btheta^{(T)}-\btheta^*\bigr)
\xrightarrow{d}
\cN\!\left(
\boldsymbol 0,
\operatorname{Diag}
\bigl(\bar \bI_1^{\mathrm{OPE}},\ldots,
      \bar \bI_K^{\mathrm{OPE}}\bigr)
\right).
\]
Further, let \(b_T\to\infty\) be a deterministic sequence and let \(\cA_0\subseteq\cA\) be nonempty such that
\[
\frac{b_T}{b_{a,T}}\to \lambda_a\in(0,\infty),
\qquad a\in\cA_0,
\]
and
\[
\frac{b_T}{b_{a,T}}\to 0,
\qquad a\notin\cA_0.
\]
Then
\[
b_T(\widehat V^{(T)}-V^*)
\xrightarrow{d}
\cN(0,\sigma_V^2),
\]
where $\sigma_V^2
=
\sum_{a\in\cA_0}
\lambda_a^2\bar \bI_a^{\mathrm{OPE}}$.

A consistent estimator of \(\sigma_V^2\) is $\widehat\sigma_{V,T}^2
=
\sum_{a\in\cA}
\left(\frac{b_T}{b_{a,T}}\right)^2
\widehat\Sigma_{a,T}^{\mathrm{OPE}},$
where $\widehat\Sigma_{a,T}^{\mathrm{OPE}}
=
\widehat J_{a,T}^{-2}\widehat I_{a,T}$, with
\begin{equation}
\widehat J_{a,T}
=
\frac{1}{T}
\sum_{t=1}^T
\frac{\mathbf 1\{A_t=a\}}{\pi_t(A_t)}, \quad 
\widehat I_{a,T}
=
\frac{1}{S_{a,T}^2}
\sum_{t=1}^T
\frac{1_{\{A_t=a\}}}{\pi_t(A_t)^2}
\left(
\pi^e(a\mid\bX_t)Y_t
-
\widehat\theta_a^{(T)}
\right)^2.   
\end{equation}
\end{corollary}
Similar to Example \ref{ex::misspecified-linear-bandits}, the variance estimator can be simplified by replacing $\widehat J_{a,T}$ with 1, since $\widehat J_{a,T}\xrightarrow{p} 1$ in this setting.

Finally, we note that \citet{zhan2021off,bibaut2021post} propose alternative estimators for the same OPE target by adding adaptive normalization weights to doubly robust estimators \citep{jiang2016doubly}. These methods can reduce variance when the adaptively fitted reward regression is stable and well aligned with the true reward function. However, in adaptive data collection regimes, guaranteeing such regression convergence and alignment can be difficult, especially when the reward model is misspecified or the behavior policy changes over time.

Our approach takes a different route. Rather than relying on a consistently convergent reward regression, we decompose the value target across actions and apply the IPW-Z framework to each component. This yields a simple inference procedure whose validity is driven by inverse-propensity stability and the effective sample sizes of the sampled actions. As a result, our theory naturally accommodates action-specific sampling rates and can allow substantially faster decay of sampling probabilities than the regimes typically handled by existing OPE methods. In the fixed-scale case, where action probabilities remain bounded away from zero, the result reduces to the familiar \(\sqrt T\)-normal limit.

This simplicity comes with practical advantages: the method does not require auxiliary data, sample splitting, or consistent estimation of a conditional reward model, and it remains applicable to general Z-estimation targets beyond OPE. In Section~\ref{sec::simulations-ope}, we show across five simulation environments that this simple estimator is consistently competitive with the approaches of \citep{zhan2021off,bibaut2021post}---whether their regression component is chosen as a zero function, a linear model, or a flexible tree model---and in several cases yields clearer gains in stability and efficiency. These findings suggest that, in adaptive settings, avoiding unstable regression-based variance reduction can be as important as exploiting it.

\bibliographystyle{unsrtnat} 
\bibliography{main}       
\end{document}